\documentclass{article}

 \usepackage[main, final]{neurips_2026}

\usepackage[utf8]{inputenc} %
\usepackage[T1]{fontenc}    %
\usepackage{hyperref}       %
\usepackage{url}            %
\usepackage{booktabs}       %
\usepackage{amsfonts}       %
\usepackage{nicefrac}       %
\usepackage{microtype}      %
\usepackage{xcolor}         %

\usepackage{amsmath}
\usepackage{amssymb}
\usepackage{amsthm}

\usepackage{thmtools}

\usepackage{multirow}
\usepackage{wrapfig}
\usepackage{xspace}
\usepackage{caption}
\usepackage{thm-restate}
\usepackage{footmisc}
\usepackage{cleveref}
\usepackage{graphicx}

\definecolor{linkblue}{HTML}{1F4E79}
\definecolor{citegreen}{HTML}{4F6F52}
\definecolor{urlbrown}{HTML}{7A4E2D}
\definecolor{notered}{HTML}{A23B3B}
\definecolor{noteblue}{HTML}{2E5C8A}
\definecolor{notemagenta}{HTML}{8A4F7D}

\hypersetup{
  colorlinks=true,
  linkcolor=linkblue,
  citecolor=citegreen,
  urlcolor=urlbrown,
}

\declaretheoremstyle[
  spaceabove=\topsep,
  spacebelow=\topsep,
]{compactthm}
\declaretheorem[style=compactthm,name=Theorem]{theorem}
\declaretheorem[style=compactthm,sibling=theorem,name=Lemma]{lemma}
\declaretheorem[style=compactthm,sibling=theorem,name=Proposition]{proposition}
\declaretheorem[style=compactthm,sibling=theorem,name=Corollary]{corollary}
\declaretheorem[style=compactthm,sibling=theorem,name=Definition]{definition}

\makeatletter
\renewenvironment{proof}[1][\proofname]{\par
  \pushQED{\qed}%
  \normalfont \topsep0\p@\relax
  \trivlist
  \item[\hskip\labelsep
        \itshape
    #1\@addpunct{.}]\ignorespaces
}{%
  \popQED\endtrivlist\@endpefalse
  \addvspace{6\p@\@plus6\p@}%
}
\makeatother

\crefname{theorem}{Theorem}{Theorems}
\crefname{lemma}{Lemma}{Lemmas}
\crefname{proposition}{Proposition}{Propositions}
\crefname{corollary}{Corollary}{Corollaries}
\crefname{definition}{Definition}{Definitions}
\crefname{figure}{Figure}{Figures}
\crefname{table}{Table}{Tables}
\crefname{section}{Section}{Sections}
\crefname{appendix}{Appendix}{Appendices}
\Crefname{appendix}{Appendix}{Appendices}
\crefname{subappendix}{Appendix}{Appendices}
\Crefname{subappendix}{Appendix}{Appendices}

\AddToHook{cmd/appendix/after}{%
  \crefalias{section}{appendix}%
  \crefalias{subsection}{appendix}%
  \crefalias{subsubsection}{appendix}%
}

\newcommand{\EE}{\mathbb{E}}
\newcommand{\RR}{\mathbb{R}}
\newcommand{\Var}{\mathrm{Var}}
\newcommand{\Cov}{\mathrm{Cov}}
\newcommand{\sigalg}{\mathcal{F}}

\newcommand{\Fourier}{\texttt{Fourier}}
\newcommand{\Legendre}{\texttt{Legendre}}
\newcommand{\Insertion}{\texttt{Insertion}}
\newcommand{\RegTAd}[1]{\texttt{Reg-TAd#1}}
\newcommand{\RegTLLd}[1]{\texttt{Reg-TLLd#1}}
\newcommand{\OursTAd}[1]{\texttt{Ours-TAd#1}}
\newcommand{\OursTLLd}[1]{\texttt{Ours-TLLd#1}}

\DeclareMathOperator{\sign}{sign}

\title{Probabilistic Signature Inversion: Learning Conditional Distributions from Truncated Signatures}

\author{%
  Junoh Kang$^{1}$ \qquad\qquad Kiseop Lee$^{2}$ \qquad\qquad Bohyung Han$^{1,3}$ \\
  $^{1}$ECE \& $^{3}$IPAI, Seoul National University \qquad $^{2}$Department of Statistics, Purdue University\\
  \texttt{junoh.kang@snu.ac.kr} \qquad \texttt{kiseop@purdue.edu} \qquad \texttt{bhhan@snu.ac.kr}
}

\begin{document}

\maketitle

\begin{abstract}
The signature transform is a principled feature map for continuous-time paths, valued for its uniqueness and universality.
Recovering a path from its truncated signature is, however, structurally ill-posed, because the truncated signature map is not injective.
We therefore reframe truncated signature inversion as a probabilistic problem---learning the conditional distribution of a path given its truncated signature---and adopt a signature-conditioned flow matching model as a practical estimator.
This formulation elucidates the fundamental difficulty of inversion: posterior-sampling reconstruction error quantifies the irreducible uncertainty remaining after conditioning on a statistic.
For Gaussian process families, we derive this error from the tractable posterior under linear-statistics conditioning, obtaining a closed form for log-GBM and numerically tractable formulas for drifted fBM and OU.
This oracle error upper-bounds the corresponding error under truncated-signature conditioning, since truncated signatures are more informative than linear statistics.
The experimental results are consistent with the theoretical analysis: the learned sampler under linear-statistics conditioning closely matches the theoretically derived posterior in both reconstruction error and distribution, while reconstruction errors decrease under richer truncated-signature conditioning.
Moreover, generated paths faithfully recover the conditioning signature and preserve key distributional and temporal structures.
Together, these results establish a well-posed probabilistic framework for truncated-signature inversion, with applicability demonstrated on financial data beyond the process families covered by theory.
Code is available at \url{https://github.com/Junoh-Kang/ProbSigInv}.
\end{abstract}

\section{Introduction}
\label{sec:introduction}
\vspace{-2mm}

The signature transform represents a continuous-time path as an infinite sequence of its iterated integrals~\citep{lyons2014rough, chevyrev2016primer}, and owes its power to two foundational properties: uniqueness, which ensures that the full signature identifies a path up to tree-like equivalence~\citep{hambly2010uniqueness, boedihardjo2016rough}, and universality, which allows any continuous path functional to be approximated by a linear functional of the signature~\citep{chevyrev2016primer}.
Signatures have thus become useful representations in time series analysis~\citep{gyurko2013extracting, guo2025consistency} and, more recently, in generative modeling~\citep{ni2021sigwasserstein, liao2024sigconditional}.
Uniqueness, in particular, motivates the problem of signature inversion.
While a path is provably recoverable from its full signature in the rough-path setting~\citep{geng2017reconstruction}, such reconstruction is infeasible in practice, as it requires the infinite-dimensional full signature.
Practical methods must instead rely on truncated signatures, employing techniques such as symmetrization-based inversion~\citep{lyons2018inverting}, insertion~\citep{chang2019insertion, fermanian2024insertion}, or basis-function recovery~\citep{barancikova2025sigdiffusion}.
Despite their algorithmic diversity, these approaches predominantly adopt a point-reconstruction viewpoint, aiming to return a single representative path.

\begin{figure}[t]
\centering
\captionsetup{skip=6pt}
\setlength{\tabcolsep}{0pt}
\renewcommand{\arraystretch}{0.8}
\scalebox{1}{
\begin{tabular}{cc}
    \multicolumn{2}{c}{\includegraphics[width=0.85\linewidth]{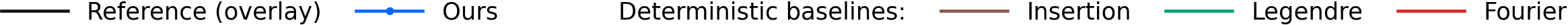}}\\
    \includegraphics[width=0.5\linewidth]{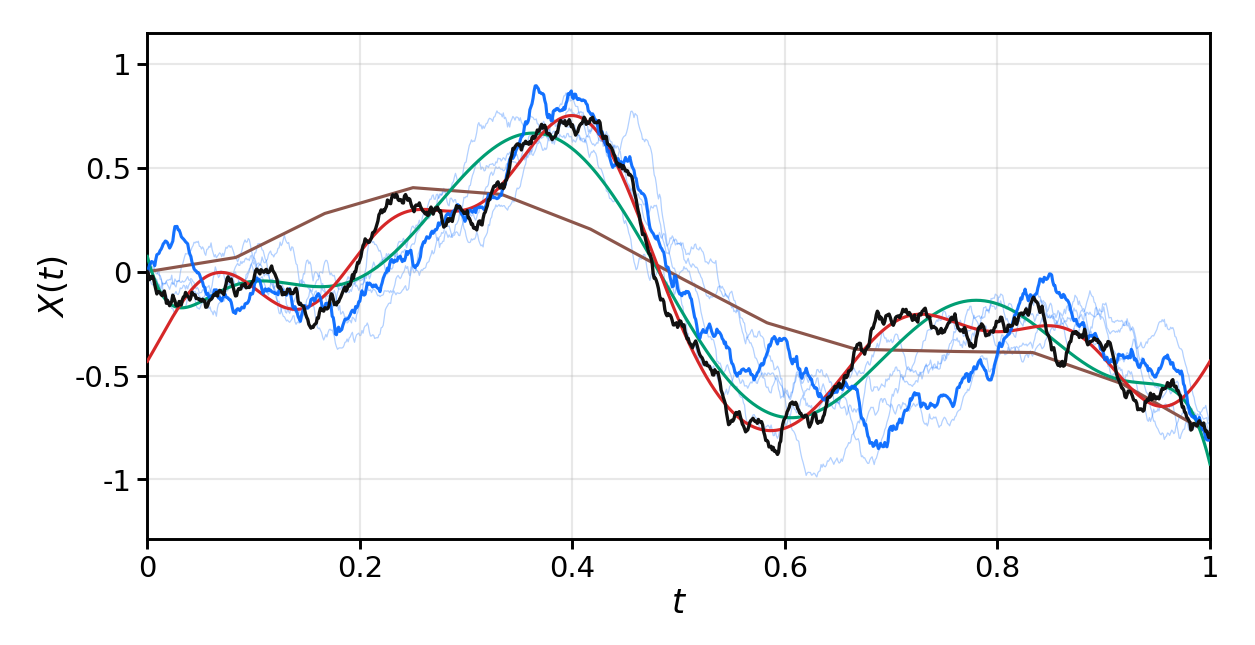}
    &
    \includegraphics[width=0.5\linewidth]{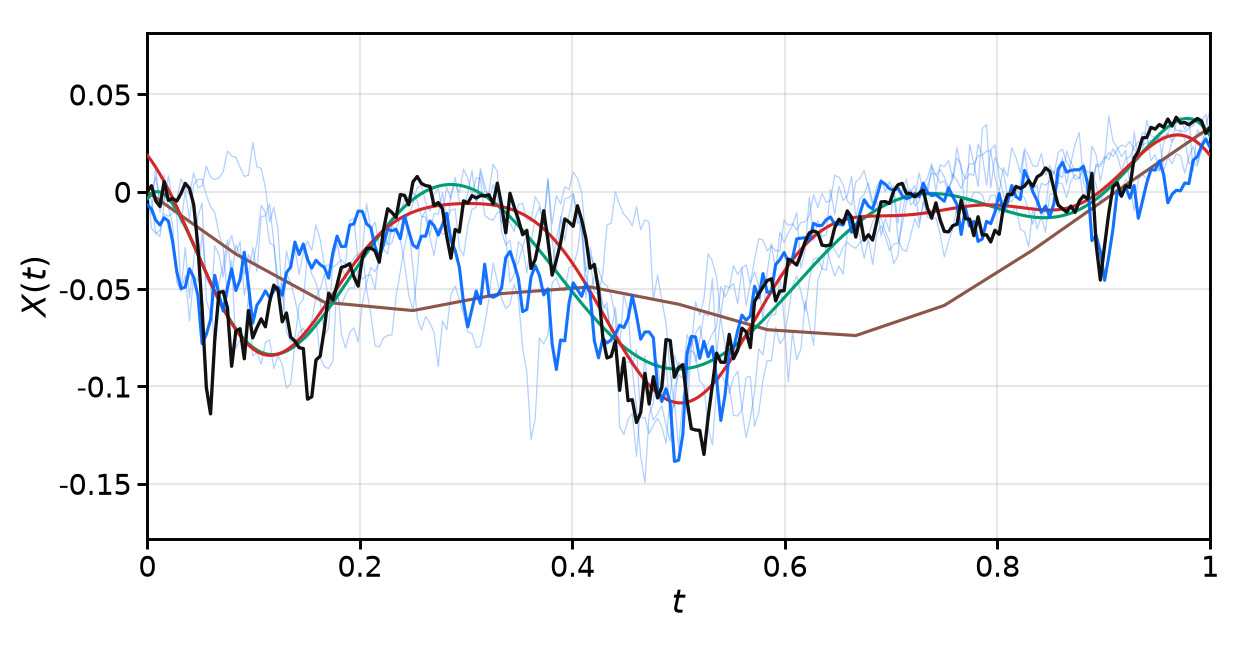}\vspace{-1mm}\\
    (a) fBM~($H=0.7$) & (b) S\&P~500
\end{tabular}
}
\caption{\textbf{Probabilistic and deterministic signature inversion.}
Thin blue curves show five randomly selected paths from our probabilistic signature inversion method, with one draw highlighted in thick blue; deterministic baselines return one path each.}
\label{fig:teaser-overlay}
\vspace{-4mm}
\end{figure}

However, this deterministic point reconstruction leaves the non-injectivity induced by finite truncation unaddressed.
While the full signature identifies a path up to tree-like equivalence, for any finite depth $r$, the truncated signature map $X \mapsto S^{\le r}(X)$ is inherently many-to-one, and thus specifies an entire family of compatible paths rather than a unique inverse.
To account for this structural ambiguity, we reframe truncated-signature inversion as the task of learning the prior-induced conditional distribution of paths given their signatures, which we estimate with a signature-conditioned flow matching model.

To quantify the ambiguity induced by finite-depth truncation, we analyze the posterior-sampling reconstruction error, which measures the irreducible uncertainty remaining after conditioning on the available information.
For Gaussian process families, we derive this error from the tractable posterior under linear-statistics conditioning---specifically path endpoints and time moments.
This yields a closed form for the log-process of geometric Brownian motion~(log-GBM) and numerically tractable formulas for drifted fractional Brownian motion~(drifted fBM) and the Ornstein--Uhlenbeck~(OU) process.
Critically, since linear statistics are coarser than truncated signatures, their associated oracle errors provide a computable upper bound on the corresponding error under truncated-signature conditioning.

We evaluate this formulation on three standard process families representing complementary structures in time-series modeling: log-GBM for Brownian diffusion, drifted fBM for long memory, and OU for mean reversion.
Across these families, the experimental results are consistent with the theoretical analysis: the learned sampler under linear-statistics conditioning closely matches the theoretically derived posterior in both reconstruction error and distribution, while empirical errors under richer signature conditionings generally follow the predicted $\sigma$-algebra hierarchy.
\cref{fig:teaser-overlay} illustrates our framework: the blue trajectories form plausible path ensembles around a drifted fBM trajectory and an S\&P~500 window.
Parameter-estimation and path-level diagnostics further show that generated samples preserve the underlying distributional and temporal structure.

More broadly, probabilistic signature inversion extends the utility of signatures as a path representation, suggesting downstream uses such as probing how signature coordinates affect path behavior and generating signature-conditioned scenarios for financial time series.
Our key contributions are as follows:
\begin{itemize}
    \item We reformulate truncated-signature inversion as probabilistic signature inversion. Since truncation breaks injectivity, we frame the target as a conditional distribution rather than a single path, and we estimate this distribution using signature-conditioned flow matching.
    \item We quantify the irreducible ambiguity in truncated-signature inversion through the posterior-sampling reconstruction error.
    For Gaussian processes, we derive a tractable linear-statistics benchmark---providing a closed form for log-GBM and numerically tractable formulas for drifted fBM and OU---which upper-bounds the corresponding errors under richer signature conditionings.
    \item We empirically validate the framework against tractable theoretical benchmarks and demonstrate that generated samples preserve distributional and temporal structure across Gaussian processes and financial time series.
\end{itemize}

\section{Related work}
\label{sec:related_work}

\subsection{Signature inversion}
\label{subsec:siginv}

The inverse problem of recovering paths from signatures has been studied from both theoretical and algorithmic perspectives.
A theoretical line of work proves reconstruction under appropriate assumptions for several path classes, including Brownian motion~\citep{lejan2013brownian}, Gaussian processes~\citep{boedihardjo2015nonmarkov}, and deterministic rough paths~\citep{geng2017reconstruction}.
However, these results rely on the full infinite-dimensional signature, limiting their direct practical use.

Explicit reconstruction procedures use truncated signatures to produce algorithmic path approximations under structural restrictions.
For example, \citet{lyons2018inverting} construct stable approximations for $C^1$ paths, while the insertion method~\citep{chang2019insertion,fermanian2024insertion} reconstructs piecewise-linear paths through iterative point insertion.
\citet{barancikova2025sigdiffusion} impose a different structural restriction, recovering signals by estimating coefficients in a prescribed finite basis, such as a Fourier or polynomial basis.
Another line seeks a path whose signature features approximately match a target, using optimization or evolutionary search~\citep{kidger2019deep,buhler2020market}.
These approaches ultimately return a single reconstructed path.
In contrast, our formulation models the truncation-induced ambiguity under an underlying path distribution, yielding a conditional distribution over compatible paths rather than a unique inverse.

\subsection{Time series generation}
\label{subsec:timeseries_genai}

Recent neural time series generators include adversarial~\citep{yoon2019time}, autoregressive~\citep{rasul2021autoregressive}, and diffusion-based models~\citep{tashiro2021csdi, yuan2024diffusionts, tanaka2025cofindiff}.
However, their conditioning mechanisms serve different objectives: forecasting and imputation models~\citep{rasul2021autoregressive, tashiro2021csdi} typically condition on partial observations, whereas CoFinDiff~\citep{tanaka2025cofindiff} targets population-level summaries like trend and volatility.
In contrast, we focus on inversion from an individual path-specific statistic.
Note that, due to these fundamental differences in conditioning and task objectives, a direct empirical comparison with forecasting- or population-summary-based generators is not straightforward.

Signatures have also been used in time series generation, but their role differs from ours.
One line of work utilizes the signature as a metric for distributional matching, either through adversarial discrimination in the signature space~\citep{ni2021sigwasserstein, liao2024sigconditional} or by employing kernel MMD objectives between path distributions~\citep{lu2025mmdsignature}.
A second line of research focuses on generating samples directly within the (log-)signature space, subsequently recovering paths through a separate inversion step~\citep{barancikova2025sigdiffusion}.
In contrast, we utilize the signature as conditioning information rather than as a training objective or a target representation to be generated.

\section{Preliminaries}
\label{sec:preliminary}

\subsection{Signature of a path}
\label{subsec:signature}

\begin{definition}[Signature] \label{def:signature}
Let the set of all multi-indices be $\mathcal{W} := \bigcup_{k=0}^{\infty} \{1,\ldots,d\}^k$.
For a continuous path $X: [0, T] \to \mathbb{R}^d$ and a multi-index $I=(i_1,\ldots,i_k) \in \mathcal{W}$, the corresponding iterated integral, when well-defined, is
\begin{align}
    S(X)^I_{0,T} := \int_{0 < t_1 < \cdots < t_k < T} dX^{i_1}_{t_1} \cdots dX^{i_k}_{t_k}.
\end{align}
The signature of $X$ over $[0,T]$ is the collection of all iterated integrals, $S(X)_{0,T} = \bigl( S(X)^I_{0,T} \bigr)_{I \in \mathcal{W}}$.
For bounded-variation paths, these integrals are understood in the Riemann--Stieltjes sense; for continuous semimartingales, in the Stratonovich sense.
For fractional Brownian motion with $H\in(1/4,1)$, they are defined via its canonical geometric rough-path lift.
\end{definition}

\begin{theorem}[Tree-like equivalence] \label{thm:tree_like_equiv}
Let $X, Y: [0,T] \to \mathbb{R}^d$ be continuous paths of bounded variation.
Then $S(X)_{0,T} = S(Y)_{0,T}$ if and only if $X$ and $Y$ are equal up to tree-like equivalence.
\end{theorem}
Informally, two paths are tree-like equivalent if they differ only by segments that retrace themselves and hence cancel out.
See \citet{hambly2010uniqueness} for this classical uniqueness result.

\begin{definition}[Signature of observations] \label{def:signature_observation}
Let $X=(x_i)_{i=0}^n$ be observations of a continuous path in $\mathbb{R}^d$ at times $0=t_0<\cdots<t_n=T$.
Writing $X^{\mathrm{lin}}$ for their piecewise-linear interpolation, we define the signature of $X$ as $S(X)_{0,T}:=S(X^{\mathrm{lin}})_{0,T}$.
\end{definition}
Throughout the paper, $S(X)_{0,T}$ denotes the signature of the continuous-time process in the theoretical analysis and that of the piecewise-linear interpolation of its observations in the experiments.

\begin{definition}[Truncated signature] \label{def:truncated_signature}
For $r \in \mathbb{N}$, the truncated signature of $X$ at level $r$, $S(X)^{\leq r}_{0,T}$, is obtained by retaining terms $S(X)^I_{0,T}$ with $|I| \leq r$, where $|I|$ is the length of $I$.
\end{definition}
For finite $r$, the truncated signature map $X \mapsto S(X)_{0,T}^{\leq r}$ is generally not injective.

\subsection{Augmentation}
\label{subsec:augmentation}

Path augmentation is widely used to enrich the information captured by signatures.
Two common choices are time and time lead-lag augmentations, defined below.

\begin{definition}[Time augmentation] \label{def:augmentation_ta}
Let $X$ be observed at times $0=t_0<\cdots<t_n=T$.
The time augmentation of these observations is a sequence of pairs:
\begin{align}
    \mathrm{TA}(X) := \bigl((t_0, X_{t_0}), (t_1, X_{t_1}), \ldots, (t_n, X_{t_n})\bigr).
\end{align}
\end{definition}

\begin{definition}[Time lead-lag augmentation] \label{def:augmentation_tll}
Let $X$ be observed at times $0=t_0<\cdots<t_n=T$.
The time lead-lag augmentation of these observations is
\begin{align}
    \mathrm{TLL}(X)_{2i}   &:= (t_i,\, X_{t_i},\, X_{t_i}), \qquad i=0,\ldots,n, \\
    \mathrm{TLL}(X)_{2i+1} &:= (t_{i+1},\, X_{t_{i+1}},\, X_{t_i}), \qquad i=0,\ldots,n-1,
\end{align}
where the second coordinate is the lead and the third is the lag.
\end{definition}

The asynchronous updating of lead and lag coordinates encodes the quadratic variation of the original series in the level-2 signature terms~\citep{chevyrev2016primer,flint2016hoff}.
For $A \in \{\mathrm{TA}, \mathrm{TLL}\}$, let $S_A^{\leq r}(X)$ denote the level-$r$ truncated signature of the augmented path $A(X)$.

\section{Probabilistic signature inversion}
\label{sec:method}

\subsection{Problem formulation}
\label{subsec:formulation}

The full signature under either augmentation in \cref{subsec:augmentation} determines the path exactly when the starting point is fixed at $X_0=0$: the monotone time coordinate strengthens tree-like equivalence in \cref{thm:tree_like_equiv} to exact equivalence~\citep{levin2013learning,morrill2021generalised}.
However, the truncated signature $S_A^{\leq r}(X)$ is a finite-dimensional projection of this infinite object, and multiple distinct paths may share the same truncation.
This non-uniqueness motivates our formulation of signature inversion as learning the conditional distribution of paths given their truncated signatures.

Formally, let $X \sim p_{\text{data}}$ be a sample path and write $s = S_A^{\leq r}(X)$ for its truncated signature.
Probabilistic signature inversion is defined as the task of learning the conditional distribution
\begin{align}
    P\bigl(X \mid S_A^{\leq r}(X)=s\bigr),
\end{align}
which is a distribution over paths that are consistent with the signature $s$ and plausible under $p_{\text{data}}$.

\subsection{Signature-conditioned flow matching}
\label{subsec:diffusion}

We employ a flow matching model~\citep{Lipman2023flow} to approximate this conditional distribution.
In the objective below, $X$ denotes the observed sequence at times $0=t_0<\cdots<t_n=T$.
We train the model to learn a signature-conditioned velocity field that transports samples from a known distribution to data paths by minimizing
\begin{align}
    \mathcal{L}_\text{FM}(\phi) = \EE_{X^{(1)}, X^{(0)}, \tau}
    \bigl[
        \| u_\phi(X^{(\tau)};\tau,S_A^{\leq r}(X^{(1)})) - (X^{(1)} - X^{(0)})\|_2^2
    \bigr],
\end{align}
where $X^{(\tau)} = \tau X^{(1)} + (1-\tau) X^{(0)}$, with $X^{(1)} \sim p_\text{data}$, $X^{(0)}\sim\mathcal{N}(\mathbf{0},\mathbf{I})$, and $\tau \sim \mathcal{U}(0,1)$.
At inference time, we invert a signature $s$ by solving the following ODE:
\begin{align}
    \frac{dX^{(\tau)}}{d\tau} = u_\phi\bigl(X^{(\tau)};\tau,s\bigr),
    \quad X^{(0)} \sim \mathcal{N}(\mathbf{0}, \mathbf{I}).
\end{align}

\section{Posterior-sampling reconstruction error analysis}
\label{sec:error_analysis}
Probabilistic signature inversion reframes the ambiguity introduced by finite-depth signature truncation as intrinsic, making the inverse problem distributional rather than deterministic.
We quantify this ambiguity through the posterior-sampling reconstruction error, which measures the irreducible path-level variation remaining after conditioning on a truncated signature.
In this section, we derive tractable oracle errors under linear statistics for Gaussian process families, which upper-bound the corresponding oracle errors under richer signature conditionings.

\subsection{Posterior-sampling reconstruction error for general parametric processes}
\label{subsec:bayes-recon-error}

\begin{definition}[Posterior-sampling reconstruction error]\label{def:recon-error}
A prior $\pi_\Theta$ on $\Theta$ induces the mixture law $\bar{P}_\Theta := \int_\Theta P_\theta\,\pi(d\theta)$.
For a statistic $\Phi$, let $Q_\Theta(\cdot\mid s) = \bar{P}_\Theta(X\in\cdot\mid \Phi(X)=s)$ be the statistic-conditioned distribution.
The posterior-sampling reconstruction error under the statistic $\Phi$ is
\begin{align} \label{eq:recon-error}
  \mathcal{E}_\Theta(\Phi)
  = \EE_{X\sim\bar{P}_{\Theta},\,
        \tilde{X}\sim Q_\Theta(\cdot\mid\Phi(X))}
    \biggl[\int_0^T |X_t-\tilde{X}_t|^2\,dt\biggr].
\end{align}
When $\Phi=S^{\le r}_A$, $Q_\Theta$ is the signature-conditioned distribution.
\end{definition}

\begin{restatable}[Within--between variance decomposition]{proposition}{withinbetweenvarprop}\label{prop:within-between-var}
For any statistic $\Phi$,
\begin{align}
  \label{eq:within-between-var}
  \mathcal{E}_\Theta(\Phi)
  &= 2\EE_{S\sim\Phi_\#\bar{P}_{\Theta}} \biggl[
      \int_0^T\!\Var_{Q_{\Theta}}\!(X_t\mid S)\,dt
    \biggr] \\
  \label{eq:within-between-var-total}
  &= 2\EE_S \biggl[
      \int_0^T \EE_{\theta\mid S}\big[\Var_\theta(X_t\mid S)\big]\,dt
      + \int_0^T \Var_{\theta\mid S}\big(\EE_\theta[X_t\mid S]\big)\,dt
  \biggr].
\end{align}
\end{restatable}

\begin{proof}
The proof is given in \cref{app:proof-within-between-var}.
\end{proof}

\begin{definition}[Linear statistic]\label{def:lin-stat}
For $r\ge 1$, the linear statistic
$L^{(r)}\colon C([0,T];\RR)\to\RR^r$ is defined by
\begin{align}\label{eq:lin-stat}
  L^{(r)}(X)
  := \bigl(\ell_0(X), \ell_1(X), \ldots, \ell_{r-1}(X)\bigr)^\top,
\end{align}
where $\ell_0(X):=X_T$ and
$\ell_p(X):=\int_0^T t^{p-1}X_t\,dt$ for $p=1,\ldots,r-1$.
\end{definition}

\begin{restatable}[Linear statistic upper bound]{proposition}{linearstatboundprop}\label{prop:lin-stat-bound}
For $r\ge 1$,
\begin{align}
  \mathcal{E}_\Theta(S^{\le r}_{\mathrm{TLL}})
  \le
  \mathcal{E}_\Theta(S^{\le r}_{\mathrm{TA}})
  \le
  \mathcal{E}_\Theta(L^{(r)}).
\end{align}
Moreover, $\mathcal{E}_\Theta(S^{\le r}_{\mathrm{TA}})=\mathcal{E}_\Theta(L^{(r)})$ for $r\le 2$.
\end{restatable}

\begin{proof}
Each $\ell_p(X)$ is determined by TA signature coordinates up to level $p+1$, directly for $p=0$ and by integration by parts for $p\ge1$.
Moreover, $\mathrm{TA}(X)$ is the (time, lead) projection of $\mathrm{TLL}(X)$, up to reparametrization.
Hence $\sigalg(L^{(r)}(X))\subseteq\sigalg(S^{\le r}_{\mathrm{TA}}(X))\subseteq\sigalg(S^{\le r}_{\mathrm{TLL}}(X))$.
For $r\le2$, the TA signature is also determined by $L^{(r)}(X)$, so their $\sigma$-algebras coincide.

For nested $\sigma$-algebras $\mathcal{F}\subseteq\mathcal{G}$, the law of total variance and the tower property give ${\EE[\Var(X_t\mid\mathcal{F})]}\ge{\EE[\Var(X_t\mid\mathcal{G})]}$.
Under the integrability condition $\EE[\int_0^T|X_t|^2dt]<\infty$, integrating over time, interchanging expectation and integration by Fubini's theorem, and applying \cref{prop:within-between-var} yields the stated inequalities, with equality when the $\sigma$-algebras coincide.
\end{proof}

\subsection{Oracle posterior-sampling reconstruction error for Gaussian processes}
\label{subsec:bayes-gaussian}

We now focus on Gaussian process families, whose analytical tractability and well-understood distributional properties have made them standard models across many domains.
\cref{prop:within-between-var} reduces the oracle error to fixed-parameter conditional mean and variance terms under the observed statistic.
For linear statistics, the following proposition gives the Gaussian regression formula for these moments, which will be used for the kernel-based computation in \cref{subsec:kernel}.

\begin{proposition}[Gaussian regression under linear statistics]\label{prop:gauss-reg}
Let $X_t = m(t) + G_t$, where $m$ is deterministic and $G_t$ is a centered Gaussian process with covariance kernel $K(\cdot,\cdot)$.
For $r \geq 1$, define
\begin{align}
  k^{(r)}(t) &:= \Cov\bigl(G_t, L^{(r)}(G)\bigr), \qquad
  \Sigma^{(r)} := \Cov\bigl(L^{(r)}(G), L^{(r)}(G)\bigr).
\end{align}
Assume $\Sigma^{(r)}$ is invertible.
Then $X \mid L^{(r)}(X)$ is a Gaussian process whose conditional mean and pointwise variance are given by
\begin{align}
  &\EE \bigl[X_t \mid L^{(r)}(X)\bigr]
  = m(t) + k^{(r)}(t) \bigl(\Sigma^{(r)}\bigr)^{-1}\, (L^{(r)}(X) - L^{(r)}(m)),\label{eq:gauss-cond-mean} \\
  &\Var \bigl(X_t \mid L^{(r)}(X)\bigr)
  = K(t,t) - k^{(r)}(t) \bigl(\Sigma^{(r)}\bigr)^{-1}\, k^{(r)}(t)^\top. \label{eq:gauss-cond-cov}
\end{align}
The conditional variance is independent of the observed value of $L^{(r)}(X)$.
\end{proposition}

\begin{proof}
Since $m$ is deterministic, $\sigalg(L^{(r)}(X))=\sigalg(L^{(r)}(G))$.
For each $t$, $(G_t, L^{(r)}(G))$ is jointly Gaussian; the standard conditional Gaussian formula gives \cref{eq:gauss-cond-mean,eq:gauss-cond-cov}.
\end{proof}

\paragraph{Process families.}
For the rest of the analysis, we consider three classical Gaussian process families chosen to represent complementary regimes in time series modeling:
the log-GBM as a canonical diffusion baseline with $\theta\!=\!(\mu,\sigma)$,
the drifted fBM for long-memory behavior with $\theta\!=\!(\mu,\sigma,H)$,
and the OU process for mean reversion with $\theta\!=\!(\mu,\sigma,\kappa)$.
All three are initialized at $X_0=0$ and defined by
\begin{align}
    \text{log-GBM:}
    &\quad dX_t = \nu\,dt + \sigma\,dB_t,
    &&\nu := \mu-\tfrac{1}{2}\sigma^2, \label{eq:logGBM} \\
    \text{drifted fBM:}
    &\quad X_t = \nu\,t + \sigma B^H_t,
    &&\nu := \mu-\tfrac{1}{2}\sigma^2, \label{eq:logfBM} \\
    \text{OU:}
    &\quad dX_t = \kappa(\mu-X_t)\,dt + \sigma\,dB_t, \label{eq:OU}
\end{align}
where $B$ is a standard Brownian motion and $B^H$ is a fractional Brownian motion with Hurst parameter $H\in(0,1)$ and covariance $\Cov(B^H_s, B^H_t)=\frac{1}{2}(s^{2H}+t^{2H}-|s-t|^{2H})$.

\subsection{Closed-form analysis for log-GBM}
\label{subsec:closed-form-gbm}

For log-GBM, the oracle error under $L^{(r)}$ can be computed in closed form.
The key observation is that $L^{(r)}(X)$ can be represented by finitely many Wiener integrals, so the fixed-parameter conditional mean and variance terms in \cref{prop:within-between-var} reduce to an $L^2$ projection.
\cref{prop:gbm-error} carries out this reduction and gives the closed-form oracle posterior-sampling reconstruction error.

\begin{restatable}[Gaussian Hilbert-space projection]{lemma}{gaussianprojectionlemma}\label{lemma:gaussian-projection}
Let $B$ be a standard Brownian motion and define the Wiener integral $I(f) := \int_0^T f(s)\, dB_s$ for $f \in L^2([0,T])$.
Then, $I$ is linear and satisfies the It\^{o} isometry, $\EE[I(f) I(g)] = \langle f, g \rangle_{L^2([0,T])}$.
For a closed subspace $M \subset L^2([0,T])$ and $f \in L^2([0,T])$,
\begin{align}
    &\EE\bigl[
        I(f) \mid \sigalg(I(g) : g \in M)
    \bigr]
    = I(\Pi_M f), \\
    &\Var\bigl(
        I(f) \mid \sigalg(I(g) : g \in M)
    \bigr)
    = \| f - \Pi_M f \|_{L^2([0,T])}^2,
\end{align}
where $\Pi_M$ denotes the orthogonal projection onto $M$.
\end{restatable}
\begin{proof}
The proof is given in \cref{app:proof-gaussian-projection}.
\end{proof}

\begin{restatable}[Closed-form oracle posterior-sampling reconstruction error for log-GBM]{proposition}{gbmerrorprop}\label{prop:gbm-error}
For log-GBM and $r \geq 1$, there exists a deterministic function $c_r: [0,T] \to \mathbb{R}^r$ such that
\begin{align}\label{eq:cond-mean-gbm}
    \EE_{\theta}\bigl[X_t \mid L^{(r)}(X)\bigr]
    = c_r(t)^\top L^{(r)}(X), \qquad t \in [0,T].
\end{align}
In particular, the conditional mean is determined entirely by the observed statistic $L^{(r)}(X)$ and does not depend on $\theta = (\mu, \sigma)$.
Thus, the between-parameter conditional-mean variance term in \cref{prop:within-between-var} vanishes.
Moreover, the integrated conditional variance is
\begin{align}\label{eq:int-var-gbm}
    \int_0^T \Var_{\theta}\bigl(X_t \mid L^{(r)}(X)\bigr)\,dt
    = a_r \sigma^2 T^2,
    \qquad
    a_r = \frac{r}{2(2r-1)(2r+1)}.
\end{align}
Substituting these identities into \cref{prop:within-between-var} yields the oracle error
\begin{align}\label{eq:bayes-gbm}
    \mathcal{E}_\Theta(L^{(r)}) = 2\,a_r\,T^2\,\EE [\sigma^2].
\end{align}
\end{restatable}

\begin{proof}[Proof sketch]
Let $S:=L^{(r)}(X)$.
By Fubini, for fixed $\theta$, ~$S=\nu\beta+\sigma(I(f_0),\ldots,I(f_{r-1}))^\top$, where $f_0\!\equiv\! 1$, $f_j(s)\!=\!(T^j-s^j)/j$ for $j \geq 1$, and $\beta\!=\!(T,\ldots,T^r/r)^\top$.
Thus, $\sigalg(S)=\sigalg(I(f):f\in\mathcal H_r)$, where $\mathcal H_r:=\mathrm{span}\{f_0,\ldots,f_{r-1}\}$.
Since $X_t=\nu t+\sigma I(\mathbf 1_{[0,t]})$, \cref{lemma:gaussian-projection} reduces the conditional mean to projecting $\mathbf 1_{[0,t]}$ onto $\mathcal H_r$.
Writing $\Pi_{\mathcal H_r}\mathbf 1_{[0,t]}=c_r(t)^\top(f_0,\ldots,f_{r-1})$,
\begin{align}
    \EE_\theta[X_t\mid S]
    = c_r(t)^\top S+\nu(t-c_r(t)^\top\beta).
\end{align}
Since $f_0\equiv 1\in \mathcal H_r$, ~$t=\langle\mathbf{1}_{[0,t]},\mathbf{1}\rangle = \langle\Pi_{\mathcal H_r}\mathbf{1}_{[0,t]},\mathbf{1}\rangle = c_r(t)^\top\beta$, the second term vanishes, yielding \eqref{eq:cond-mean-gbm}.
Because the resulting conditional mean does not depend on $\theta$, the between-parameter conditional-mean variance term in \cref{prop:within-between-var} vanishes.
For the conditional variance,
\begin{align}
    \int_0^T \Var_\theta(X_t\mid S)\,dt
    = \sigma^2T^2\int_0^1
    \|\mathbf 1_{[0,u]}-\Pi_{\mathcal H_r}\mathbf 1_{[0,u]}\|^2\,du
    = a_r\sigma^2T^2,
\end{align}
where the last equality follows from \cref{lemma:legendre}.
The full proof is given in \cref{app:proof-gbm-error}.
\end{proof}

\begin{corollary}[Depth-two oracle error equality for log-GBM]
\label{cor:gbm-depth-two-equality}
For log-GBM,
\begin{align}
\mathcal{E}_\Theta(L^{(2)})
=
\mathcal{E}_\Theta(S^{\le 2}_{\mathrm{TA}})
=
\mathcal{E}_\Theta(S^{\le 2}_{\mathrm{TLL}}).
\end{align}
\end{corollary}
\begin{proof}
By \cref{prop:lin-stat-bound}, $\mathcal{E}_\Theta(S^{\le2}_{\mathrm{TA}})=\mathcal{E}_\Theta(L^{(2)})$.
For log-GBM, $S^{\le2}_{\mathrm{TLL}}$ is equivalent to $(L^{(2)},\sigma)$, and \cref{prop:gbm-error} shows that the added $\sigma$ does not change the conditional mean.
The conditional-variance decomposition then gives $\mathcal{E}_\Theta(S^{\le2}_{\mathrm{TLL}})=\mathcal{E}_\Theta(L^{(2)})$.
\end{proof}

\begin{corollary}[Multivariate log-GBM oracle]
\label{cor:multivariate-gbm-oracle}
For multivariate log-GBM with diffusion covariance $\Sigma$, the componentwise linear-statistic oracle under squared Euclidean path loss satisfies $\mathcal E_\Theta(L_d^{(r)})=2a_rT^2\EE_{\pi_\Theta}[\operatorname{tr}(\Sigma)]$ for $r\geq1$.
\end{corollary}

\begin{proof}
The full statement and proof are given in \cref{app:multivariate-gbm}.
\end{proof}

\subsection{Kernel-based numerical analysis}
\label{subsec:kernel}

For Gaussian processes with known deterministic mean and covariance kernel---see \cref{tab:kernel} in \cref{app:kernel-formulas} for the three processes considered in this study---both direct sampling from the linear-statistic posterior and computation of $\mathcal{E}_\Theta(L^{(r)})$ are numerically tractable.
At fixed $\theta$, \cref{prop:gauss-reg} expresses the conditional mean and pointwise variance terms required by \cref{prop:within-between-var} in terms of $L^{(r)}(m_\theta)$, $k_\theta^{(r)}(t)$, and $\Sigma_\theta^{(r)}$.
\cref{lemma:kernel} in \cref{app:kernel-formulas} writes $k_\theta^{(r)}(t)$ and $\Sigma_\theta^{(r)}$ from the covariance kernel $K_\theta(\cdot, \cdot)$.
The two-time extension of \cref{eq:gauss-cond-cov} gives the conditional covariance kernel needed for path sampling.
Finally, writing $S:=L^{(r)}(X)$, Bayes' rule gives $p(\theta\mid S)\propto p(S\mid\theta)\pi(\theta)$, where $S\mid\theta\sim\mathcal{N}\bigl(L^{(r)}(m_\theta),\Sigma_\theta^{(r)}\bigr)$.

\section{Experiment}
\label{sec:experiment}
\subsection{Experimental setup}
\label{subsec:implementation}

\paragraph{Datasets.}
We use three Gaussian process families---log-GBM, drifted fBM, and OU---including two-dimensional correlated log-GBM, as controlled benchmarks.
We evaluate behavior beyond parametric models using cumulative log returns of the S\&P~500 and paired foreign-exchange series~(GBP/USD, JPY/USD).

\paragraph{Signature inversion methods.}
We train a flow-matching model with a Diffusion Transformer~\citep{peebles2023scalable}, using linear-statistic~(LS), time-augmented~(TA), or time lead-lag~(TLL) conditioning~(see \cref{app:conditional-diffusion} for conditional diffusion experiments).
Univariate Gaussian-process experiments use depths $r=1,\ldots,6$.
For 2D log-GBM, we evaluate LS and TA at $r=2,4,6$.
S\&P~500 and foreign-exchange experiments use TA and TLL at $r=4$.
Deterministic baselines include regression models trained with the $\ell_2$ loss using the same conditioning and backbone, \Fourier{} and \Legendre~\citep{barancikova2025sigdiffusion}, and \Insertion~\citep{fermanian2024insertion}.

\paragraph{Evaluation.}
For each reference path, inversion generates $30$ samples using $1000$ explicit Euler steps.
We assess theoretical alignment using posterior-sampling reconstruction errors.
Comparisons with deterministic baselines assess process-parameter recovery for Gaussian processes, and stylized facts, autocorrelation, and cross-series correlation for financial data.
We evaluate signature fidelity and provide qualitative comparisons using path overlays, ACF plots, and QQ plots.

Further experimental details are provided in \cref{app:implementation}.

\begin{table}[t]
\centering
\captionsetup{skip=6pt}

\caption{\textbf{Posterior-sampling reconstruction errors.}
LS Oracle denotes the LS oracle posterior-sampling reconstruction error from \cref{sec:error_analysis}, computed analytically for log-GBM and estimated numerically for drifted fBM and OU; LS/TA/TLL Emp. are empirical reconstruction-error estimates.}
\label{tab:bayes_error}

\renewcommand{\arraystretch}{1.1}
\setlength{\tabcolsep}{5pt}
\scalebox{0.85}{
\begin{tabular}{@{}c c cccc c cccc c cccc@{}}
\toprule
& & \multicolumn{4}{c}{log-GBM} & & \multicolumn{4}{c}{fBM} & & \multicolumn{4}{c}{OU} \\
\cmidrule(lr){3-6}\cmidrule(lr){8-11}\cmidrule(lr){13-16}
$r$ & & \shortstack{LS\\Oracle} & \shortstack{LS\\Emp.} & \shortstack{TA\\Emp.} & \shortstack{TLL\\Emp.}
&
& \shortstack{LS\\Oracle} & \shortstack{LS\\Emp.} & \shortstack{TA\\Emp.} & \shortstack{TLL\\Emp.}
&
& \shortstack{LS\\Oracle} & \shortstack{LS\\Emp.} & \shortstack{TA\\Emp.} & \shortstack{TLL\\Emp.} \\ \midrule
1 & & 1.36 & 1.32 & 1.32 & 1.32 & & 1.49 & 1.49 & 1.49 & 1.49 & & 1.16 & 1.14 & 1.14 & 1.14 \\
2 & & 0.54 & 0.55 & 0.55 & 0.55 & & 0.66 & 0.66 & 0.66 & 0.67 & & 0.50 & 0.51 & 0.51 & 0.51 \\
3 & & 0.35 & 0.35 & 0.32 & 0.32 & & 0.45 & 0.46 & 0.45 & 0.44 & & 0.33 & 0.33 & 0.31 & 0.30 \\
4 & & 0.26 & 0.26 & 0.20 & 0.20 & & 0.35 & 0.36 & 0.32 & 0.31 & & 0.25 & 0.25 & 0.19 & 0.20 \\
5 & & 0.21 & 0.21 & 0.15 & 0.15 & & 0.29 & 0.31 & 0.25 & 0.25 & & 0.20 & 0.20 & 0.15 & 0.16 \\
6 & & 0.17 & 0.17 & 0.14 & 0.14 & & 0.25 & 0.27 & 0.23 & 0.25 & & 0.17 & 0.17 & 0.14 & 0.15 \\

\bottomrule
\end{tabular}
}
\vspace{-2mm}
\end{table}

\begin{figure}[t]
\centering
\captionsetup{skip=6pt}
\setlength{\tabcolsep}{0pt}
\renewcommand{\arraystretch}{0.8}
\scalebox{1}{
\begin{tabular}{ccc}
    \multicolumn{3}{c}{\includegraphics[width=0.85\linewidth]{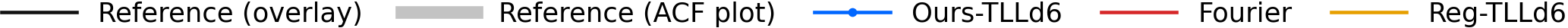}}\\
    \includegraphics[width=0.48\linewidth]{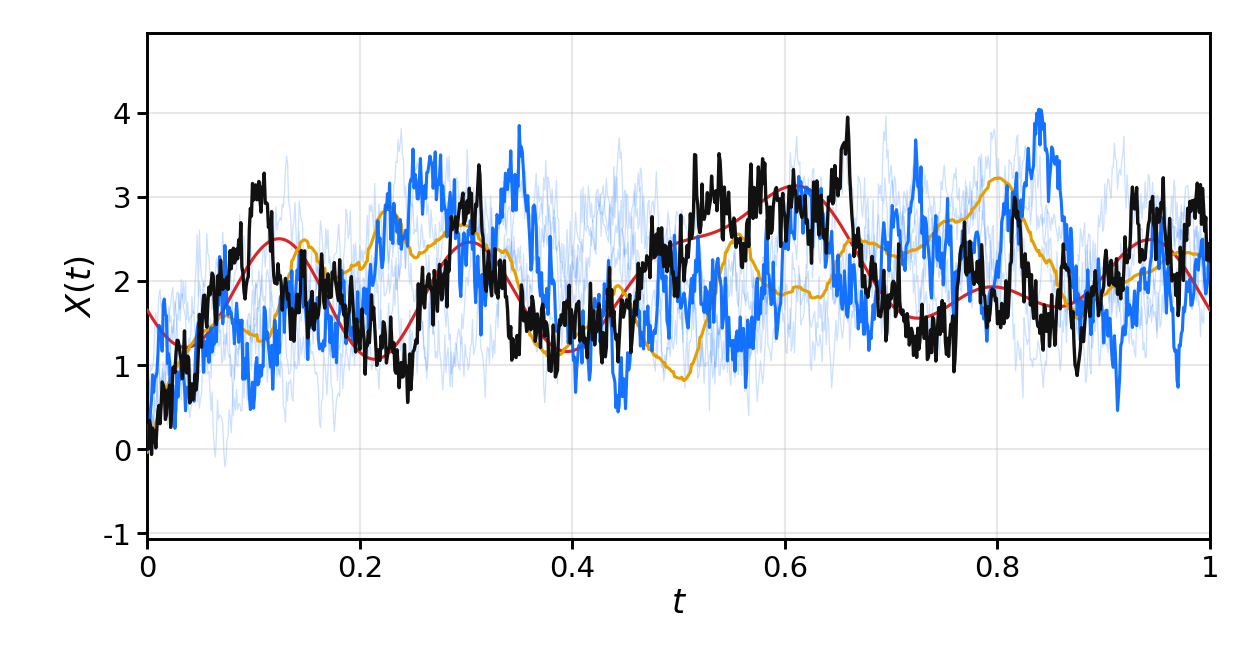}
    &
    \includegraphics[width=0.25\linewidth]{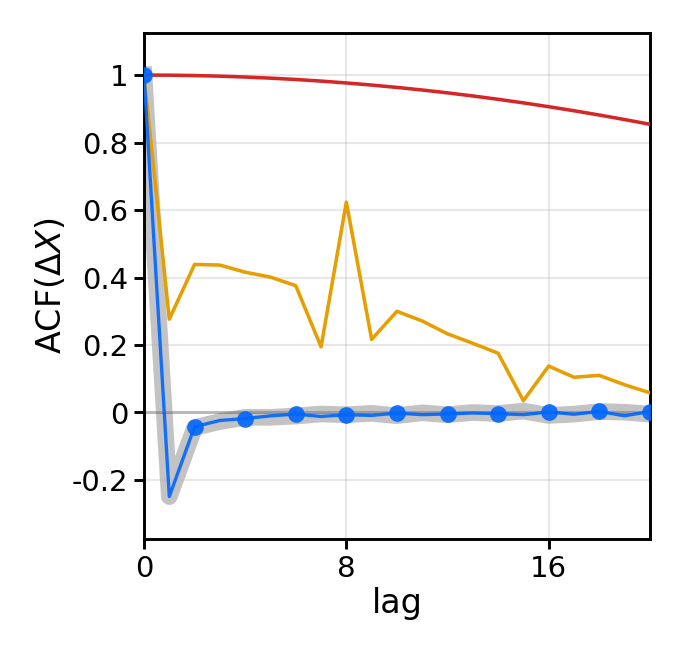}
    &
    \includegraphics[width=0.25\linewidth]{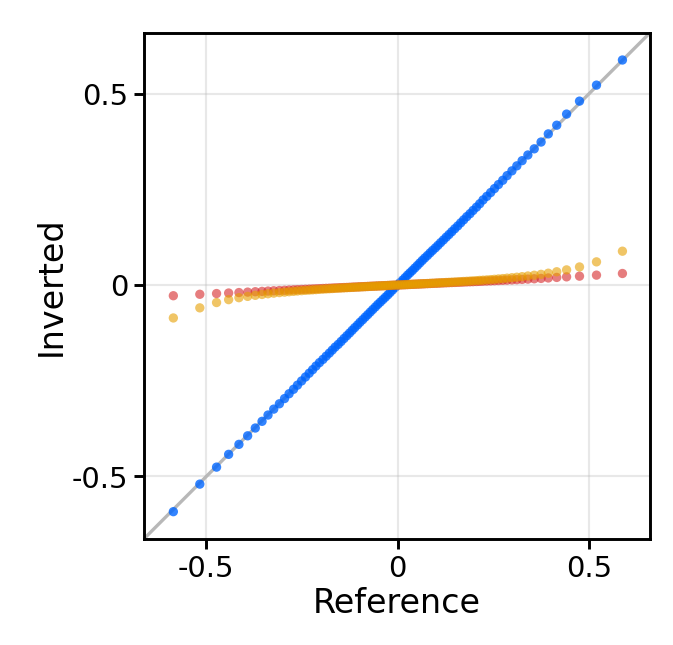} \\[-1.5mm]
    \multicolumn{3}{c}{(a) fBM~($H=0.3,\mu=2.0,\sigma=2.0$)} \\[1mm]
    \includegraphics[width=0.48\linewidth]{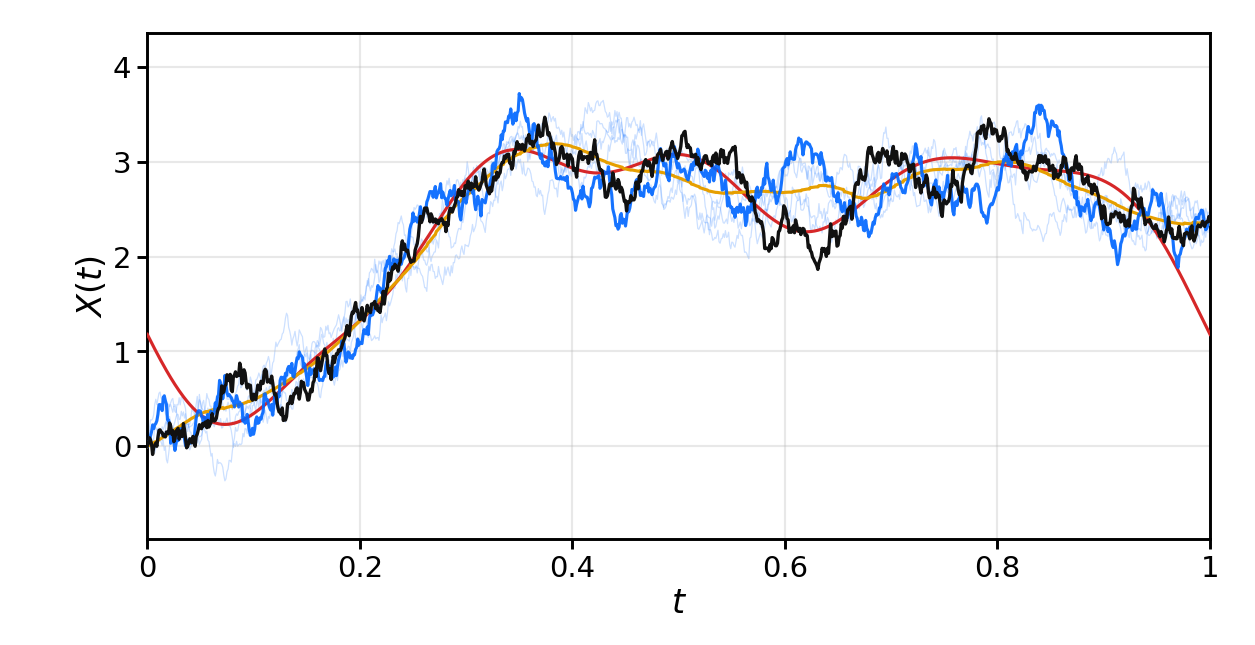}
    &
    \includegraphics[width=0.25\linewidth]{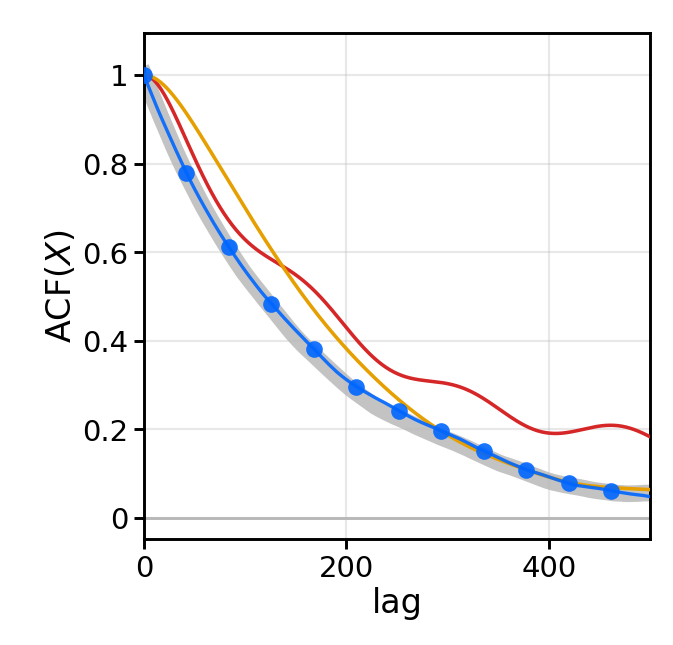}
    &
    \includegraphics[width=0.25\linewidth]{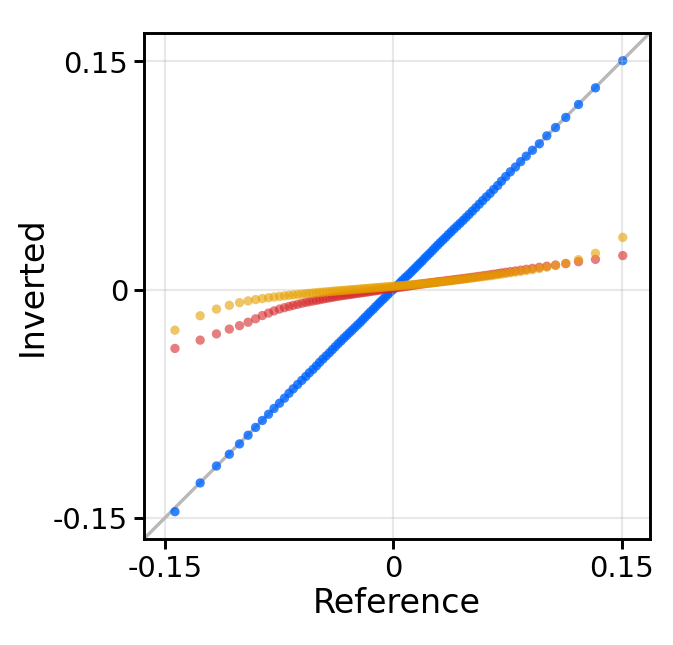} \\[-1.5mm]
    \multicolumn{3}{c}{(b) OU~($\kappa=5.0,\mu=3.0,\sigma=2.0$)}
\end{tabular}
}
\caption{\textbf{Path overlays, ACF plots, and QQ plots for drifted fBM and OU.}
ACFs are computed from $\Delta X$ for drifted fBM and $X$ for OU.
QQ plots are computed from $\Delta X$.
Our model uses TLL conditioning at depth $r=6$.
Additional qualitative results are provided in \cref{app:synthetic-main-style-qual}.}
\label{fig:synthetic-overlay}
\vspace{-2mm}
\end{figure}

\begin{table}[t]
\centering
\captionsetup{skip=6pt}
\caption{\textbf{Parameter-estimation errors for Gaussian processes.}
For drifted fBM, $\sigma$ and $H$ errors condition on the other parameter; for OU, $\sigma$ error conditions on $\kappa,\mu$.
For 2D log-GBM, $\rho$ denotes Brownian correlation.
Bold and underlined values indicate the lowest and second-lowest errors in each column, respectively.
Our models use depth $r=6$ for univariate processes and $r=4$ for 2D log-GBM.
Signature depths for each method are given in \cref{app:training}.}
\label{tab:synthetic_structure}
\renewcommand{\arraystretch}{1.10}
\setlength{\tabcolsep}{8pt}
\scalebox{0.90}{%
\begin{tabular}{@{}lcccccc@{}}
    \toprule
    & \multicolumn{1}{c}{log-GBM} & \multicolumn{2}{c}{fBM} & \multicolumn{1}{c}{OU} & \multicolumn{2}{c}{2D log-GBM} \\
    \cmidrule(lr){2-2}\cmidrule(lr){3-4}\cmidrule(lr){5-5}\cmidrule(lr){6-7}
    Method & $\sigma$ err. & $\sigma$ err. & $H$ err. & $\sigma$ err. & $\sigma$ err. & $\rho$ err. \\ \midrule
    \textbf{Deterministic} \\
    \Fourier & 1.71 & 1.69 & 0.44 & 1.64 & 1.75 & 0.34 \\
    \Legendre & 1.76 & 1.75 & 0.43 & 1.74 & 1.80 & 0.29 \\
    \Insertion & 1.93 & 1.93 & 0.46 & 1.90 & -- & -- \\
    \texttt{Reg-TA} & 1.73 & 1.72 & 0.33 & 1.72 & -- & -- \\
    \texttt{Reg-TLL} & 1.69 & 1.67 & 0.29 & 1.69 & -- & -- \\
    \midrule
    \textbf{Probabilistic} \\
    \texttt{Ours-TA} & \underline{0.22} & \underline{0.87} & \underline{0.05} & \underline{0.23} & \underline{0.24} & \underline{0.24} \\
    \texttt{Ours-TLL} & \textbf{0.06} & \textbf{0.10} & \textbf{0.01} & \textbf{0.06} & \textbf{0.07} & \textbf{0.08} \\
    \bottomrule
\end{tabular}%
}
\end{table}

\subsection{Theoretical alignment}
\label{subsec:exp_bayes}

We empirically validate the theoretical predictions by examining agreement with the LS benchmark and the ordering of posterior-sampling reconstruction errors across conditioning representations.
\cref{tab:bayes_error} shows that under LS conditioning, empirical posterior-sampling reconstruction errors closely track the theory-derived oracle error across process families and truncation depths.
The learned sampler also closely matches the numerical LS posterior across most evaluated settings~(see \cref{tab:ls-posterior-energy-gap} in \cref{app:synthetic-quant}).
Although the oracle is derived in continuous time while empirical errors are computed on piecewise-linear paths with $n=1000$ segments, we expect this discrepancy to be negligible.
For log-GBM, where the piecewise-linear oracle is available in closed form, the correction is $O(n^{-1})$ and below $0.002$ for all $r\le6$~(see \cref{app:pl-gbm-oracle}).

Across conditioning representations, empirical errors also decrease with depth and broadly follow the expected LS--TA--TLL ordering, with a few deeper-level TA/TLL reversals that may reflect the harder optimization induced by the larger TLL signature dimension.
At depth two for log-GBM, the TLL estimate matches the LS oracle, as predicted by \cref{cor:gbm-depth-two-equality}.
These findings extend to two-dimensional log-GBM: empirical LS errors agree with the multivariate oracle in \cref{cor:multivariate-gbm-oracle}, while TA conditioning yields lower reconstruction errors~(see \cref{app:multivariate-gbm-experiments}).
Furthermore, the truncated signatures of generated paths remain close to the conditioning signatures, indicating consistency with the prescribed conditioning information~(see \cref{tab:conditioning_consistency} in \cref{app:synthetic-quant}).

\subsection{Structural recovery on Gaussian processes}

\label{subsec:synthetic-structural-recovery}

We next evaluate whether probabilistic signature inversion recovers the distributional structure of Gaussian-process paths beyond matching the conditioning signature.
\cref{fig:teaser-overlay}~(a) and \cref{fig:synthetic-overlay} show qualitative comparisons across different path regimes, including smoother paths~(fBM with $H=0.7$), rougher paths~(fBM with $H=0.3$), and strongly mean-reverting paths~(OU with $\kappa=5$).
Across these regimes, probabilistic inversions follow the broad shape of the reference path while retaining local fluctuations.
Although deterministic baselines also remain close to the conditioning signature~(see \cref{tab:signature_fidelity_synthetic} in \cref{app:synthetic-quant}), their reconstructions often appear oversmoothed relative to the reference path, missing the characteristic roughness and fluctuation scales of the process.
For fBM with $H=0.3$, our method reproduces the negative short-lag autocorrelation of the reference increments, whereas the two deterministic baselines exhibit positive autocorrelation, failing to capture the antipersistent behavior.
For OU, our method closely follows the decay of the reference process-value ACF, whereas the deterministic baselines exhibit slower decay.
In both QQ plots, our method closely follows the diagonal, whereas the deterministic baselines compress the reconstructed quantiles toward zero, consistent with their suppressed local fluctuations.
Additional qualitative diagnostics are provided in \cref{app:synthetic-main-style-qual}.

\cref{tab:synthetic_structure} adds quantitative parameter-recovery evidence, showing that for univariate processes, TLL gives the lowest estimation errors across all reported parameters, while TA also outperforms the deterministic baselines.
This TLL advantage is consistent with time lead-lag augmentation retaining quadratic-variation information, particularly for estimation of $\sigma$ in log-GBM and OU.
TLL also improves estimation of both $\sigma$ and $H$ for drifted fBM.
These findings extend to two-dimensional log-GBM, where probabilistic inversion with TA achieves lower errors in both $\sigma$ and $\rho$ than the \Fourier{} and \Legendre{} baselines.
Recovery of $\rho$ assesses cross-coordinate dependence in addition to the fluctuation scale of each coordinate.

\subsection{Extension to financial data}
\label{subsec:sp500}
\cref{fig:teaser-overlay}~(b) and \cref{fig:real-overlay} show that probabilistic signature inversion extends beyond controlled Gaussian process families to S\&P~500 cumulative log-return windows.
Probabilistic inversions generate multiple paths around the reference trajectory, following its broad shape while exhibiting diverse local fluctuations.
Beyond this qualitative resemblance, the ACF of absolute returns and the QQ plot of returns in \cref{fig:real-overlay} show that the generated paths more closely reproduce volatility clustering and the return distribution than deterministic reconstructions.
The in-sample results in \cref{tab:financial_structure_id_main} complement these qualitative observations: probabilistic inversions achieve lower median errors than the deterministic baselines in stylized-fact and autocorrelation diagnostics for both S\&P~500 and foreign-exchange data, as well as cross-series correlation diagnostics for the latter.
Full in-sample and out-of-sample results for both financial datasets are provided in \cref{app:real}.

\begin{figure}[t]
\centering
\captionsetup{skip=0pt}
\setlength{\tabcolsep}{0pt}
\renewcommand{\arraystretch}{0.8}
\scalebox{1}{
\begin{tabular}{ccc}
    \multicolumn{3}{c}{\includegraphics[width=0.85\linewidth]{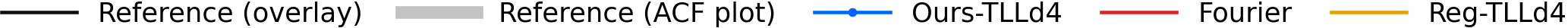}}\\
    \includegraphics[width=0.48\linewidth]{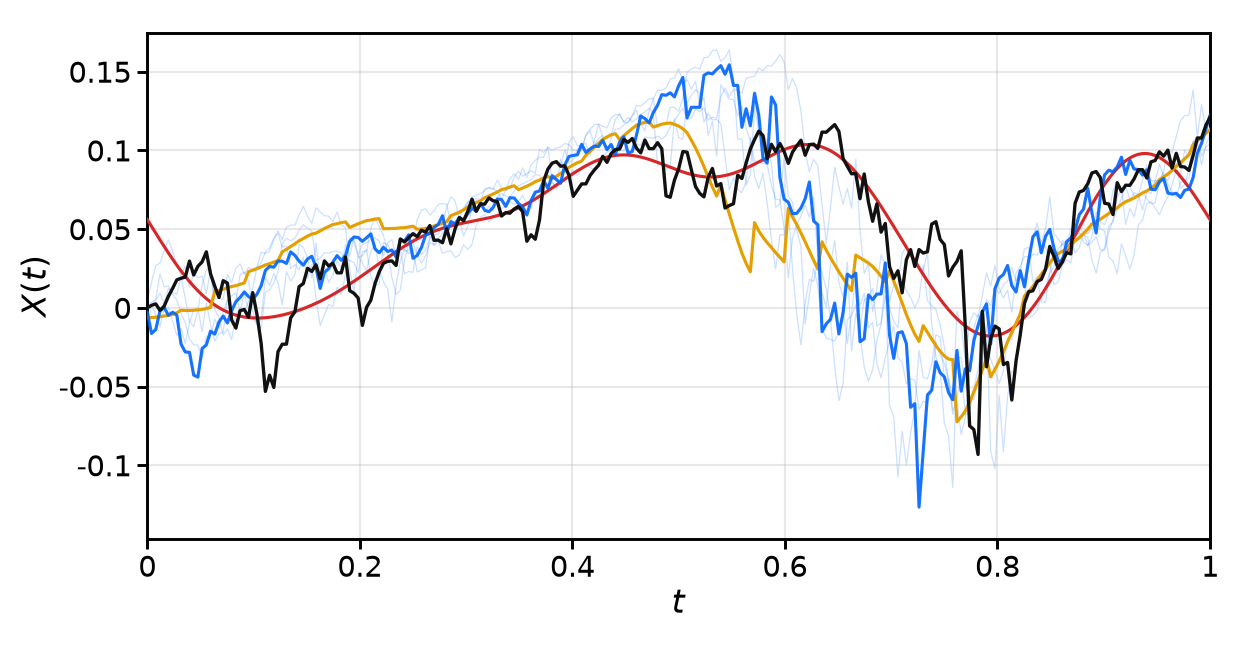}
    &
    \includegraphics[width=0.25\linewidth]{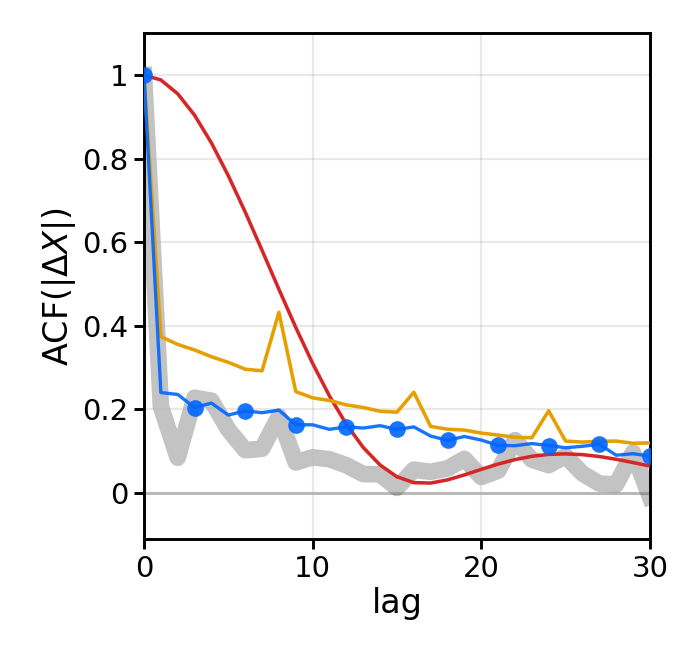}
    &
    \includegraphics[width=0.25\linewidth]{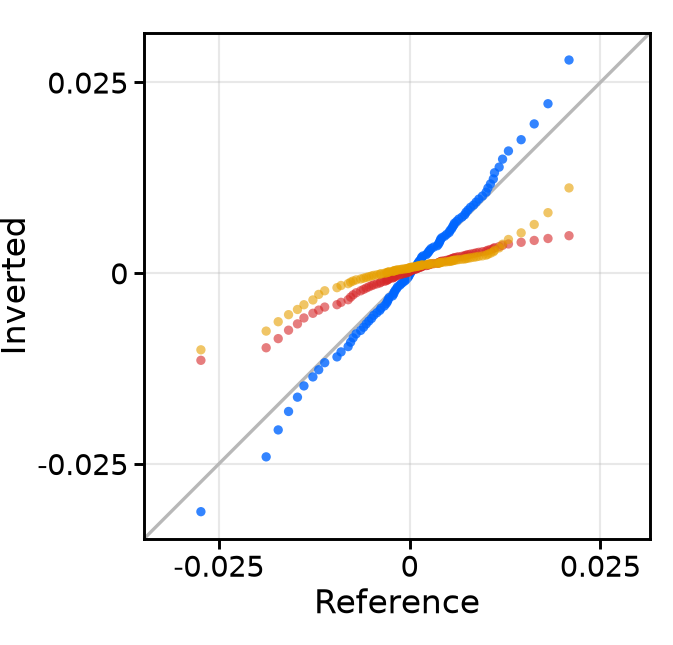} \\
\end{tabular}
}
\caption{\textbf{Path overlay, ACF plot, and QQ plot for out-of-sample S\&P~500 data.}
The ACF is computed from absolute returns.
The QQ plot compares return quantiles between generated and reference paths.
Our model uses TLL conditioning at depth $r=4$.
Additional qualitative results are provided in \cref{app:real-qualitative}.}
\label{fig:real-overlay}
\vspace{-4mm}
\end{figure}

\begin{table}[t]
\centering
\captionsetup{skip=6pt}
\caption{\textbf{Stylized-fact, autocorrelation, and cross-series correlation errors for financial data~(in-sample).}
Vol., Skew., and Kurt. denote volatility, skewness, and kurtosis.
ACF and $\rho$ denote autocorrelation and cross-series correlation, respectively, for returns $r$ or absolute returns $|r|$.
Bold and underlined entries indicate the best and second-best results, respectively, based on unrounded median errors.
Signature depths are given in \cref{app:evaluation}.}
\label{tab:financial_structure_id_main}
\renewcommand{\arraystretch}{1.1}
\setlength{\tabcolsep}{4pt}
\scalebox{0.90}{%
\begin{tabular}{@{}lcccccccccc@{}}
\toprule
& \multicolumn{5}{c}{S\&P~500} & \multicolumn{5}{c}{Foreign exchange} \\
\cmidrule(lr){2-6}\cmidrule(lr){7-11}
Method & Vol. & Skew. & Kurt. & ACF($r$) & ACF($|r|$) & Vol. & Skew. & Kurt. & ACF($|r|$) & $\rho(|r|)$ \\
\midrule
\textbf{Deterministic} \\
\Fourier & 0.11 & 0.39 & 1.89 & 0.41 & 0.23 & 0.07 & 0.50 & 2.81 & 0.35 & 0.27 \\
\Legendre & 0.12 & 1.37 & 5.24 & 0.31 & 0.18 & 0.08 & 1.38 & 4.83 & 0.29 & 0.22 \\
\Insertion & 0.13 & 0.50 & 3.19 & 0.80 & 0.50 & -- & -- & -- & -- & -- \\
\RegTAd{4} & 0.10 & 1.43 & 11.33 & 0.15 & 0.11 & -- & -- & -- & -- & -- \\
\RegTLLd{4} & 0.10 & 1.01 & 12.04 & 0.14 & 0.11 & -- & -- & -- & -- & -- \\
\midrule
\textbf{Probabilistic} \\
\OursTAd{4} & \underline{0.02} & \underline{0.30} & \underline{1.26} & \underline{0.08} & \underline{0.08} & \textbf{0.01} & \underline{0.43} & \underline{2.26} & \underline{0.07} & \underline{0.10} \\
\OursTLLd{4} & \textbf{0.02} & \textbf{0.29} & \textbf{1.19} & \textbf{0.08} & \textbf{0.08} & \underline{0.01} & \textbf{0.42} & \textbf{1.90} & \textbf{0.07} & \textbf{0.09} \\
\bottomrule
\end{tabular}
}
\end{table}

\section{Conclusion}
\label{sec:conclusion}
We reformulated truncated-signature inversion as a probabilistic problem---learning the conditional distribution of paths given a truncated signature---and quantified the irreducible ambiguity induced by truncation through a posterior-sampling reconstruction error analysis.
For Gaussian process families, we derived this error from the tractable posterior under linear-statistics conditioning, obtaining a closed form for log-GBM and numerically tractable expressions for drifted fBM and OU, yielding a concrete theoretical baseline that upper-bounds the corresponding error under richer signature conditionings.
Empirically, our signature-conditioned flow matching estimator closely matched the linear-statistics oracle errors in both univariate and bivariate settings and the corresponding posterior distributions across most evaluated univariate settings.
It also broadly tracked the expected $\sigma$-algebra hierarchy and preserved distributional and temporal structure across Gaussian processes.
Experiments on financial data further demonstrated the applicability of probabilistic signature inversion beyond parametric process families.

Beyond inversion itself, the proposed probabilistic signature inversion method may support signature-conditioned scenario generation for financial time series and help interpret the path characteristics captured by individual signature coordinates.
Future work could extend the method to conditioning on subsets of signature coordinates to facilitate such interpretation and evaluate its scalability on higher-dimensional paths.
On the theoretical side, future work includes deriving tighter bounds or exact characterizations of the reconstruction error under truncated-signature conditioning and extending the analysis beyond Gaussian process families.

\begin{ack}
This work was partly supported by the Institute of Information \& Communications Technology Planning \& Evaluation (IITP) grants funded by the Korea government~(MSIT)~[RS-2022-II220959 (No. 2022-0-00959), (Part 2) Few-Shot Learning of Causal Inference in Vision and Language for Decision Making; RS-2026-25526850, High-Efficiency Neural Networks for Artificial General Intelligence~(HERMES-Net); and RS-2021-II211343, Artificial Intelligence Graduate School Program~(Seoul National University)].
\end{ack}

\bibliographystyle{plainnat}
\bibliography{main}

\begin{thebibliography}{32}
\providecommand{\natexlab}[1]{#1}
\providecommand{\url}[1]{\texttt{#1}}
\expandafter\ifx\csname urlstyle\endcsname\relax
  \providecommand{\doi}[1]{doi: #1}\else
  \providecommand{\doi}{doi: \begingroup \urlstyle{rm}\Url}\fi

\bibitem[Barancikova et~al.(2025)Barancikova, Huang, and
  Salvi]{barancikova2025sigdiffusion}
Barbora Barancikova, Zhuoyue Huang, and Cristopher Salvi.
\newblock Sigdiffusions: Score-based diffusion models for time series via
  log-signature embeddings.
\newblock In \emph{ICLR}, 2025.

\bibitem[Boedihardjo and Geng(2015)]{boedihardjo2015nonmarkov}
Horatio Boedihardjo and Xi~Geng.
\newblock The uniqueness of signature problem in the non-markov setting.
\newblock \emph{Stochastic Processes and Their Applications}, 125\penalty0
  (12):\penalty0 4674--4701, 2015.

\bibitem[Boedihardjo et~al.(2016)Boedihardjo, Geng, Lyons, and
  Yang]{boedihardjo2016rough}
Horatio Boedihardjo, Xi~Geng, Terry Lyons, and Danyu Yang.
\newblock The signature of a rough path: Uniqueness.
\newblock \emph{Advances in Mathematics}, 293:\penalty0 720--737, 2016.

\bibitem[B{\"u}hler et~al.(2024)B{\"u}hler, Horvath, Lyons, Perez~Arribas, and
  Wood]{buhler2020market}
Hans B{\"u}hler, Blanka Horvath, Terry Lyons, Imanol Perez~Arribas, and Ben
  Wood.
\newblock A data-driven market simulator for small data environments.
\newblock In \emph{Modern Topics in Stochastic Analysis and Applications},
  pages 273--310, 2024.

\bibitem[Chang and Lyons(2019)]{chang2019insertion}
Jiawei Chang and Terry~J. Lyons.
\newblock Insertion algorithm for inverting the signature of a path.
\newblock \emph{arXiv preprint arXiv:1907.08423}, 2019.

\bibitem[Chevyrev and Kormilitzin(2016)]{chevyrev2016primer}
Ilya Chevyrev and Andrey Kormilitzin.
\newblock A primer on the signature method in machine learning.
\newblock \emph{arXiv preprint arXiv:1603.03788}, 2016.

\bibitem[{\relax DLMF}()]{NIST:DLMF}
{\relax DLMF}.
\newblock {\it NIST Digital Library of Mathematical Functions}.
\newblock \url{https://dlmf.nist.gov/}, Release 1.2.8 of 2026-09-15.
\newblock URL \url{https://dlmf.nist.gov/}.
\newblock F.~W.~J. Olver, A.~B. {Olde Daalhuis}, D.~W. Lozier, B.~I. Schneider,
  R.~F. Boisvert, C.~W. Clark, B.~R. Miller, B.~V. Saunders, H.~S. Cohl, and
  M.~A. McClain, eds.

\bibitem[Fermanian et~al.(2024)Fermanian, Chang, Lyons, and
  Biau]{fermanian2024insertion}
Adeline Fermanian, Jiawei Chang, Terry Lyons, and G{\'e}rard Biau.
\newblock The insertion method to invert the signature of a path.
\newblock In \emph{Recent Advances in Econometrics and Statistics}, pages
  575--595. Springer, 2024.

\bibitem[Flint et~al.(2016)Flint, Hambly, and Lyons]{flint2016hoff}
Guy Flint, Ben Hambly, and Terry Lyons.
\newblock Discretely sampled signals and the rough {Hoff} process.
\newblock \emph{Stochastic Processes and their Applications}, 126\penalty0
  (9):\penalty0 2593--2614, 2016.

\bibitem[Geng(2017)]{geng2017reconstruction}
Xi~Geng.
\newblock Reconstruction for the signature of a rough path.
\newblock \emph{Proceedings of the London Mathematical Society}, 114\penalty0
  (3):\penalty0 495--526, 2017.

\bibitem[Guo et~al.(2025)Guo, Wang, Zhang, and Zhao]{guo2025consistency}
Xin Guo, Binnan Wang, Ruixun Zhang, and Chaoyi Zhao.
\newblock On consistency of signature using {Lasso}.
\newblock \emph{Operations Research}, 73\penalty0 (5):\penalty0 2530--2549,
  2025.

\bibitem[Gyurk{\'o} et~al.(2013)Gyurk{\'o}, Lyons, Kontkowski, and
  Field]{gyurko2013extracting}
Lajos~Gergely Gyurk{\'o}, Terry Lyons, Mark Kontkowski, and Jonathan Field.
\newblock Extracting information from the signature of a financial data stream.
\newblock \emph{arXiv preprint arXiv:1307.7244}, 2013.

\bibitem[Hambly and Lyons(2010)]{hambly2010uniqueness}
Ben Hambly and Terry Lyons.
\newblock Uniqueness for the signature of a path of bounded variation and the
  reduced path group.
\newblock \emph{Annals of Mathematics}, 171\penalty0 (1):\penalty0 109--167,
  2010.

\bibitem[Ho et~al.(2020)Ho, Jain, and Abbeel]{ho2020denoising}
Jonathan Ho, Ajay~N. Jain, and Pieter Abbeel.
\newblock Denoising diffusion probabilistic models.
\newblock In \emph{NeurIPS}, 2020.

\bibitem[Kidger and Lyons(2021)]{kidger2021signatory}
Patrick Kidger and Terry Lyons.
\newblock {S}ignatory: Differentiable computations of the signature and
  logsignature transforms, on both {CPU} and {GPU}.
\newblock In \emph{ICLR}, 2021.

\bibitem[Kidger et~al.(2019)Kidger, Bonnier, Perez~Arribas, Salvi, and
  Lyons]{kidger2019deep}
Patrick Kidger, Patric Bonnier, Imanol Perez~Arribas, Cristopher Salvi, and
  Terry Lyons.
\newblock Deep signature transforms.
\newblock In \emph{NeurIPS}, 2019.

\bibitem[Koepf and Schmersau(1998)]{koepf1998representations}
Wolfram Koepf and Dieter Schmersau.
\newblock Representations of orthogonal polynomials.
\newblock \emph{Journal of Computational and Applied Mathematics}, 90\penalty0
  (1):\penalty0 57--94, 1998.

\bibitem[Le~Jan and Qian(2013)]{lejan2013brownian}
Yves Le~Jan and Zhongmin Qian.
\newblock Stratonovich's signatures of {Brownian} motion determine {Brownian}
  sample paths.
\newblock \emph{Probability Theory and Related Fields}, 157\penalty0
  (1--2):\penalty0 209--223, 2013.

\bibitem[Levin et~al.(2013)Levin, Lyons, and Ni]{levin2013learning}
Daniel Levin, Terry Lyons, and Hao Ni.
\newblock Learning from the past, predicting the statistics for the future,
  learning an evolving system.
\newblock \emph{arXiv preprint arXiv:1309.0260}, 2013.

\bibitem[Liao et~al.(2024)Liao, Ni, Sabate-Vidales, Szpruch, Wiese, and
  Xiao]{liao2024sigconditional}
Shujian Liao, Hao Ni, Marc Sabate-Vidales, Lukasz Szpruch, Magnus Wiese, and
  Baoren Xiao.
\newblock {Sig-Wasserstein} {GANs} for conditional time series generation.
\newblock \emph{Mathematical Finance}, 34\penalty0 (2):\penalty0 622--670,
  2024.

\bibitem[Lipman et~al.(2023)Lipman, Chen, Ben-Hamu, Nickel, and
  Le]{Lipman2023flow}
Yaron Lipman, Ricky T.~Q. Chen, Heli Ben-Hamu, Maximilian Nickel, and Matthew
  Le.
\newblock Flow matching for generative modeling.
\newblock In \emph{ICLR}, 2023.

\bibitem[Lu and Sester(2025)]{lu2025mmdsignature}
Chung~I Lu and Julian Sester.
\newblock Generative modelling of financial time series with structured noise
  and {MMD}-based signature learning.
\newblock \emph{Statistics \& Risk Modeling}, 42\penalty0 (3-4):\penalty0
  91--122, 2025.

\bibitem[Lyons and Xu(2018)]{lyons2018inverting}
Terry~J. Lyons and Weijun Xu.
\newblock Inverting the signature of a path.
\newblock \emph{Journal of the European Mathematical Society}, 20\penalty0
  (7):\penalty0 1655--1687, 2018.

\bibitem[Lyons et~al.(2007)Lyons, Caruana, and L{\'e}vy]{lyons2014rough}
Terry~J. Lyons, Michael Caruana, and Thierry L{\'e}vy.
\newblock \emph{Differential Equations Driven by Rough Paths}, volume 1908 of
  \emph{Lecture Notes in Mathematics}.
\newblock Springer, 2007.

\bibitem[Morrill et~al.(2020)Morrill, Fermanian, Kidger, and
  Lyons]{morrill2021generalised}
James Morrill, Adeline Fermanian, Patrick Kidger, and Terry Lyons.
\newblock A generalised signature method for multivariate time series feature
  extraction.
\newblock \emph{arXiv preprint arXiv:2006.00873}, 2020.

\bibitem[Ni et~al.(2021)Ni, Szpruch, Sabate-Vidales, Xiao, Wiese, and
  Liao]{ni2021sigwasserstein}
Hao Ni, Lukasz Szpruch, Marc Sabate-Vidales, Baoren Xiao, Magnus Wiese, and
  Shujian Liao.
\newblock {Sig-Wasserstein} {GANs} for time series generation.
\newblock In \emph{ICAIF}, 2021.

\bibitem[Peebles and Xie(2023)]{peebles2023scalable}
William Peebles and Saining Xie.
\newblock Scalable diffusion models with transformers.
\newblock In \emph{ICCV}, 2023.

\bibitem[Rasul et~al.(2021)Rasul, Seward, Schuster, and
  Vollgraf]{rasul2021autoregressive}
Kashif Rasul, Calvin Seward, Ingmar Schuster, and Roland Vollgraf.
\newblock Autoregressive denoising diffusion models for multivariate
  probabilistic time series forecasting.
\newblock In \emph{ICML}, 2021.

\bibitem[Tanaka et~al.(2025)Tanaka, Hashimoto, Takayanagi, Piao, Murayama, and
  Izumi]{tanaka2025cofindiff}
Yuki Tanaka, Ryuji Hashimoto, Takehiro Takayanagi, Zhe Piao, Yuri Murayama, and
  Kiyoshi Izumi.
\newblock {CoFinDiff}: Controllable financial diffusion model for time series
  generation.
\newblock In \emph{IJCAI}, 2025.

\bibitem[Tashiro et~al.(2021)Tashiro, Song, Song, and Ermon]{tashiro2021csdi}
Yusuke Tashiro, Jiaming Song, Yang Song, and Stefano Ermon.
\newblock {CSDI}: Conditional score-based diffusion models for probabilistic
  time series imputation.
\newblock In \emph{NeurIPS}, 2021.

\bibitem[Yoon et~al.(2019)Yoon, Jarrett, and van~der Schaar]{yoon2019time}
Jinsung Yoon, Daniel Jarrett, and Mihaela van~der Schaar.
\newblock Time-series generative adversarial networks.
\newblock In \emph{NeurIPS}, 2019.

\bibitem[Yuan and Qiao(2024)]{yuan2024diffusionts}
Xinyu Yuan and Yan Qiao.
\newblock Diffusion-{TS}: Interpretable diffusion for general time series
  generation.
\newblock In \emph{ICLR}, 2024.

\end{thebibliography}

\appendix
\clearpage

\section{Proofs and auxiliary derivations}
\label{app:proofs}

\subsection{Proof of \cref{prop:within-between-var}}
\label{app:proof-within-between-var}
\withinbetweenvarprop*
\begin{proof}
Let $S:=\Phi(X)$.
By definition, conditional on $S$, $X$ and $\tilde{X}$ are independent draws from $Q_\Theta(\cdot\mid S)$.
Thus, for each $t\in[0,T]$,
\begin{align}
    \EE\bigl[(X_t-\tilde{X}_t)^2\mid S\bigr]
    &= \Var(X_t\mid S)+\Var(\tilde{X}_t\mid S)
    + \bigl(\EE[X_t\mid S]-\EE[\tilde{X}_t\mid S]\bigr)^2 \\
    &= 2\Var(X_t\mid S).
\end{align}
Integrating over $t$ and taking expectation over $S$ gives \cref{eq:within-between-var}.

For the second identity, write the mixture conditional distribution through the posterior over parameters.
The law of total variance gives
\begin{align}
    \Var_{Q_\Theta}(X_t\mid S)
    =
    \EE_{\theta\mid S}\bigl[\Var_\theta(X_t\mid S)\bigr]
    + \Var_{\theta\mid S}\bigl(\EE_\theta[X_t\mid S]\bigr).
\end{align}
Substituting this identity into \cref{eq:within-between-var} yields \cref{eq:within-between-var-total}.
\end{proof}

\subsection{Kernel formulas and numerical inputs}
\label{app:kernel-formulas}
\cref{tab:kernel} lists the process-specific numerical inputs used in the kernel-based computation.

\begin{table}[h]
\centering
\captionsetup{skip=6pt}

\caption{\textbf{Kernel formulas and deterministic moments.}
The table lists process-specific quantities for the kernel-based analysis of the posterior-sampling reconstruction error: the covariance kernel $K_\theta$ and the deterministic moments $\ell_j(m_\theta)$ for $(j\geq1)$.
Here $L^{(r)}(m_\theta)$ is determined by $\ell_0(m_\theta)=m_\theta(T)$ together with these values.}
\label{tab:kernel}

\renewcommand{\arraystretch}{1.2}
\setlength{\tabcolsep}{6pt}
\scalebox{0.90}{
\begin{tabular}{ccccc}
\toprule
Process &
$m_\theta$(t) & $G^\theta_t$ &
$K_\theta(t,s)$ & $\ell_j(m_\theta)$ \\ \midrule

log-GBM &
$\nu t$ & $\sigma B_t$ &
$\sigma^2 \min(s,t)$ &
$\nu \frac{T^{j+1}}{j+1}$ \\

fBM &
$\nu t$ & $\sigma B^H_t$ &
$\frac{\sigma^2}{2} (s^{2H} + t^{2H} - |t-s|^{2H})$ &
$\nu \frac{T^{j+1}}{j+1}$ \\

OU &
$\mu (1 - e^{-\kappa t})$ & $\sigma \int_0^t e^{-\kappa(t-u)} dB_u$ &
$\frac{\sigma^2}{2\kappa} \bigl( e^{-\kappa|t-s|} - e^{-\kappa(t+s)}\bigr)$ &
$\mu\left(\frac{T^j}{j}- \int_0^T t^{j-1} e^{-\kappa t}\,dt\right)$ \\

\bottomrule
\end{tabular}
}
\end{table}

The following lemma explains how these inputs determine $k^{(r)}(t)$ and $\Sigma^{(r)}$ in \cref{prop:gauss-reg}.

\begin{lemma}[Linear-statistic covariance formulas] \label{lemma:kernel}
The quantities $k^{(r)}(t)$ and $\Sigma^{(r)}$ in \cref{prop:gauss-reg} are determined by the covariance kernel $K(\cdot,\cdot)$.
More precisely,
\begin{align}
    \Cov(G_t, \ell_0(G))
    &= K(t,T), \\
    \Cov(G_t, \ell_j(G))
    &= \int_0^T u^{j-1}K(t,u)\,du,
    \qquad j \ge 1.
\end{align}
Similarly,
\begin{align}
    \Cov(\ell_0(G), \ell_0(G))
    &= K(T,T), \\
    \Cov(\ell_0(G), \ell_j(G))
    &= \int_0^T u^{j-1}K(T,u)\,du,
    \qquad j \ge 1, \\
    \Cov(\ell_i(G), \ell_j(G))
    &= \int_0^T\!\!\int_0^T
        u^{i-1} v^{j-1} K(u,v)\,du\,dv,
    \qquad i,j \ge 1.
\end{align}
\end{lemma}

\begin{proof}
By Fubini,
\begin{align}
    \Cov(G_t, \ell_0(G))
    &= \Cov(G_t,G_T) = K(t,T), \\
    \Cov(G_t, \ell_j(G))
    &= \Cov\!\left(G_t, \int_0^T u^{j-1}G_u\,du\right)
    = \int_0^T u^{j-1} K(t,u)\,du.
\end{align}
Similarly,
\begin{align}
    \Cov(\ell_0(G), \ell_0(G))
    &= K(T,T), \\
    \Cov(\ell_0(G), \ell_j(G))
    &= \Cov(G_T, \ell_j(G))
    = \int_0^T u^{j-1}K(T,u)\,du, \\
    \Cov(\ell_i(G), \ell_j(G))
    &= \Cov\!\left(
        \int_0^T u^{i-1}G_u\,du,
        \int_0^T v^{j-1}G_v\,dv
    \right) \\
    &= \int_0^T\!\!\int_0^T
        u^{i-1} v^{j-1} K(u,v)\,du\,dv.
\end{align}
\end{proof}

\subsection{Proof of \cref{lemma:gaussian-projection}}
\label{app:proof-gaussian-projection}
\gaussianprojectionlemma*
\begin{proof}
Since $I$ is linear, $I(f) - I(\Pi_M f) = I(f - \Pi_M f)$.
For any $g \in M$, the It\^{o} isometry gives
\begin{align}
    \EE\left[ I(f - \Pi_M f)\, I(g) \right]
    = \langle f - \Pi_M f,\, g \rangle_{L^2([0,T])} = 0.
\end{align}
Hence $I(f - \Pi_M f)$ is uncorrelated with every $I(g)$, $g \in M$.
Since these are jointly Gaussian, this implies independence, so $I(\Pi_M f)$ is the conditional expectation.
The second identity follows from the It\^{o} isometry.
\end{proof}

\subsection{Proof of \cref{prop:gbm-error}}
\label{app:proof-gbm-error}
We first state an auxiliary residual formula used below.

\begin{lemma}[Legendre residual formula]\label{lemma:legendre}
Define
\begin{align}
    a_r
    := \int_0^1 \| \mathbf{1}_{[0,u]} - \Pi_{\mathcal{H}_{r}} \mathbf{1}_{[0,u]} \|^2_{L^2([0,1])}\, du,
\end{align}
where $\mathcal{H}_{r} := \mathrm{span}
\{1,t,\ldots,t^{r-1}\}$.
Then,
\begin{align}
    a_r = \frac{r}{2(2r-1)(2r+1)}.
\end{align}
\end{lemma}

\begin{proof}
Let $L_n(x) = P_n(2x - 1)$ be the shifted Legendre polynomials
and $e_n(x) := \sqrt{2n+1}\, L_n(x)$ the $L^2([0,1])$-orthonormal basis.
From the Legendre identity,
\begin{align}
    \frac{d}{dx}\bigl(L_{n+1}(x) - L_{n-1}(x)\bigr) = 2(2n+1)\, L_n(x).
\end{align}
Integrating from $0$ to $u$,
\begin{align}
    c_n(u)
    := \int_0^u e_n(x)\, dx
    = \frac{L_{n+1}(u) - L_{n-1}(u)}{2\sqrt{2n+1}}.
\end{align}
Since $\mathbf{1}_{[0,u]} = \sum_{n=0}^{\infty} c_n(u)\, e_n$,
the projection residual has squared norm $\sum_{n=r}^{\infty} c_n(u)^2$.
Computing
\begin{align}
    \int_0^1 c_n(u)^2\, du
    = \frac{1}{4(2n+1)}
      \left( \frac{1}{2n+3} + \frac{1}{2n-1} \right)
    = \frac{1}{8}
      \left( \frac{1}{2n-1} - \frac{1}{2n+3} \right),
\end{align}
the sum telescopes:
\begin{align}
    a_r
    = \sum_{n=r}^{\infty} \frac{1}{8}
      \left( \frac{1}{2n-1} - \frac{1}{2n+3} \right)
    = \frac{1}{8}
      \left( \frac{1}{2r-1} + \frac{1}{2r+1} \right)
    = \frac{r}{2(2r-1)(2r+1)}.
\end{align}
\end{proof}

\gbmerrorprop*
\begin{proof}
By Fubini, each linear statistic component satisfies
\begin{align}
    \ell_j(G)
    = \int_0^T t^{j-1} \sigma B_t\,dt
    = \sigma I(f_j), \qquad
    f_j(t) := \frac{T^j - t^j}{j},\quad j \geq 1,
\end{align}
and $\ell_0(G) = \sigma B_T = \sigma I(f_0)$ with $f_0 \equiv 1$.
Hence
\begin{align}
    S := L^{(r)}(X)
    = \nu\beta
    + \sigma(I(f_0), \ldots, I(f_{r-1}))^\top,
    \qquad \beta := (T, \ldots, T^r/r)^\top.
\end{align}
Therefore, for a fixed $\theta$ and
$\mathcal{H}_{r} := \mathrm{span}\{f_0,\ldots,f_{r-1}\}
= \mathrm{span}\{1,t,\ldots,t^{r-1}\}$,
\begin{align}
    \sigalg(S)= \sigalg(I(f): f \in \mathcal{H}_r).
\end{align}
From $X_t = \nu t + \sigma I(\mathbf{1}_{[0,t]})$ and \cref{lemma:gaussian-projection},
\begin{align}
    \EE_{\theta}[X_t \mid S]
    &= \nu t + \sigma I(\Pi_{\mathcal{H}_{r}} \mathbf{1}_{[0,t]}) \\
    &= \nu t + \sigma I(c_r(t)^\top (f_0,\ldots,f_{r-1})) \\
    &= \nu t + c_r(t)^\top(S - \nu\beta) \\
    &= c_r(t)^\top S + \nu (t - c_r(t)^\top \beta) \\
    &= c_r(t)^\top S.
\end{align}
The identity $c_r(t)^\top \beta=t$ follows from $1 \equiv f_0 \in \mathcal{H}_r$:
\begin{align}
    t
    &= \langle\mathbf{1}_{[0,t]}, \mathbf{1} \rangle \\
    &= \langle\Pi_{\mathcal{H}_r}\mathbf{1}_{[0,t]}, \mathbf{1} \rangle \\
    &= \langle c_{r}(t)^\top(f_0,\ldots,f_{r-1}), \mathbf{1} \rangle \\
    &= c_{r}(t)^\top
    \bigl(\langle f_0, \mathbf{1} \rangle, \ldots,
    \langle f_{r-1}, \mathbf{1} \rangle\bigr) \\
    &= c_{r}(t)^\top \beta.
\end{align}
The resulting conditional mean is independent of $\theta = (\mu,\sigma)$, so the between-parameter conditional-mean variance term in \cref{prop:within-between-var} vanishes.
For the variance, \cref{lemma:gaussian-projection} and \cref{lemma:legendre} prove
\begin{align}
    \int_0^T \Var(X_t \mid S)\,dt
    &= \sigma^2 \int_0^T
        \|\mathbf{1}_{[0,t]} - \Pi_{\mathcal{H}_{r}}\mathbf{1}_{[0,t]}\|^2\,dt \\
    &= \sigma^2 T^2 \int_0^1
        \|\mathbf{1}_{[0,u]} - \Pi_{\mathcal{H}_{r}}\mathbf{1}_{[0,u]}\|^2\,du \\
    &= a_r\sigma^2 T^2.
\end{align}
The second equality uses the change of variables $t=Tu$ and the corresponding identification of polynomial spans.
The tower property yields
$\mathcal{E}_\Theta(L^{(r)}) = 2a_r T^2 \EE_\pi[\sigma^2]$.
\end{proof}

\subsection{Proof of \cref{cor:multivariate-gbm-oracle}}
\label{app:multivariate-gbm}

\begingroup
\theoremstyle{compactthm}
\newtheorem*{multivariategbmrestatement}{\cref*{cor:multivariate-gbm-oracle}}
\begin{multivariategbmrestatement}[Multivariate log-GBM oracle]
Let $B$ be a standard $d$-dimensional Brownian motion and consider $X_t=\nu t+\Sigma^{1/2}B_t$, where $\nu\in\mathbb{R}^d$ and $\Sigma$ is a positive semidefinite diffusion covariance matrix.
Let $L_d^{(r)}$ apply the linear statistic in \cref{def:lin-stat} to each coordinate, and extend \cref{def:recon-error} using the squared Euclidean path loss $\int_0^T\|X_t-\tilde X_t\|_2^2 dt$, which sums the coordinate-wise squared errors.
For $r\geq1$ and a parameter prior $\pi_\Theta$ with finite second moments, the oracle posterior-sampling reconstruction error is
\begin{align}
    \mathcal E_\Theta\bigl(L_d^{(r)}\bigr)
    = \frac{rT^2}{(2r-1)(2r+1)}
      \EE_{\pi_\Theta}\bigl[\operatorname{tr}(\Sigma)\bigr].
    \label{eq:multivariate-gbm-oracle}
\end{align}
\end{multivariategbmrestatement}
\endgroup

\begin{proof}
The squared Euclidean loss is additive across coordinates, so
\begin{align}
    \mathcal E_\Theta\bigl(L_d^{(r)}\bigr)
    =\sum_{j=1}^d\EE\left[
        \int_0^T (X_t^{(j)}-\tilde X_t^{(j)})^2dt
    \right].
\end{align}
Since the same temporal projection is used in every coordinate, each scalar residual is independent of the full LS observation at fixed $\theta$.
Thus, the scalar argument in \cref{prop:gbm-error} applies to each term with $\sigma^2=\Sigma_{jj}$, and the conditional mean remains independent of $\theta$.
Summing the resulting scalar errors gives
\begin{align}
    \mathcal E_\Theta\bigl(L_d^{(r)}\bigr)
    &=2a_rT^2\sum_{j=1}^d\EE_{\pi_\Theta}[\Sigma_{jj}]
    =2a_rT^2\EE_{\pi_\Theta}[\operatorname{tr}(\Sigma)].
\end{align}
\end{proof}

\subsection{Piecewise-linear correction to the LS oracle for log-GBM}
\label{app:pl-gbm-oracle}

\begin{proposition}[Piecewise-linear LS oracle for log-GBM]
\label{prop:pl-gbm-oracle}
Let $X$ be the log-GBM process in \cref{eq:logGBM}, with parameter prior $\pi_\Theta$ satisfying $\EE[\nu^2+\sigma^2]<\infty$.
For an integer $n\geq1$, let $X^{\mathrm{lin}}$ denote the piecewise-linear interpolation of the observations $(X_{t_i})_{i=0}^n$ at $t_i=iT/n$, as in \cref{def:signature_observation}.
For $1\leq r\leq n$, there exists a deterministic function $c_{r,n}:[0,T]\to\mathbb R^r$ such that
\begin{align}
    \EE_\theta\bigl[X_t^{\mathrm{lin}}\mid L^{(r)}(X^{\mathrm{lin}})\bigr]
    =c_{r,n}(t)^\top L^{(r)}(X^{\mathrm{lin}}),\qquad t\in[0,T].
    \label{eq:pl-gbm-cond-mean}
\end{align}
In particular, the conditional mean is determined entirely by the observed statistic $L^{(r)}(X^{\mathrm{lin}})$ and does not depend on $\theta=(\mu,\sigma)$.
Thus, the between-parameter conditional-mean variance term in \cref{prop:within-between-var} vanishes.
Moreover, the integrated conditional variance is
\begin{align}
    \int_0^T\Var_\theta\bigl(X_t^{\mathrm{lin}}\mid L^{(r)}(X^{\mathrm{lin}})\bigr)dt = a_{r,n}^{\mathrm{PL}}\sigma^2T^2,
    \qquad a_{r,n}^{\mathrm{PL}} :=a_r\left(1+\frac{r^2-1}{3n^2}\right)-\frac1{6n}, \label{eq:pl-gbm-int-var}
\end{align}
for $a_r=\frac{r}{2(2r-1)(2r+1)}$.
Let $\mathcal E_{\Theta,n}^{\mathrm{PL}}(L^{(r)})$ denote the oracle posterior-sampling reconstruction error in \cref{def:recon-error}, with both the reconstruction target and the conditioning statistic computed from $X^{\mathrm{lin}}$.
Substituting these identities into \cref{prop:within-between-var} yields the oracle error
\begin{align}
    \mathcal E_{\Theta,n}^{\mathrm{PL}}\bigl(L^{(r)}\bigr)
    =\mathcal E_\Theta\bigl(L^{(r)}\bigr)
      \left(1+\frac{r^2-1}{3n^2}\right)-\frac{T^2\EE[\sigma^2]}{3n}.
    \label{eq:pl-gbm-oracle}
\end{align}
For fixed $r$, the piecewise-linear oracle differs from the continuous-time oracle by $O(n^{-1})$ as $n\to\infty$.
\end{proposition}

\begin{proof}
Fix $\theta$ and write $S=L^{(r)}(X^{\mathrm{lin}})$.
Let $P_n$ be the $L^2([0,T])$-orthogonal projection onto functions that are constant on each grid cell, and use $I$, $f_0,\ldots,f_{r-1}$, and $\beta$ from \cref{app:proof-gbm-error}.
Set $g_p=P_nf_p$ and $V_{r,n}=\operatorname{span}\{g_0,\ldots,g_{r-1}\}$.
Piecewise-linear interpolation preserves the deterministic linear drift, so
\begin{align}
    X_t^{\mathrm{lin}}&=\nu t+\sigma I\bigl(P_n\mathbf1_{[0,t]}\bigr),
    &S&=\nu\beta+\sigma\bigl(I(g_0),\ldots,I(g_{r-1})\bigr)^\top.
\end{align}
Write $\Pi_{V_{r,n}}P_n\mathbf1_{[0,t]}=c_{r,n}(t)^\top(g_0,\ldots,g_{r-1})^\top$.
The functions $g_p$ and hence the projection coefficients depend only on $r,n,T$.
Since $g_0=1\in V_{r,n}$ and $\beta_p=\langle g_p,1\rangle$, orthogonal projection preserves
\begin{align}
    c_{r,n}(t)^\top\beta
    =\left\langle\Pi_{V_{r,n}}P_n\mathbf1_{[0,t]},1\right\rangle
    =\left\langle P_n\mathbf1_{[0,t]},1\right\rangle=t.
\end{align}
For $\sigma>0$, \cref{lemma:gaussian-projection} therefore gives
\begin{align}
    \EE_\theta[X_t^{\mathrm{lin}}\mid S]
    =\nu t+c_{r,n}(t)^\top(S-\nu\beta)
    =c_{r,n}(t)^\top S.
\end{align}
The same identity holds when $\sigma=0$, since then $X_t^{\mathrm{lin}}=\nu t$ and $S=\nu\beta$.
This proves \cref{eq:pl-gbm-cond-mean} and makes the between-parameter conditional-mean variance zero.

The conditional covariance is unaffected by the deterministic drift.
Using \cref{lemma:gaussian-projection} and rescaling to $[0,1]$ gives
\begin{align}
    \int_0^T\Var_\theta(X_t^{\mathrm{lin}}\mid S)dt
    &=\sigma^2T^2\int_0^1
      \left\|P_n\mathbf1_{[0,u]}-\Pi_{V_{r,n}}P_n\mathbf1_{[0,u]}\right\|_{L^2([0,1])}^2du,
    \label{eq:pl-gbm-projection}
\end{align}
where $P_n$ and $V_{r,n}$ now refer to the rescaled grid.

To evaluate the projection residual, let $Q_k$ be the degree-$k$ Hahn polynomial with parameters $(0,0,n-1)$, normalized by $Q_k(0)=1$.
The cellwise-constant functions $s\mapsto Q_k(\lfloor ns\rfloor)$ for $0\leq k<r$ form an orthogonal basis of $V_{r,n}$, since cell averaging preserves polynomial degree in the cell index.
These polynomials are orthogonal on $\{0,\ldots,n-1\}$ with uniform weights, and their standard norm formula~\citep[Table~18.19.1]{NIST:DLMF} is
\begin{align}
    H_k:=\sum_{i=0}^{n-1}Q_k(i)^2
    =\frac{(n+k)!(n-k-1)!}{(2k+1)((n-1)!)^2},\qquad 0\leq k\leq n-1.
    \label{eq:pl-discrete-norm}
\end{align}
For $0\leq k<n$, these functions form a basis of the cellwise-constant functions and have squared $L^2([0,1])$ norms $H_k/n$.
Writing $F_k(u):=\int_0^u Q_k(\lfloor ns\rfloor)ds=\langle P_n\mathbf1_{[0,u]},Q_k(\lfloor n\cdot\rfloor)\rangle$, orthogonality gives
\begin{align}
    \int_0^1\left\|P_n\mathbf1_{[0,u]}-\Pi_{V_{r,n}}P_n\mathbf1_{[0,u]}\right\|^2du
    &=\sum_{k=r}^{n-1}E_k,
    &E_k&:=\frac{n}{H_k}\int_0^1F_k(u)^2du.
\end{align}
Thus, only the components of degree $k\geq r$ contribute to the residual.

For $1\leq k\leq n-1$, the antidifference identity of~\citet[Corollary~1, p.~65]{koepf1998representations}, expressed in this normalization, gives
\begin{align}
    nF_k(i/n)=\sum_{j=0}^{i-1}Q_k(j)
    = \frac{n+k}{2(2k+1)}Q_{k-1}(i)-\frac12Q_k(i) -\frac{n-1-k}{2(2k+1)}Q_{k+1}(i),
    \label{eq:pl-partial-sum}
\end{align}
for $i=0,\ldots,n-1$, with the last term omitted when $k=n-1$.
Also, $F_k(0)=F_k(1)=0$ because $Q_k$ is orthogonal to the constant polynomial.

Since $F_k$ is piecewise linear, integration of its square on each cell gives
\begin{align}
    E_k
    &=\frac1{H_k}\sum_{i=0}^{n-1}F_k(i/n)^2
      -\frac1{6H_k}\sum_{i=0}^{n-1}
       \bigl(F_k((i+1)/n)-F_k(i/n)\bigr)^2 \notag\\
    &=\frac1{n^2}\left(
      \frac{n^2-k^2}{4(2k-1)(2k+1)}+\frac1{12}
      +\frac{n^2-(k+1)^2}{4(2k+1)(2k+3)}\right) \notag\\
    &=\frac{3n^2-k(k+1)}{6n^2(2k-1)(2k+3)}.
    \label{eq:pl-mode-energy}
\end{align}
The first equality uses the cellwise identity $\int_0^h(a+(b-a)t/h)^2dt=h(a^2+ab+b^2)/3$ and the zero endpoints.
The second follows from \cref{eq:pl-discrete-norm,eq:pl-partial-sum}, orthogonality, and $F_k((i+1)/n)-F_k(i/n)=Q_k(i)/n$.

Using \cref{eq:pl-gbm-int-var}, direct substitution in \cref{eq:pl-mode-energy} gives $E_k=a_{k,n}^{\mathrm{PL}}-a_{k+1,n}^{\mathrm{PL}}$ and $a_{n,n}^{\mathrm{PL}}=0$, so
\begin{align}
    \sum_{k=r}^{n-1}E_k
    =a_{r,n}^{\mathrm{PL}}
    =a_r\left(1+\frac{r^2-1}{3n^2}\right)-\frac1{6n}.
\end{align}
This identity also holds for $r=n$, when the residual is zero and the sum is empty.
Multiplying by $\sigma^2T^2$ in \cref{eq:pl-gbm-projection} proves \cref{eq:pl-gbm-int-var}.
Finally, \cref{prop:within-between-var} and the parameter-independent conditional mean give
\begin{align}
    \mathcal E_{\Theta,n}^{\mathrm{PL}}\bigl(L^{(r)}\bigr)
    =2\EE_S\EE_{\theta\mid S}\bigl[a_{r,n}^{\mathrm{PL}}\sigma^2T^2\bigr]
    =2a_{r,n}^{\mathrm{PL}}T^2\EE[\sigma^2].
\end{align}
Substituting \cref{eq:pl-gbm-int-var,eq:bayes-gbm} yields \cref{eq:pl-gbm-oracle}.
\end{proof}

\clearpage

\section{Implementation details}
\label{app:implementation}

\subsection{Training}
\label{app:training}

\paragraph{Datasets.}
\cref{tab:train-datasets} summarizes the training datasets used for the Gaussian-process and financial-data experiments.
Both S\&P~500 and FX paths are 252-trading-day cumulative log-return windows, multiplied by $10$ during training.
FX uses the paired GBP/USD and JPY/USD series from the Exchange Rate dataset spanning 1990--2016, obtained from the SigDiffusions repository~\citep{barancikova2025sigdiffusion}.

\begin{table}[h!]
\centering
\captionsetup{skip=6pt}
\caption{\textbf{Training dataset configurations.}
All paths use $T=1$.
For Gaussian processes, parameters are sampled independently from $\Theta$ for each path.
For 2D log-GBM, $j=1,2$ indexes coordinates and $\rho$ denotes Brownian correlation.
Length counts path points, including the initial point.}
\label{tab:train-datasets}
\renewcommand{\arraystretch}{1.4}
\setlength{\tabcolsep}{5pt}
\small
\begin{tabular}{lllcc}
\toprule
Type & Dataset & Configuration & Length & \# Paths \\ \midrule
\multirow{4}{*}{Gaussian processes} & Log-GBM  & $\mu \in [1.5,2.5],\;\sigma \in [1.5,2.5]$                      & $1001$ & $100$k \\
  & fBM      & $\mu \in [1.5,2.5],\;\sigma \in [1.5,2.5],\;H \in [0.25,0.75]$ & $1001$ & $100$k \\
  & OU       & $\mu=3,\;\sigma \in [1.5,2.5],\;\kappa \in [0.5,5.0]$           & $1001$ & $100$k \\
  & 2D log-GBM & $\mu_j,\sigma_j \in [1.5,2.5],\rho\sim\mathrm{Uniform}(-0.7,0.7)$ & $1001$ & $100$k \\ \midrule
\multirow{2}{*}{Financial data} & S\&P~500 & 2009-05 $\sim$ 2022-12~(Stride 1)                               & $253$  & $3190$ \\
  & FX (2D) & 1990--2016~(first $80\%$; stride 1) & $253$ & $5818$ \\
\bottomrule
\end{tabular}
\end{table}

\paragraph{Model architecture.}
We adapt DiT~\citep{peebles2023scalable} to one-dimensional paths and use it as the flow matching velocity network.
Paths are split into patches of size $8$, projected to $128$-dimensional tokens, processed by $7$~Transformer blocks with $8$ attention heads and FFN hidden dimension $512$, and unpatchified back to the original path resolution.
Sequences whose lengths are not divisible by $8$ are zero-padded before patchification and cropped back after unpatchification.
The time and log-signature embeddings are combined and injected through AdaLN-Zero modulation.
For regression ablations, we use the same backbone and append one convolutional output layer.

\paragraph{Optimization.}
We train with AdamW using the PyTorch default hyperparameters, learning rate $2\times10^{-4}$, batch size $64$, and no learning-rate scheduler.
All learned models are trained for $100$ epochs on a single A5000 GPU, taking approximately $3$ hours per run.

\paragraph{Signature preprocessing.}
Log-signatures are computed with \texttt{iisignature}\footnote{\url{https://github.com/bottler/iisignature}}.
LS conditioning uses the TA log-signature coordinates corresponding to the linear statistics in \cref{def:lin-stat}.
For TLL, the lead-lag construction introduces repeated coordinate paths, which produce duplicate or linearly dependent log-signature coordinates.
For example, by the TLL construction in \cref{def:augmentation_tll}, iterated-integral coordinates associated with different words, such as $(1,2)$ and $(1,3)$, can evaluate identically or become linearly dependent.
We therefore precompute a reduced TLL coordinate set on a fixed reference batch of $1000$ randomly generated log-GBM paths by scanning the TLL log-signature coordinates one at a time and greedily retaining only those that add a new independent direction.
Replacing this reference batch with drifted fBM or OU paths yields the same retained set.
This preprocessing step is not tuned using downstream reconstruction performance; the reference batch is used only to detect algebraic or numerical redundancy in the TLL log-signature coordinates.
The selected TLL indices are then fixed and reused for training and inversion.
Finally, we apply the level-wise scaling map $x \mapsto \sign(x)\log(1+k!\,|x|)$ to each retained level-$k$ conditioning coordinate.
The factor $k!$ compensates for the factorial decay of level-$k$ signature terms~\citep{morrill2021generalised}, while the signed logarithm is nearly linear for small magnitudes and suppresses only very large coordinates.
\cref{tab:conditioning-dimensions} reports the resulting conditioning-vector dimensions.

\begin{table}[h!]
\centering
\captionsetup{skip=6pt}
\caption{\textbf{Conditioning-vector dimensions by depth.}
Each entry gives the number of conditioning coordinates at depth $r$.
TLL full gives the raw TLL log-signature dimension, and TLL reduced gives the dimension after greedy removal of duplicate or linearly dependent coordinates.}
\label{tab:conditioning-dimensions}
\renewcommand{\arraystretch}{1.2}
\setlength{\tabcolsep}{5pt}
\small
\begin{tabular}{lcccccc}
\toprule
Conditioning & $r=1$ & $r=2$ & $r=3$ & $r=4$ & $r=5$ & $r=6$ \\
\midrule
LS & $1$ & $2$ & $3$ & $4$ & $5$ & $6$ \\
TA & $2$ & $3$ & $5$ & $8$ & $14$ & $23$ \\
TLL full & $3$ & $6$ & $14$ & $32$ & $80$ & $196$ \\
TLL reduced & $2$ & $4$ & $9$ & $19$ & $43$ & $93$ \\
\bottomrule
\end{tabular}
\end{table}

\subsection{Evaluation}
\label{app:evaluation}

\paragraph{Evaluation datasets.}
\cref{tab:evaluation-datasets} summarizes the evaluation datasets used for the quantitative metrics and qualitative diagnostics defined in \cref{app:metrics}.
The S\&P~500 and FX in-sample splits are subsampled training-period diagnostics rather than held-out generalization sets.
The S\&P~500 splits are evaluated on the original cumulative log-return scale.

\begin{table}[h]
\centering
\captionsetup{skip=6pt}
\caption{\textbf{Evaluation dataset configurations.}
Length counts path points, including the initial point.
The full FX dataset spans 1990--2016.}
\label{tab:evaluation-datasets}
\renewcommand{\arraystretch}{1.4}
\setlength{\tabcolsep}{3.5pt}
\small
\scalebox{0.9}{%
\begin{tabular}{lllccc}
\toprule
Type & Dataset & Configuration & Evaluation & Length & \# Paths \\ \midrule
\multirow{13}{*}{Gaussian processes} & \multirow{4}{*}{Log-GBM}
  & $\mu \in [1.5,2.5],\;\sigma \in [1.5,2.5]$         & Quant. & \multirow{4}{*}{$1001$} & $1000$ \\
  &   & $\mu=2.0,\;\sigma=1.6$                               & Qual. & & $100$ \\
  &   & $\mu=2.0,\;\sigma=2.0$                               & Qual. & & $100$ \\
  &   & $\mu=2.0,\;\sigma=2.4$                               & Qual. & & $100$ \\
\cmidrule(lr){2-6}
  & \multirow{4}{*}{fBM}
  & $\mu \in [1.5,2.5],\;\sigma \in [1.5,2.5],\;H \in [0.25,0.75]$ & Quant. & \multirow{4}{*}{$1001$} & $1000$ \\
  &   & $\mu=2.0,\;\sigma=2.0,\;H=0.30$                     & Qual. & & $100$ \\
  &   & $\mu=2.0,\;\sigma=2.0,\;H=0.50$                     & Qual. & & $100$ \\
  &   & $\mu=2.0,\;\sigma=2.0,\;H=0.70$                     & Qual. & & $100$ \\
\cmidrule(lr){2-6}
  & \multirow{4}{*}{OU}
  & $\mu=3,\;\sigma \in [1.5,2.5],\;\kappa \in [0.5,5.0]$ & Quant. & \multirow{4}{*}{$1001$} & $1000$ \\
  &   & $\mu=3,\;\sigma=2.0,\;\kappa=0.5$                    & Qual. & & $100$ \\
  &   & $\mu=3,\;\sigma=2.0,\;\kappa=2.0$                    & Qual. & & $100$ \\
  &   & $\mu=3,\;\sigma=2.0,\;\kappa=5.0$                    & Qual. & & $100$ \\
\cmidrule(lr){2-6}
  & 2D log-GBM & $\mu_j,\sigma_j \in [1.5,2.5],\rho\sim\mathrm{Uniform}(-0.7,0.7)$ & Quant. & $1001$ & $1000$ \\
\midrule
\multirow{4}{*}{Financial data} & \multirow{2}{*}{S\&P~500}
  & In-sample: 2009-05 $\sim$ 2022-12~(Stride 63)  & Quant./Qual. & \multirow{2}{*}{$253$} & $51$  \\
  &   & Out-of-sample: 2023-01 $\sim$ 2025-12~(Stride 5)  & Quant./Qual. &                        & $100$ \\
\cmidrule(lr){2-6}
  & \multirow{2}{*}{FX (2D)}
  & In-sample: first $80\%$~(stride 63) & Quant. & \multirow{2}{*}{$253$} & $93$ \\
  &   & Out-of-sample: last $20\%$~(stride 5) & Quant. & & $254$ \\
\bottomrule
\end{tabular}
}
\end{table}

\paragraph{Sampling protocol.}
Sampling uses $1000$ explicit Euler steps.
For quantitative evaluation, we generate $30$ samples for each of $1000$ reference paths, which takes less than $1$ hour.

\paragraph{Deterministic baselines.}
Deterministic baselines return one reconstructed path for each target path.
Regression ablations and \Insertion{} are used only for univariate experiments; \Fourier{} and \Legendre{} are used for both univariate and multivariate experiments.
\begin{itemize}
    \item Regression ablations use the same conditioning vector and backbone, but produce a single reconstruction through the supervised $\ell_2$ regression objective. We use \RegTAd{6} and \RegTLLd{6} for univariate Gaussian-process experiments, and \RegTAd{4} and \RegTLLd{4} for S\&P~500.
    \item For univariate experiments, \Fourier~(order $n=6$) uses the signature of the augmented path $(t,\sin(t),\cos(t)-1,x(t))$ up to depth $8$; an order-$n$ expansion requires depth $n+2$.
    For multivariate experiments, we use order $n=4$.
    Our implementation follows the public SigDiffusions codebase\footnote{\url{https://github.com/Barb0ra/SigDiffusions}}~\citep{barancikova2025sigdiffusion}.
    \item For univariate experiments, \Legendre~(order $n=10$) uses the signature of the time-augmented path $(t,x(t))$ up to depth $12$; an order-$n$ expansion requires depth $n+2$.
    For multivariate experiments, we use order $n=8$.
    We implement this baseline directly from \citet{barancikova2025sigdiffusion}.
    \item \Insertion~($n=12$ piecewise-linear pieces) uses the signature of the time-augmented path $(t,x(t))$ up to depth $12$; a path with $n$ piecewise-linear pieces requires depth $n$.
    Our implementation adapts the public Signatory codebase\footnote{\url{https://github.com/patrick-kidger/signatory}}~\citep{kidger2021signatory}.
\end{itemize}

\subsection{Evaluation metrics}
\label{app:metrics}

Let $X_i$ denote the $i$-th observed test path and $\tilde{X}_i$ an inverted path associated with $X_i$; without the sample index, $\tilde{X}_i$ denotes one sampled inversion for probabilistic methods and the single reconstruction for deterministic methods.
When all probabilistic samples are used, $\tilde{X}_{i,j}$ denotes the $j$-th sample, $j=1,\ldots,M$, with $M=30$.
Let $\Phi_r$ denote the conditioning statistic used by a method at depth $r$, and let $\|\cdot\|_2^2$ denote the discrete path norm used to approximate the integrated squared error in \cref{eq:recon-error}.

\paragraph{Posterior-sampling reconstruction and conditioning diagnostics.}
Motivated by the posterior-sampling reconstruction-error analysis in \cref{sec:error_analysis}, these diagnostics report empirical reconstruction error of the learned sampler against the oracle posterior-sampling reconstruction error, measure inter-sample variability, and verify consistency with the conditioning statistic.
\begin{itemize}
    \item \textbf{Empirical reconstruction error.}
    We report the empirical reconstruction error
    \begin{align*}
        \frac{1}{NM}
        \sum_{i=1}^{N}\sum_{j=1}^{M}
        \|X_i-\tilde{X}_{i,j}\|_2^2 .
    \end{align*}
    For LS conditioning, this is compared with the oracle posterior-sampling reconstruction error $\mathcal{E}_\Theta(L^{(r)})$.
    \item \textbf{Inter-sample spread ratio.}
    We define the inter-sample spread as
    \begin{align*}
        \frac{1}{N\binom{M}{2}}
        \sum_{i=1}^{N}
        \sum_{1\leq j<k\leq M}
        \|\tilde{X}_{i,j}-\tilde{X}_{i,k}\|_2^2 .
    \end{align*}
    The reported spread ratio divides this quantity by the corresponding empirical reconstruction error.
    A ratio near $1$ provides evidence that the sampled inversions are not strongly contracted relative to the empirical reconstruction-error scale.
    \item \textbf{Conditioning consistency.}
    For each $\tilde{X}_{i,j}$ generated conditional on $\Phi_r(X_i)$, we report the median relative MSE of the conditioning statistic,
    \begin{align*}
        \operatorname*{median}_{i,j}
        \frac{\|\Phi_r(\tilde{X}_{i,j})-\Phi_r(X_i)\|_2^2}
             {\|\Phi_r(X_i)\|_2^2}
        \times 100\%.
    \end{align*}
\end{itemize}

\paragraph{Process-structure diagnostics.}
These metrics evaluate whether the generated paths preserve process-level structure for Gaussian process families.
\begin{itemize}
    \item \textbf{Parameter estimation error.}
    We apply process-specific estimators to the inverted paths and compare them with the ground-truth parameters of the conditioning paths.
    The $\sigma$ estimators for log-GBM and drifted fBM with known $H$ are Gaussian MLEs; for OU, $\sigma$ is estimated by a Gaussian transition-residual MLE with $\kappa$ and $\mu$ fixed to their known values.
    The $H$ estimator for drifted fBM with known $\sigma$ instead uses a log-least-squares fit based on second-difference variance scaling.
    We report the mean absolute error~(MAE)
    \begin{align*}
        \frac{1}{N}\sum_{i=1}^{N}
        \left|
            \widehat{\theta}_q(\tilde{X}_{i})-\theta_{i,q}
        \right|.
    \end{align*}
    \item \textbf{Signature fidelity.}
    We measure whether the inverted path matches the reference path at evaluation depth $r$, using time-augmented signatures; we use $r=6$ for Gaussian-process diagnostics and $r=4$ for financial-data diagnostics.
    \begin{align*}
        \operatorname*{median}_{i}
        \frac{
            \|S_{\mathrm{TA}}^{\leq r}(\tilde{X}_i)
              - S_{\mathrm{TA}}^{\leq r}(X_i)\|_2^2
        }{
            \|S_{\mathrm{TA}}^{\leq r}(X_i)\|_2^2
        }
        \times 100\%.
    \end{align*}
\end{itemize}

\paragraph{Stylized-fact, autocorrelation, and cross-series correlation errors.}
These metrics evaluate whether generated financial time series preserve the stylized facts, autocorrelation, and cross-series correlation of the reference paths.
Let $\Delta X_i$ denote the daily log-return sequence of path window $X_i$, with returns computed separately for each coordinate in multivariate data.
We report median absolute errors between generated and empirical windows, together with 10th--90th percentile ranges.
\begin{itemize}
    \item \textbf{Realized volatility.}
    We compute realized volatility as $\|\Delta X_i\|_2$.
    \item \textbf{Skewness.}
    We compute the adjusted Fisher--Pearson skewness of $\Delta X_i$.
    \item \textbf{Kurtosis.}
    We compute the bias-corrected Fisher excess kurtosis of $\Delta X_i$.
    \item \textbf{Autocorrelation.}
    We compute the autocorrelation of $\Delta X_i$ and $|\Delta X_i|$ and average the absolute differences from the corresponding reference autocorrelations over lags $1$--$30$.
    \item \textbf{Cross-series correlation.}
    For multivariate paths, we compute pairwise correlations between coordinates of $\Delta X_i$ and between coordinates of $|\Delta X_i|$.
\end{itemize}

\paragraph{Qualitative diagnostics.}
These diagnostics visualize distributional behavior that is difficult to summarize with a single scalar metric.
\begin{itemize}
    \item \textbf{Path overlays.}
    We overlay inverted paths with the reference path to inspect sample diversity and path-level fidelity.
    \item \textbf{ACF plots.}
    For Gaussian processes, we plot the ACF of $\Delta X_i$ for log-GBM and drifted fBM, and the ACF of $X_i$ for OU.
    For financial data, we plot the ACF of $|\Delta X_i|$ to visualize volatility clustering.
    \item \textbf{QQ plots.}
    We compare empirical quantiles of generated and reference samples to assess marginal distributional fit.
\end{itemize}

\clearpage

\section{Additional Gaussian process results}
\label{app:synthetic}

\subsection{Quantitative results for univariate Gaussian processes}
\label{app:synthetic-quant}
\label{app:additional}

\paragraph{Theoretical-alignment diagnostics.}
\cref{tab:bayes_error_appendix} expands \cref{tab:bayes_error} by reporting full-sample empirical reconstruction-error estimates with normal-approximation 95\% bootstrap half-widths, together with the inter-sample spread ratio~(sample-to-sample to sample-to-reference MSE).
\cref{tab:ls-posterior-energy-gap} complements these reconstruction-error and spread diagnostics with a direct distributional comparison between generated samples and the numerical LS posterior using squared energy distance.

Overall, the LS empirical means stay close to the LS Oracle within small bootstrap uncertainty, while TA and TLL usually reduce path MSE under richer conditioning.
The 95\% bootstrap confidence intervals for the squared energy distance gaps contain zero in 8 of 9 process--depth settings, supporting agreement with the numerical LS posterior across most evaluated settings.
The spread ratios fall within $[0.96, 1.01]$ across all settings, indicating that sample-to-sample MSE is close to sample-to-reference MSE.

\begin{table}[h]
\centering
\captionsetup{skip=6pt}

\caption{\textbf{Posterior-sampling reconstruction errors and spread ratios.}
LS Oracle is analytic for log-GBM; for drifted fBM and OU, it reports the mean with 95\% confidence intervals based on the t-distribution over 10 numerical seeds.
LS, TA, and TLL report empirical reconstruction-error means with normal-approximation 95\% bootstrap confidence intervals.
Spread ratio is the ratio of sample-to-sample MSE to sample-to-reference MSE.}
\label{tab:bayes_error_appendix}

\renewcommand{\arraystretch}{1.1}
\setlength{\tabcolsep}{4pt}
\scalebox{0.95}{
\begin{tabular}{@{}c c c ccc c ccc@{}}
\toprule
& & \multicolumn{4}{c}{Posterior-sampling reconstruction error} & & \multicolumn{3}{c}{Spread ratio} \\
\cmidrule(lr){3-6}\cmidrule(lr){8-10}
Process & $r$ & LS Oracle & LS & TA & TLL & & LS & TA & TLL \\ \midrule

\multirow{6}{*}{log-GBM}
& 1 & 1.364 & $1.317{\pm}0.040$ & $1.317{\pm}0.040$ & $1.317{\pm}0.040$ & & 0.98 & 0.98 & 0.98 \\
& 2 & 0.545 & $0.547{\pm}0.014$ & $0.547{\pm}0.014$ & $0.553{\pm}0.016$ & & 0.98 & 0.98 & 0.99 \\
& 3 & 0.351 & $0.348{\pm}0.007$ & $0.321{\pm}0.010$ & $0.318{\pm}0.011$ & & 1.00 & 1.01 & 1.01 \\
& 4 & 0.260 & $0.261{\pm}0.005$ & $0.197{\pm}0.006$ & $0.197{\pm}0.006$ & & 0.99 & 1.00 & 1.01 \\
& 5 & 0.207 & $0.208{\pm}0.004$ & $0.150{\pm}0.004$ & $0.155{\pm}0.004$ & & 0.99 & 1.01 & 1.01 \\
& 6 & 0.172 & $0.173{\pm}0.003$ & $0.135{\pm}0.003$ & $0.144{\pm}0.004$ & & 1.01 & 1.00 & 1.01 \\ \midrule

\multirow{6}{*}{fBM}
& 1 & $1.486{\pm}0.001$ & $1.488{\pm}0.050$ & $1.488{\pm}0.050$ & $1.488{\pm}0.050$ & & 0.99 & 0.99 & 0.99 \\
& 2 & $0.656{\pm}0.002$ & $0.664{\pm}0.022$ & $0.664{\pm}0.022$ & $0.674{\pm}0.031$ & & 0.96 & 0.96 & 0.99 \\
& 3 & $0.450{\pm}0.002$ & $0.456{\pm}0.017$ & $0.454{\pm}0.025$ & $0.442{\pm}0.026$ & & 0.96 & 0.98 & 0.99 \\
& 4 & $0.352{\pm}0.001$ & $0.365{\pm}0.015$ & $0.315{\pm}0.019$ & $0.314{\pm}0.020$ & & 0.98 & 0.98 & 0.99 \\
& 5 & $0.291{\pm}0.002$ & $0.307{\pm}0.013$ & $0.252{\pm}0.016$ & $0.255{\pm}0.016$ & & 0.98 & 1.00 & 1.00 \\
& 6 & $0.253{\pm}0.002$ & $0.265{\pm}0.012$ & $0.225{\pm}0.014$ & $0.248{\pm}0.016$ & & 0.97 & 1.00 & 1.00 \\ \midrule

\multirow{6}{*}{OU}
& 1 & $1.163{\pm}0.006$ & $1.143{\pm}0.036$ & $1.143{\pm}0.036$ & $1.143{\pm}0.036$ & & 0.98 & 0.98 & 0.98 \\
& 2 & $0.502{\pm}0.001$ & $0.506{\pm}0.012$ & $0.506{\pm}0.012$ & $0.509{\pm}0.014$ & & 0.99 & 0.99 & 0.99 \\
& 3 & $0.334{\pm}0.001$ & $0.333{\pm}0.007$ & $0.305{\pm}0.011$ & $0.301{\pm}0.011$ & & 1.00 & 1.00 & 1.00 \\
& 4 & $0.252{\pm}0.001$ & $0.254{\pm}0.005$ & $0.194{\pm}0.006$ & $0.197{\pm}0.006$ & & 0.99 & 0.99 & 1.01 \\
& 5 & $0.202{\pm}0.001$ & $0.204{\pm}0.004$ & $0.150{\pm}0.004$ & $0.155{\pm}0.004$ & & 0.99 & 1.00 & 1.01 \\
& 6 & $0.169{\pm}0.000$ & $0.171{\pm}0.003$ & $0.135{\pm}0.003$ & $0.145{\pm}0.004$ & & 1.01 & 0.99 & 1.00 \\

\bottomrule
\end{tabular}
}
\end{table}

\begin{table}[h]
\centering
\captionsetup{skip=6pt}
\caption{\textbf{Squared energy distance gaps under LS conditioning.}
We compare the squared energy distance between samples from our model and the numerical LS posterior against that between two independent sample sets from the same posterior.
Entries report 95\% bootstrap confidence intervals for the first distance minus the second, using 1,000 inputs and 30 samples per distribution.}
\label{tab:ls-posterior-energy-gap}
\renewcommand{\arraystretch}{1.1}
\setlength{\tabcolsep}{6pt}
\begin{tabular}{@{}cccc@{}}
\toprule
$r$ & log-GBM & fBM & OU \\
\midrule
2 & $[-9.2\mathrm{e}{-4}, 1.2\mathrm{e}{-3}]$ & $[-1.1\mathrm{e}{-3}, 9.4\mathrm{e}{-4}]$ & $[-6.6\mathrm{e}{-4}, 1.4\mathrm{e}{-3}]$ \\
4 & $[-4.9\mathrm{e}{-4}, 5.8\mathrm{e}{-4}]$ & $[-2.0\mathrm{e}{-4}, 8.8\mathrm{e}{-4}]$ & $[-5.0\mathrm{e}{-4}, 5.7\mathrm{e}{-4}]$ \\
6 & $[-4.1\mathrm{e}{-4}, 3.4\mathrm{e}{-4}]$ & $[2.1\mathrm{e}{-4}, 1.0\mathrm{e}{-3}]$ & $[-2.9\mathrm{e}{-4}, 4.7\mathrm{e}{-4}]$ \\
\bottomrule
\end{tabular}
\end{table}

\paragraph{Conditioning consistency.}
\cref{tab:conditioning_consistency} shows that conditioning errors are generally small, although they tend to increase with depth under TLL conditioning.

\begin{table}[h]
\centering
\captionsetup{skip=6pt}

\caption{\textbf{Conditioning consistency.}
Entries report the median relative MSE between the input conditioning statistic and the statistic recomputed from each generated path, expressed as percentages.}
\label{tab:conditioning_consistency}

\renewcommand{\arraystretch}{1.1}
\setlength{\tabcolsep}{4pt}
\scalebox{0.85}{
\begin{tabular}{c ccc ccc ccc}
\toprule
& \multicolumn{3}{c}{log-GBM} & \multicolumn{3}{c}{fBM} & \multicolumn{3}{c}{OU} \\
\cmidrule(lr){2-4}\cmidrule(lr){5-7}\cmidrule(lr){8-10}
$r$ & LS & TA & TLL & LS & TA & TLL & LS & TA & TLL \\ \midrule

1 & 0.03\% & 0.03\% & 0.03\% & 0.04\% & 0.04\% & 0.04\% & 0.01\% & 0.01\% & 0.01\% \\
2 & 0.05\% & 0.05\% & 0.04\% & 0.06\% & 0.06\% & 0.05\% & 0.01\% & 0.01\% & 0.03\% \\
3 & 0.04\% & 0.04\% & 0.24\% & 0.05\% & 0.05\% & 0.42\% & 0.01\% & 0.02\% & 0.12\% \\
4 & 0.04\% & 0.13\% & 0.44\% & 0.07\% & 0.16\% & 0.72\% & 0.01\% & 0.09\% & 0.41\% \\
5 & 0.04\% & 0.33\% & 1.75\% & 0.05\% & 0.34\% & 3.12\% & 0.01\% & 0.21\% & 1.40\% \\
6 & 0.04\% & 0.72\% & 2.76\% & 0.05\% & 0.77\% & 4.85\% & 0.01\% & 0.32\% & 2.28\% \\

\bottomrule
\end{tabular}
}
\end{table}

\paragraph{Signature fidelity.}
\cref{tab:signature_fidelity_synthetic} shows that the probabilistic methods achieve low signature errors, indicating that their generated samples remain close to the conditioning path in signature space.

\begin{table}[h]
\centering
\captionsetup{skip=6pt}

\caption{\textbf{Signature fidelity on Gaussian processes.}
Entries report relative MSE of truncated time-augmented signatures up to depth $6$, expressed as percentages.}
\label{tab:signature_fidelity_synthetic}

\renewcommand{\arraystretch}{1.1}
\setlength{\tabcolsep}{4pt}
\scalebox{0.85}{
\begin{tabular}{l ccc}
\toprule
Method & log-GBM & fBM & OU \\ \midrule
\textbf{Deterministic} \\
\Fourier   & 60.08\% & 57.85\% & 82.65\% \\
\Legendre  & 5.42\% & 6.27\% & 3.65\% \\
\Insertion & 0.99\% & 1.01\% & 1.13\% \\
\RegTAd{6}    & 0.11\% & 0.15\% & 0.05\% \\
\RegTLLd{6}   & 0.18\% & 0.25\% & 0.08\% \\
\cmidrule(lr){1-4}
\textbf{Probabilistic} \\
\OursTAd{6}        & 0.03\% & 0.03\% & 0.01\% \\
\OursTLLd{6}       & 0.05\% & 0.06\% & 0.02\% \\
\bottomrule
\end{tabular}
}
\end{table}

The nonzero signature errors of deterministic baselines in \cref{tab:signature_fidelity_synthetic} can be understood from their reconstruction mechanisms.
\begin{itemize}
    \item \textbf{Fourier and Legendre} reconstruct finite-dimensional basis projections.
    For example, Fourier guarantees $R_F^N\left(S_F^{\leq N+2}(X)\right)=P_F^N(X)$, where $R_F^N$ is the inversion map and $P_F^N$ is the Fourier projection.
    Since $P_F^N(X)\neq X$ in general, their truncated signatures need not match.
    Empirically, this is consistent with the visibly different endpoint increment $X_T-X_0$ of the Fourier reconstruction in \cref{fig:teaser-overlay}.
    \item \textbf{Insertion} estimates local slopes using only the two highest signature levels without explicitly enforcing agreement at the lower levels.
\end{itemize}

\clearpage
\subsection{Qualitative results across univariate Gaussian processes}
\label{app:synthetic-main-style-qual}

This subsection reports path overlays, ACF plots, and QQ plots for the two most extreme settings of each Gaussian process under TLLd6 and TAd6 conditioning.
Probabilistic inversions better preserve the characteristic roughness and fluctuation scales of the reference paths, whereas deterministic reconstructions often appear oversmoothed.
The ACF and QQ plots generally support this pattern, with probabilistic inversions more closely matching the reference autocorrelation and marginal distributions.
The periodic spikes in the regression baselines' ACF plots occur at multiples of the DiT patch size~(8) and reflect patch-boundary discontinuities in the reconstructed paths.

\begin{figure}[h]
\centering
\captionsetup{skip=6pt}
\setlength{\tabcolsep}{0pt}
\renewcommand{\arraystretch}{0.8}
\scalebox{1}{
\begin{tabular}{ccc}
    \multicolumn{3}{c}{\includegraphics[height=0.015223\linewidth,trim=0 0 317.28bp 0,clip]{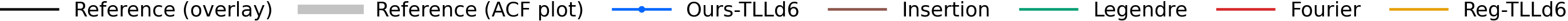}}\\
    \multicolumn{3}{c}{\includegraphics[height=0.015223\linewidth,trim=341.52bp 0 0 0,clip]{assets/results/figures/assets/appendix/gbm/gbm_mu2_sig1.6/TLL3d6_legend.png}}\\
    \includegraphics[width=0.48\linewidth]{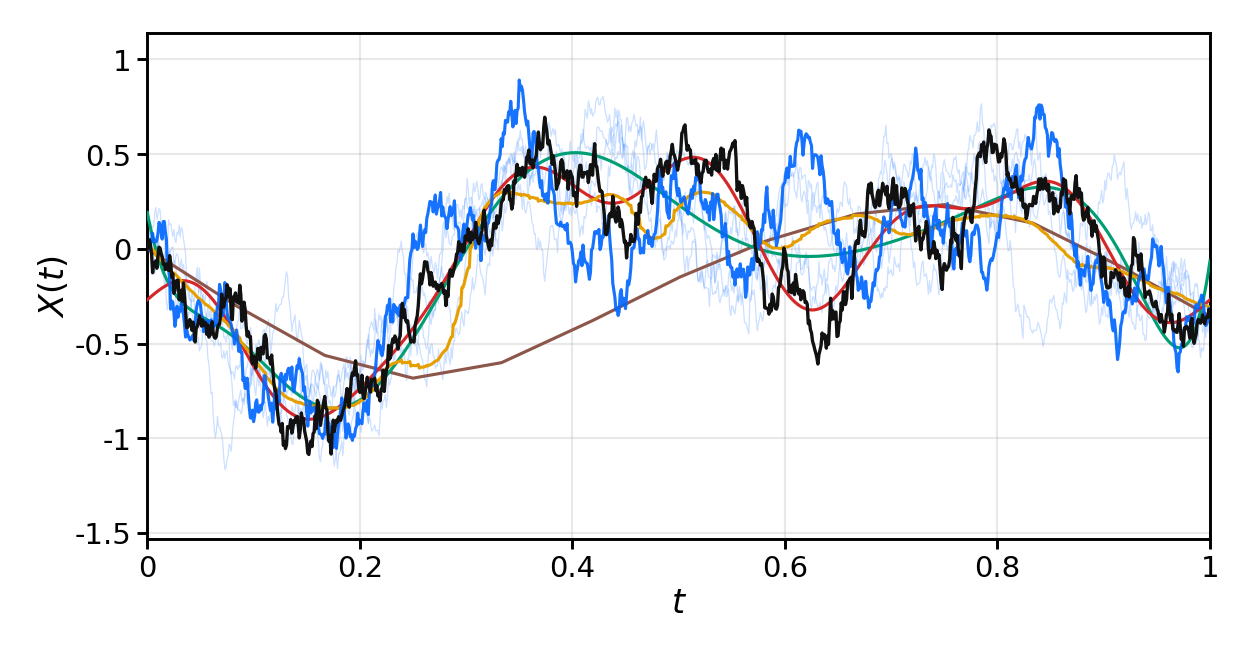}
    &
    \includegraphics[width=0.25\linewidth]{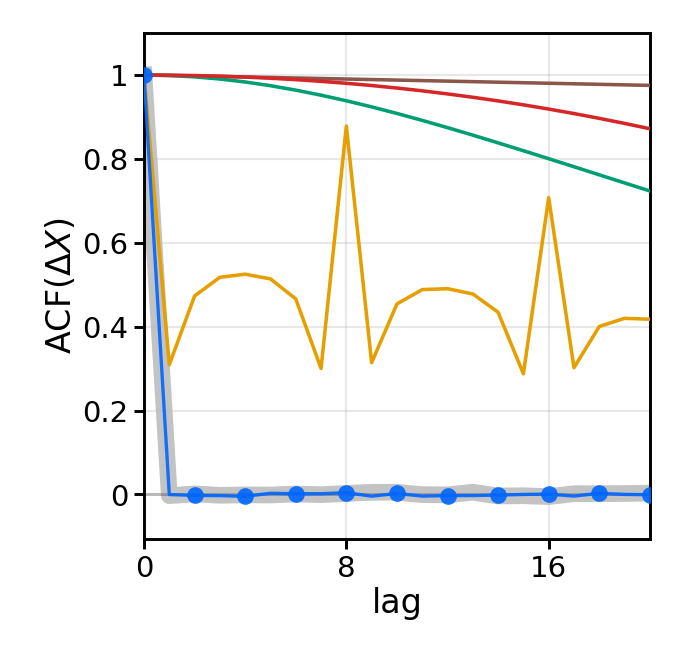}
    &
    \includegraphics[width=0.25\linewidth]{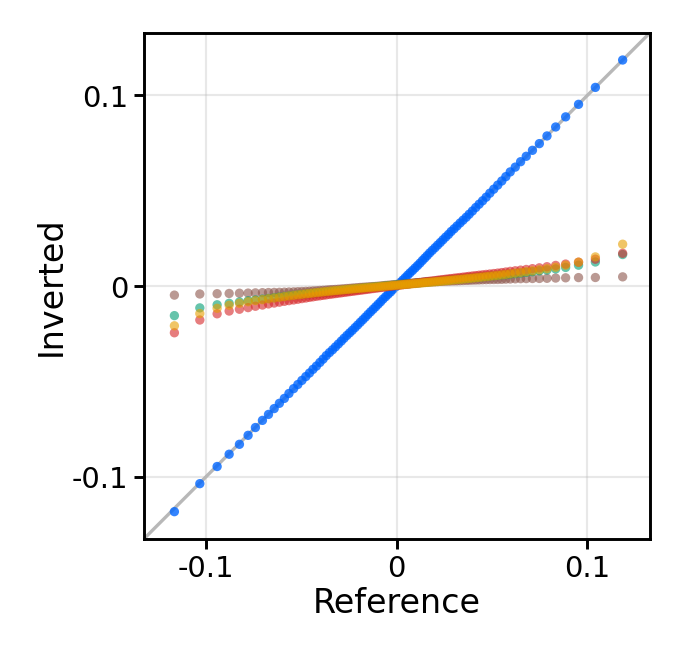} \\[-2pt]
    \multicolumn{3}{c}{(a) log-GBM~($\sigma=1.6,\mu=2.0$)} \\
    \includegraphics[width=0.48\linewidth]{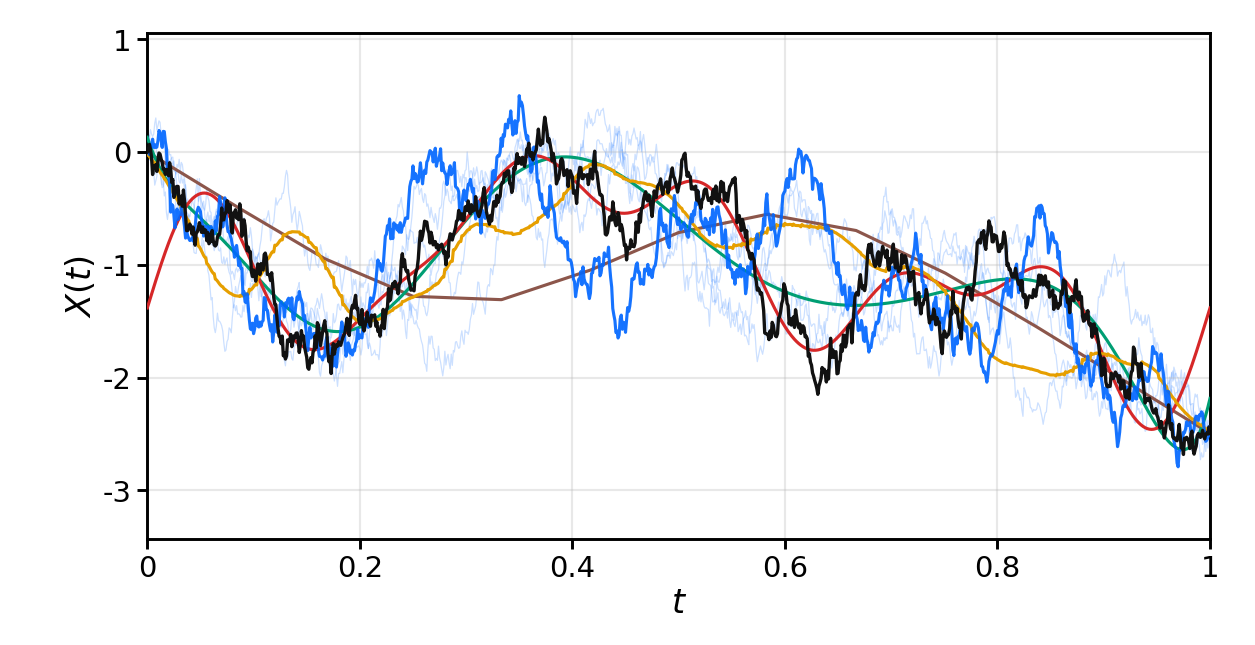}
    &
    \includegraphics[width=0.25\linewidth]{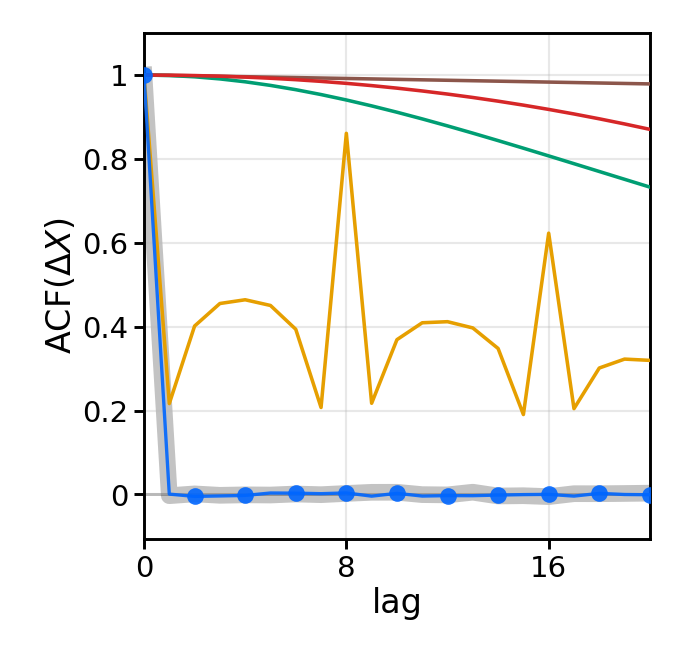}
    &
    \includegraphics[width=0.25\linewidth]{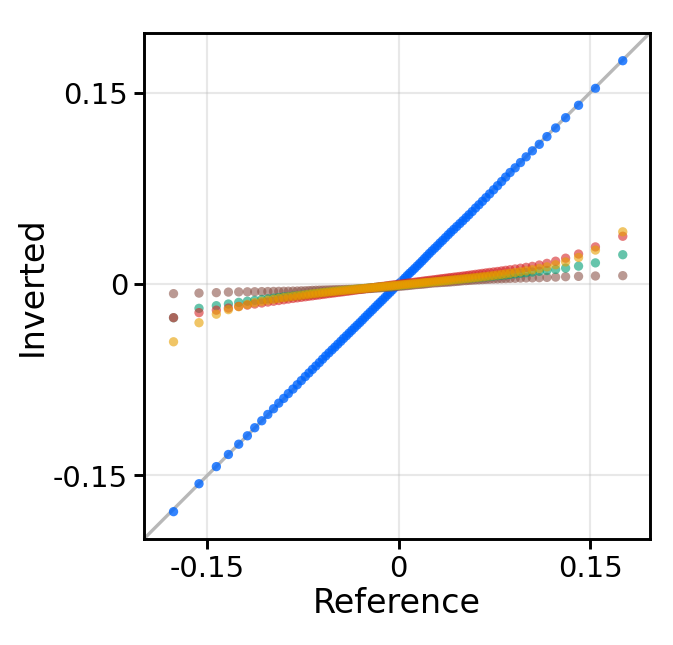} \\[-2pt]
    \multicolumn{3}{c}{(b) log-GBM~($\sigma=2.4,\mu=2.0$)}
\end{tabular}
}
\caption{\textbf{Path overlays, ACF plots, and QQ plots for log-GBM with TLL6 conditioning.}
ACFs are computed from increments $\Delta X$.}
\label{fig:app-main-style-gbm-tll}
\vspace{-4mm}
\end{figure}

\begin{figure}[h]
\centering
\captionsetup{skip=6pt}
\setlength{\tabcolsep}{0pt}
\renewcommand{\arraystretch}{0.8}
\scalebox{1}{
\begin{tabular}{ccc}
    \multicolumn{3}{c}{\includegraphics[height=0.015223\linewidth,trim=0 0 312.72bp 0,clip]{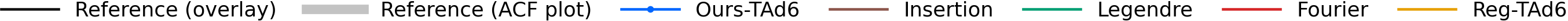}}\\
    \multicolumn{3}{c}{\includegraphics[height=0.015223\linewidth,trim=337.20bp 0 0 0,clip]{assets/results/figures/assets/appendix/gbm/gbm_mu2_sig1.6/TAd6_legend.png}}\\
    \includegraphics[width=0.48\linewidth]{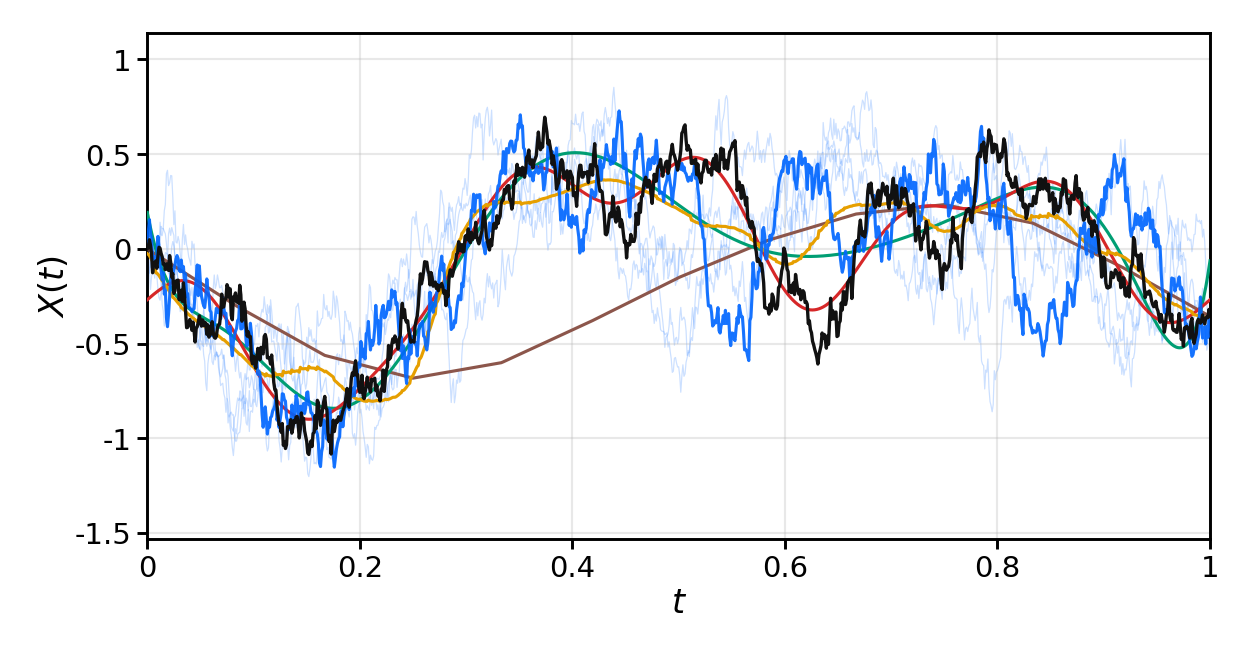}
    &
    \includegraphics[width=0.25\linewidth]{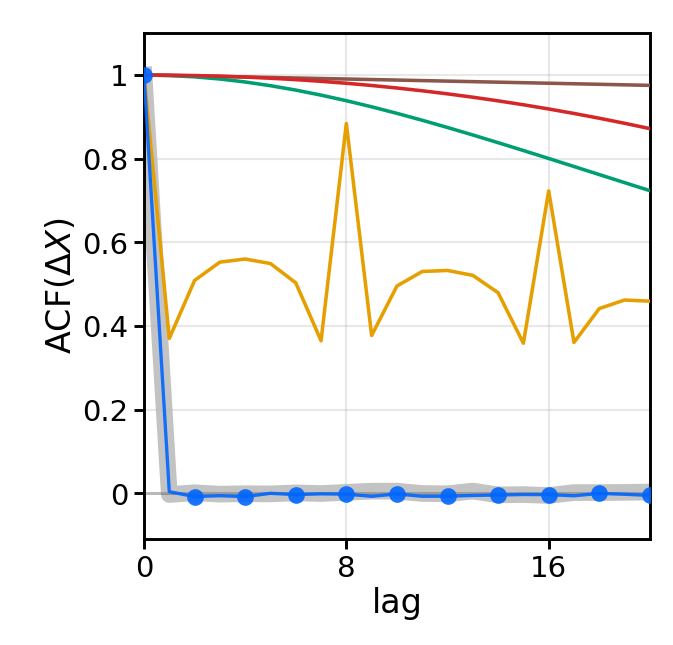}
    &
    \includegraphics[width=0.25\linewidth]{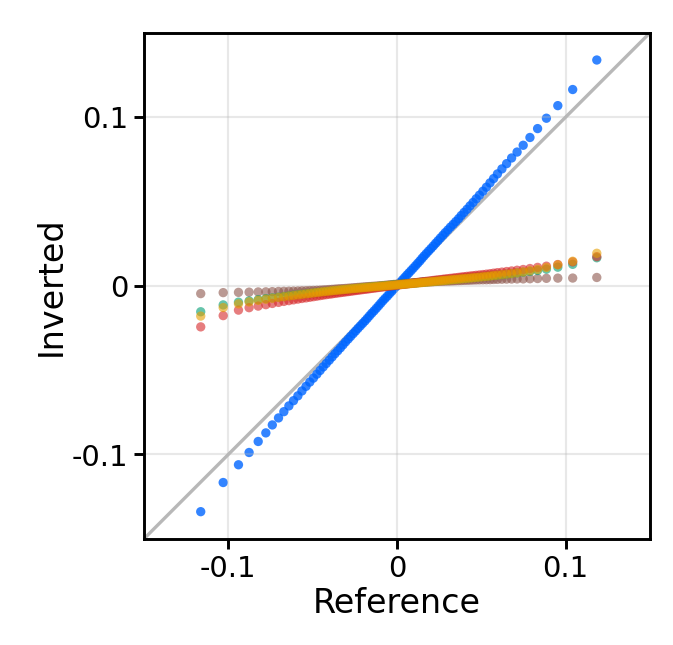} \\[-2pt]
    \multicolumn{3}{c}{(a) log-GBM~($\sigma=1.6,\mu=2.0$)} \\
    \includegraphics[width=0.48\linewidth]{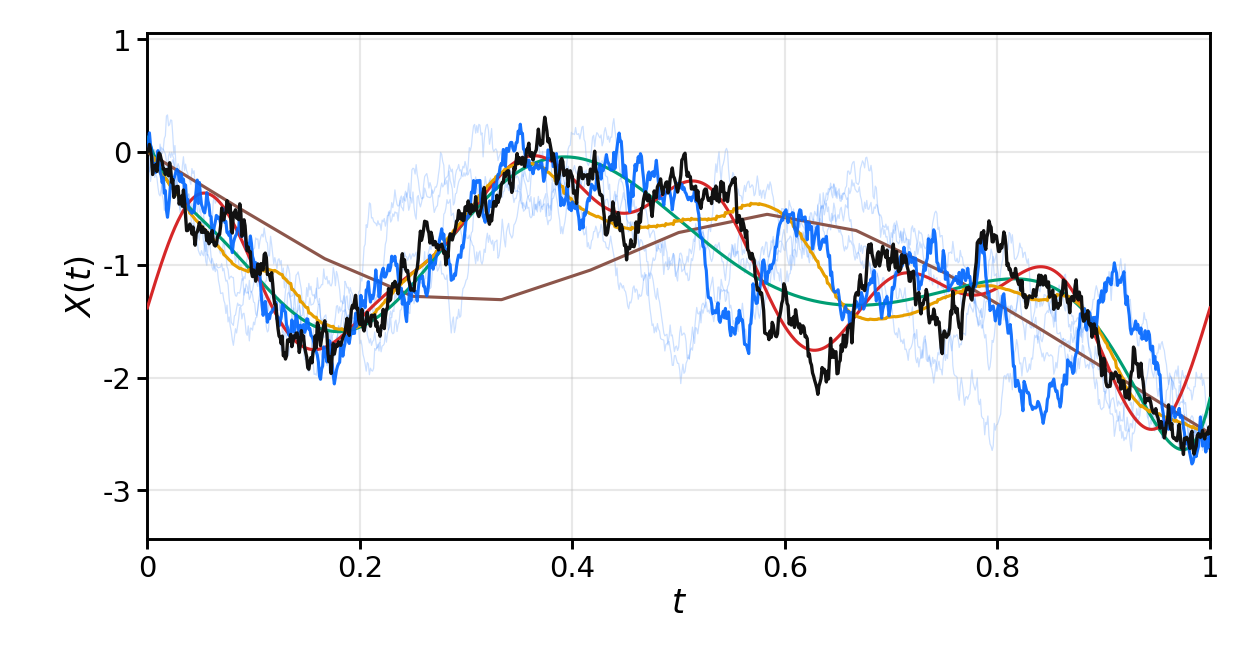}
    &
    \includegraphics[width=0.25\linewidth]{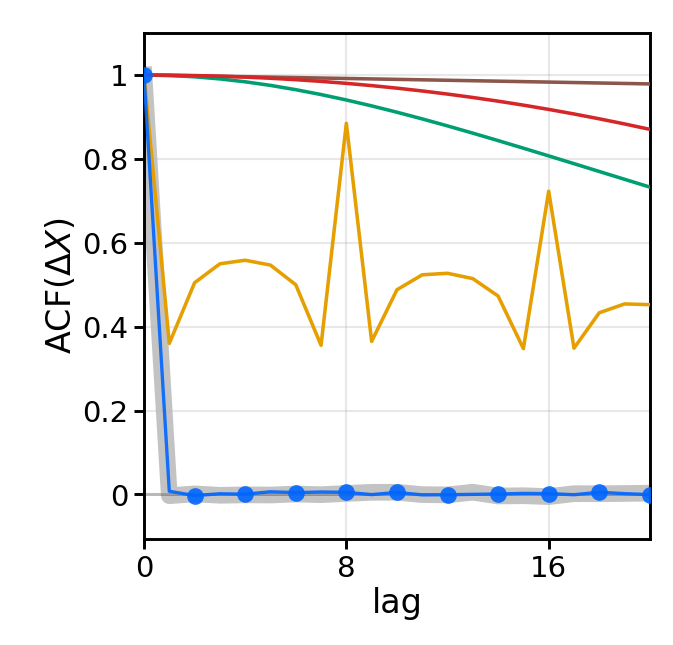}
    &
    \includegraphics[width=0.25\linewidth]{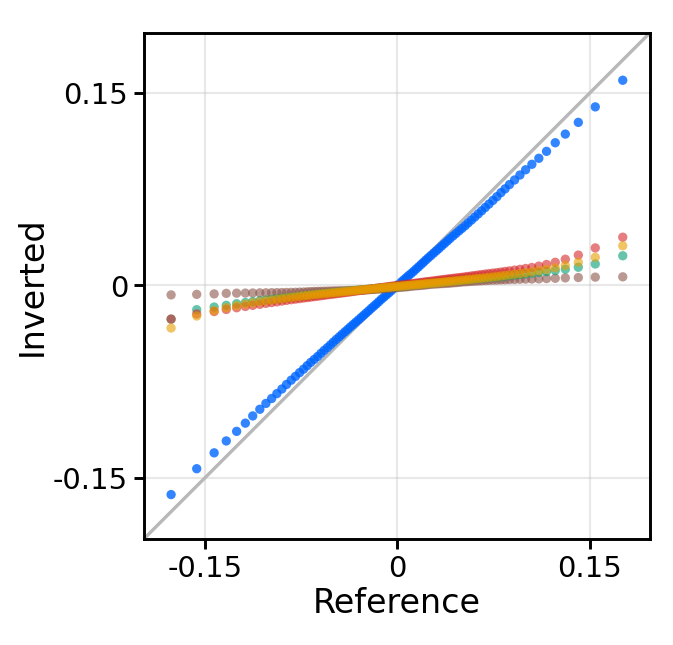} \\[-2pt]
    \multicolumn{3}{c}{(b) log-GBM~($\sigma=2.4,\mu=2.0$)}
\end{tabular}
}
\caption{\textbf{Path overlays, ACF plots, and QQ plots for log-GBM with TA6 conditioning.}
ACFs are computed from increments $\Delta X$.}
\label{fig:app-main-style-gbm-ta}
\vspace{-4mm}
\end{figure}

\clearpage
\begin{figure}[h]
\centering
\captionsetup{skip=6pt}
\setlength{\tabcolsep}{0pt}
\renewcommand{\arraystretch}{0.8}
\scalebox{1}{
\begin{tabular}{ccc}
    \multicolumn{3}{c}{\includegraphics[height=0.015223\linewidth,trim=0 0 317.28bp 0,clip]{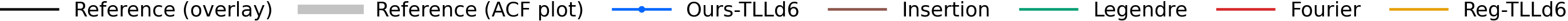}}\\
    \multicolumn{3}{c}{\includegraphics[height=0.015223\linewidth,trim=341.52bp 0 0 0,clip]{assets/results/figures/assets/appendix/fbm/fbm_mu2_sig2_H0.3/TLL3d6_legend.png}}\\
    \includegraphics[width=0.48\linewidth]{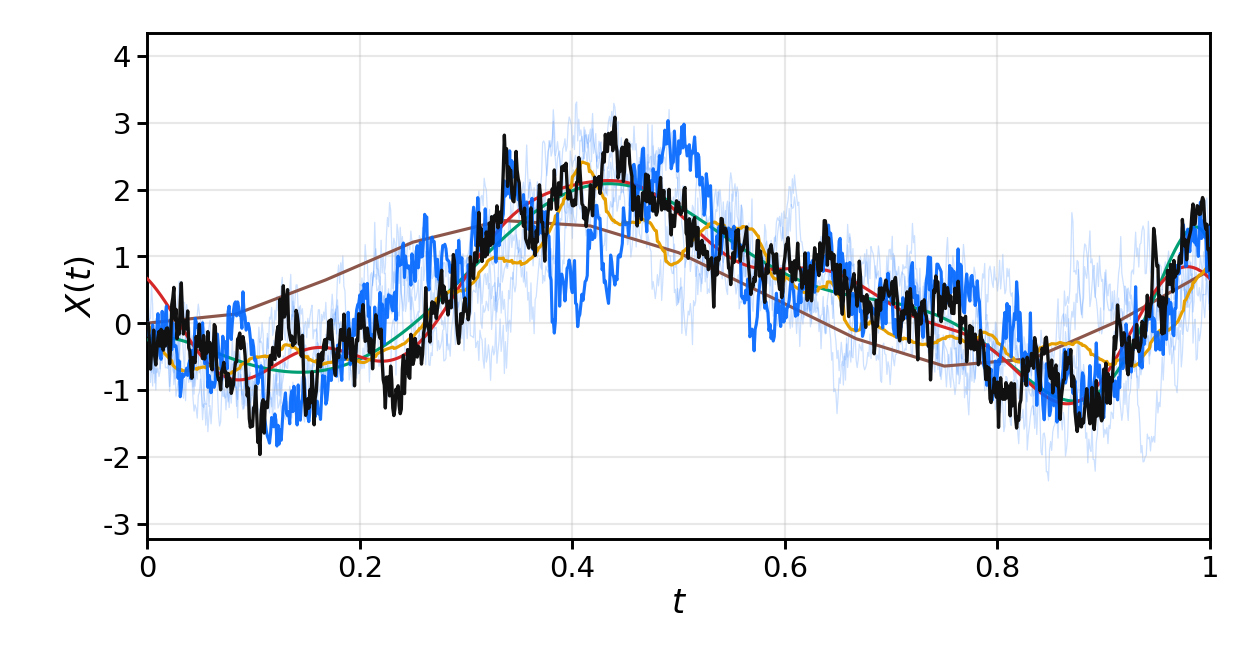}
    &
    \includegraphics[width=0.25\linewidth]{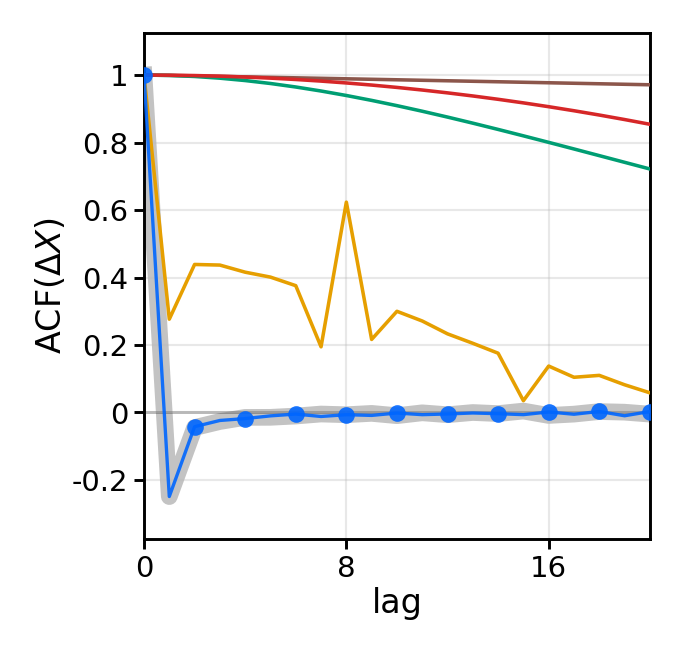}
    &
    \includegraphics[width=0.25\linewidth]{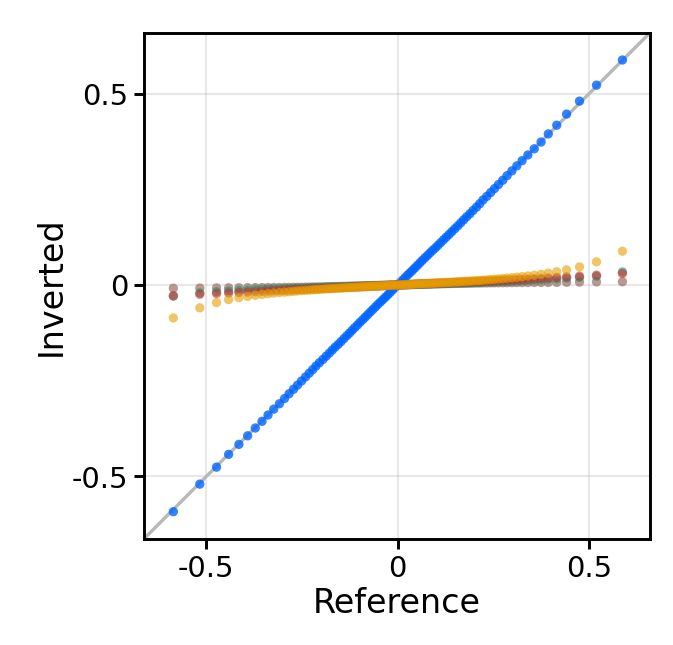} \\[-2pt]
    \multicolumn{3}{c}{(a) fBM~($H=0.3,\mu=2.0,\sigma=2.0$)} \\
    \includegraphics[width=0.48\linewidth]{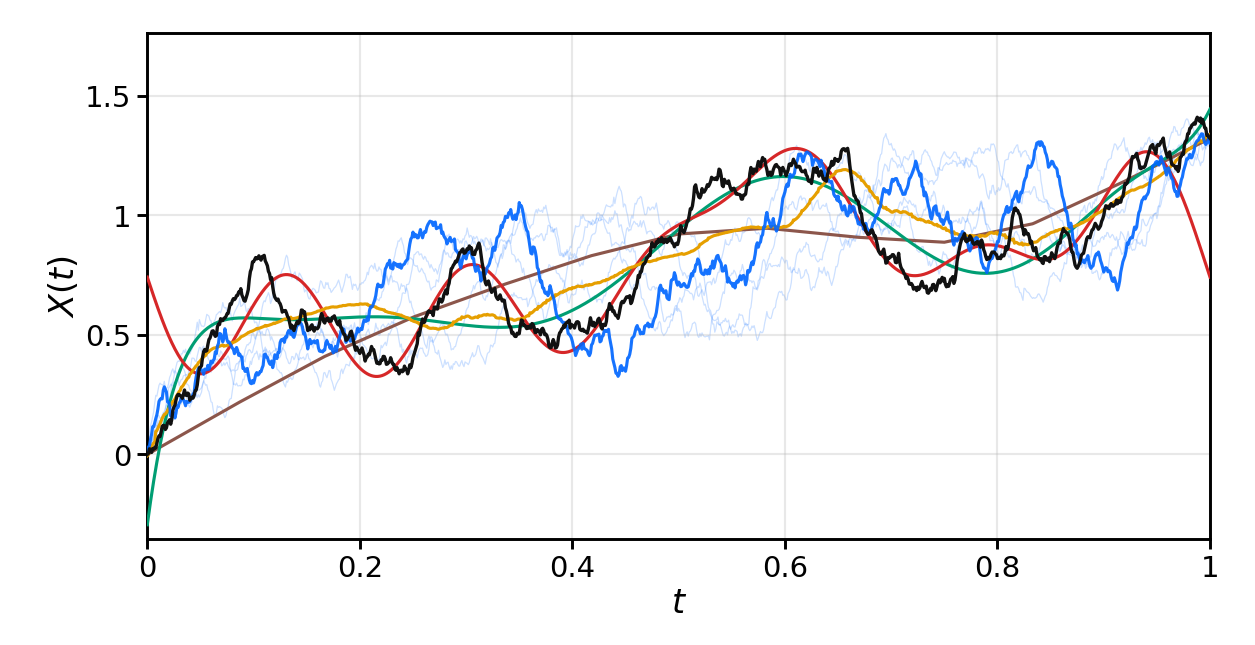}
    &
    \includegraphics[width=0.25\linewidth]{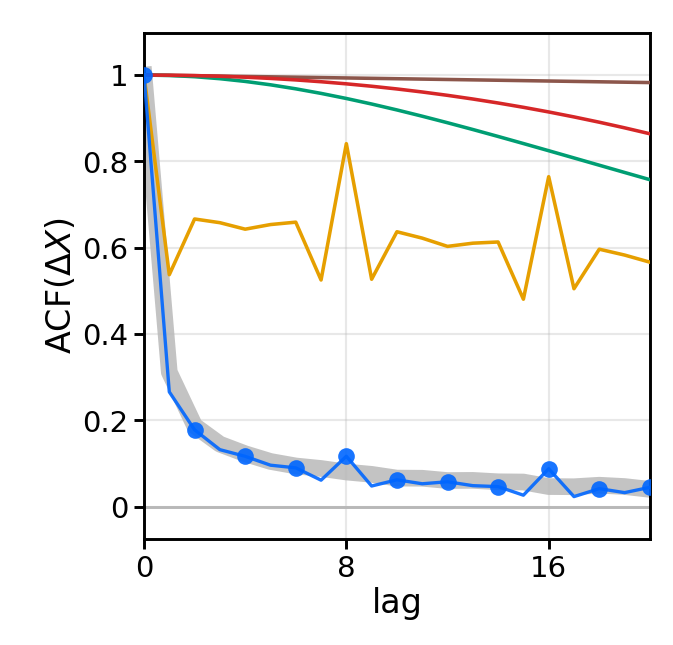}
    &
    \includegraphics[width=0.25\linewidth]{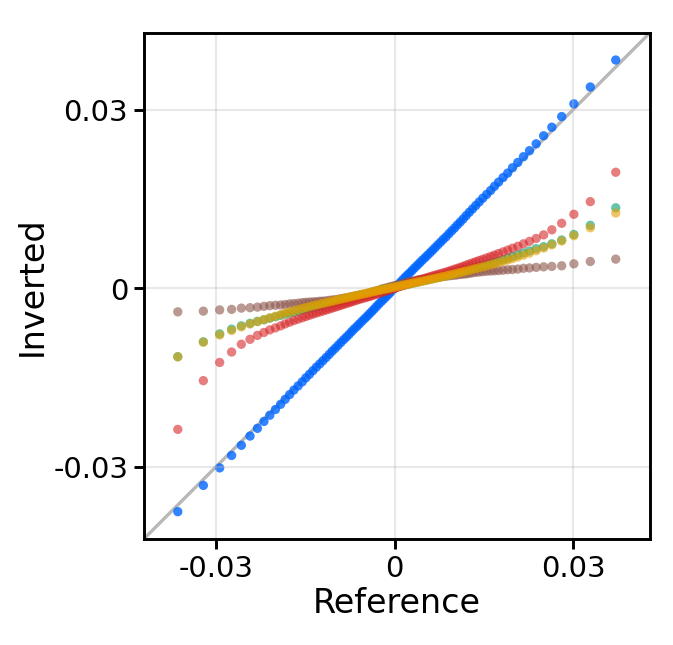} \\[-2pt]
    \multicolumn{3}{c}{(b) fBM~($H=0.7,\mu=2.0,\sigma=2.0$)}
\end{tabular}
}
\caption{\textbf{Path overlays, ACF plots, and QQ plots for drifted fBM with TLL6 conditioning.}
ACFs are computed from increments $\Delta X$.}
\label{fig:app-main-style-fbm-tll}
\vspace{-4mm}
\end{figure}

\begin{figure}[h]
\centering
\captionsetup{skip=6pt}
\setlength{\tabcolsep}{0pt}
\renewcommand{\arraystretch}{0.8}
\scalebox{1}{
\begin{tabular}{ccc}
    \multicolumn{3}{c}{\includegraphics[height=0.015223\linewidth,trim=0 0 312.72bp 0,clip]{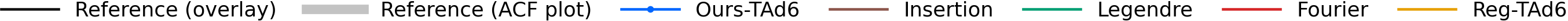}}\\
    \multicolumn{3}{c}{\includegraphics[height=0.015223\linewidth,trim=337.20bp 0 0 0,clip]{assets/results/figures/assets/appendix/fbm/fbm_mu2_sig2_H0.3/TAd6_legend.png}}\\
    \includegraphics[width=0.48\linewidth]{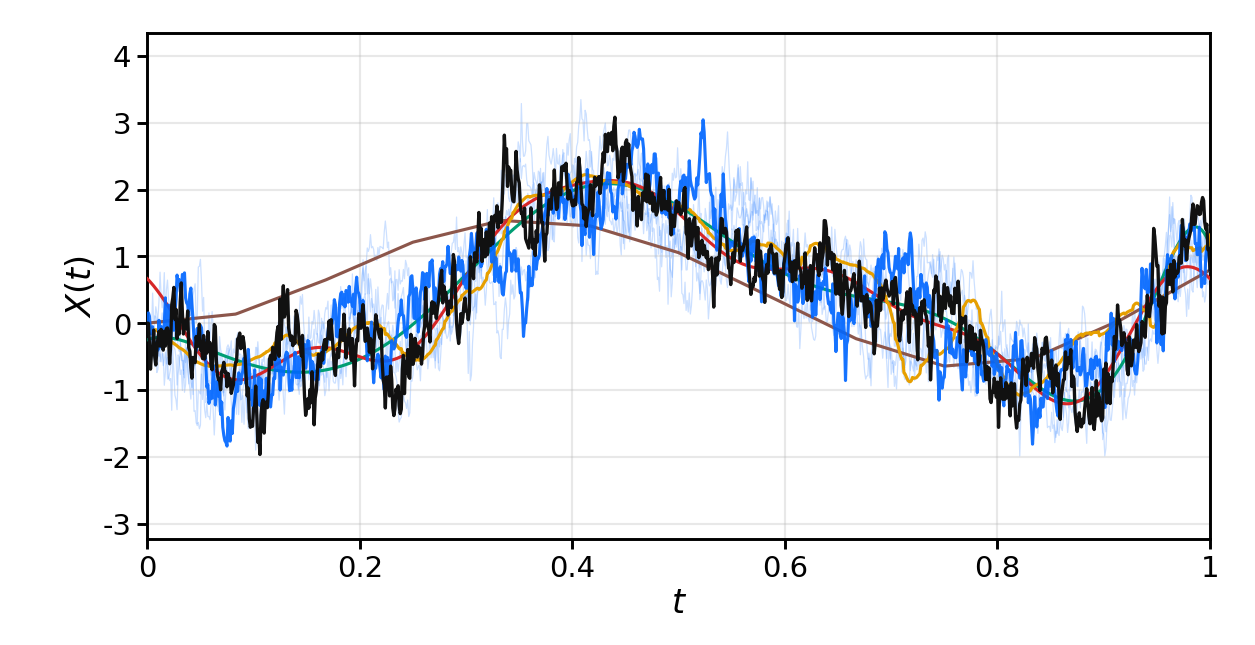}
    &
    \includegraphics[width=0.25\linewidth]{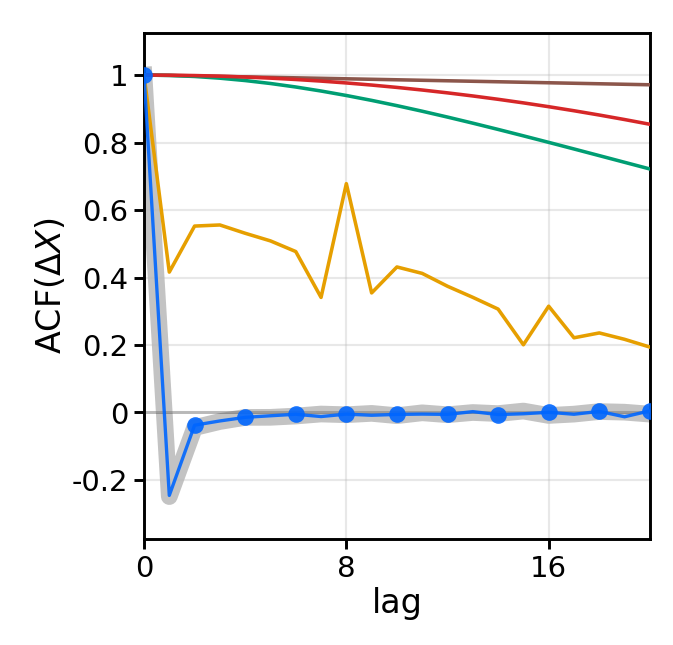}
    &
    \includegraphics[width=0.25\linewidth]{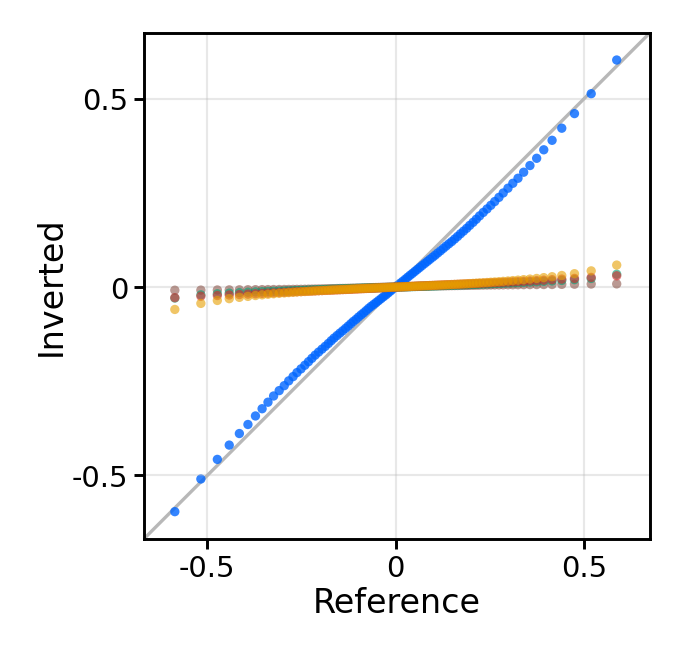} \\[-2pt]
    \multicolumn{3}{c}{(a) fBM~($H=0.3,\mu=2.0,\sigma=2.0$)} \\
    \includegraphics[width=0.48\linewidth]{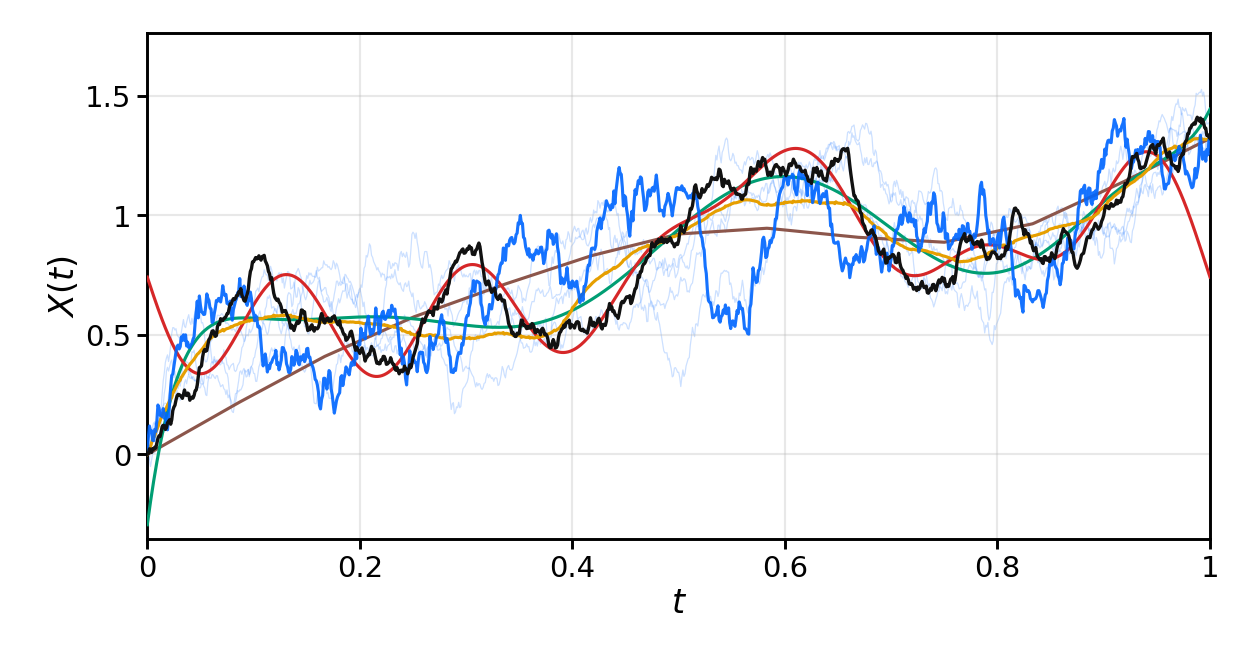}
    &
    \includegraphics[width=0.25\linewidth]{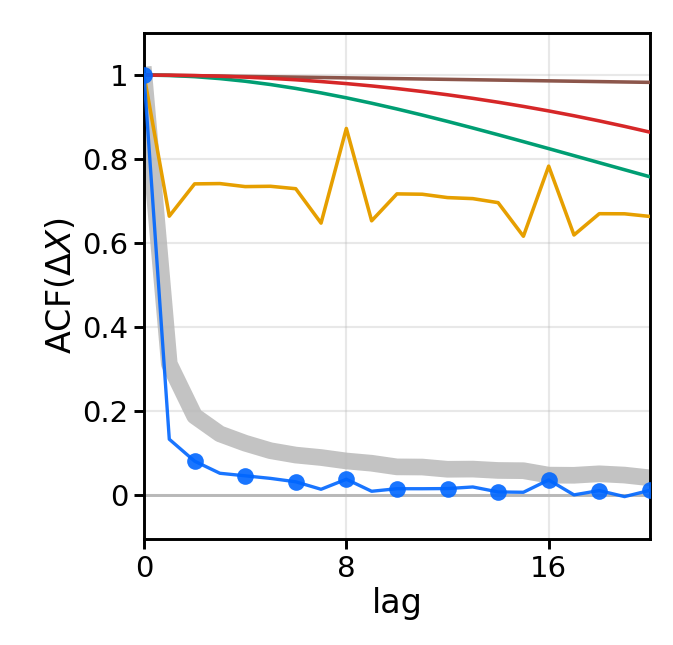}
    &
    \includegraphics[width=0.25\linewidth]{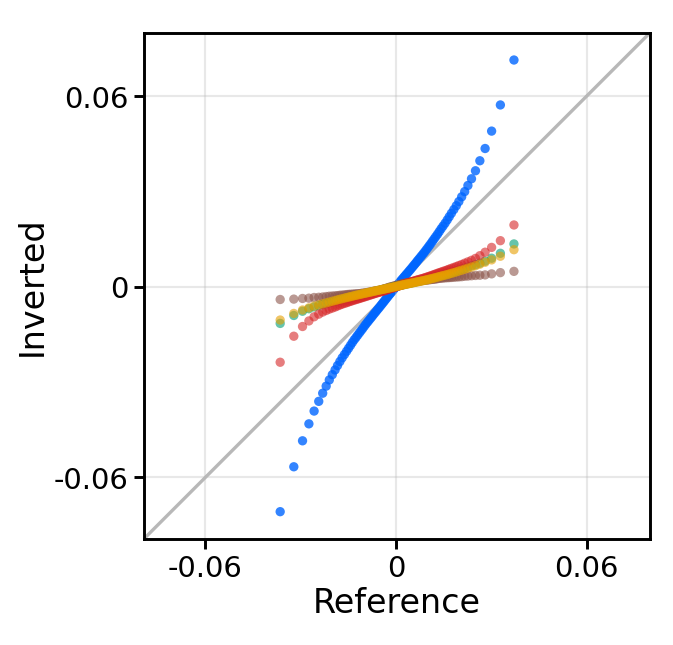} \\[-2pt]
    \multicolumn{3}{c}{(b) fBM~($H=0.7,\mu=2.0,\sigma=2.0$)}
\end{tabular}
}
\caption{\textbf{Path overlays, ACF plots, and QQ plots for drifted fBM with TA6 conditioning.}
ACFs are computed from increments $\Delta X$.}
\label{fig:app-main-style-fbm-ta}
\vspace{-4mm}
\end{figure}

\clearpage
\begin{figure}[h]
\centering
\captionsetup{skip=6pt}
\setlength{\tabcolsep}{0pt}
\renewcommand{\arraystretch}{0.8}
\scalebox{1}{
\begin{tabular}{ccc}
    \multicolumn{3}{c}{\includegraphics[height=0.015223\linewidth,trim=0 0 317.28bp 0,clip]{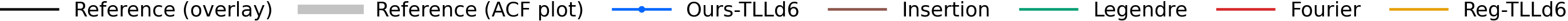}}\\
    \multicolumn{3}{c}{\includegraphics[height=0.015223\linewidth,trim=341.52bp 0 0 0,clip]{assets/results/figures/assets/appendix/ou/ou_theta0.5_mu3_sig2/TLL3d6_legend.png}}\\
    \includegraphics[width=0.48\linewidth]{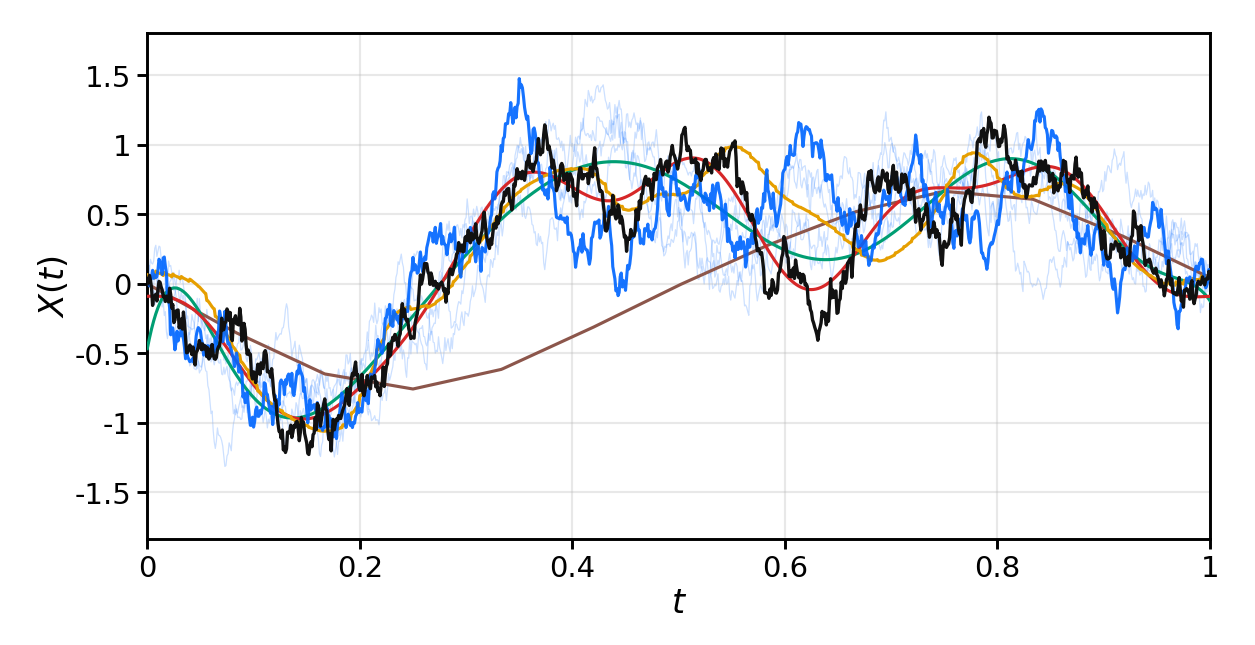}
    &
    \includegraphics[width=0.25\linewidth]{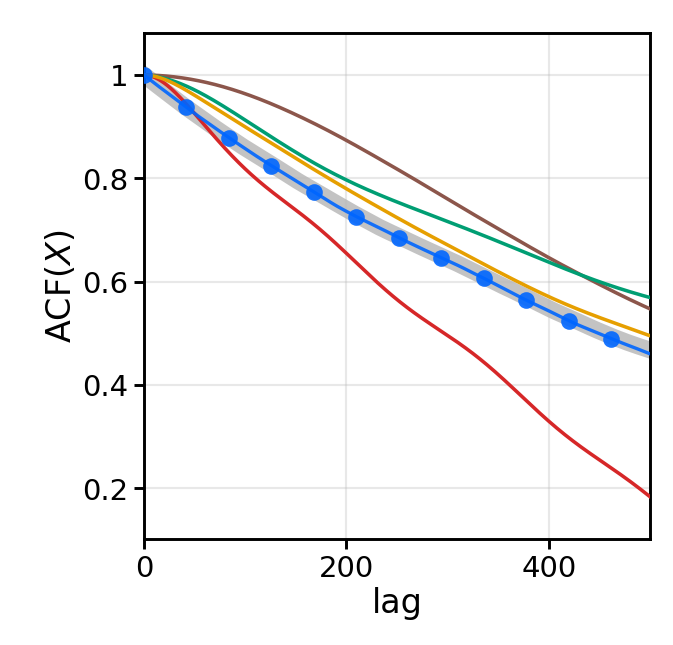}
    &
    \includegraphics[width=0.25\linewidth]{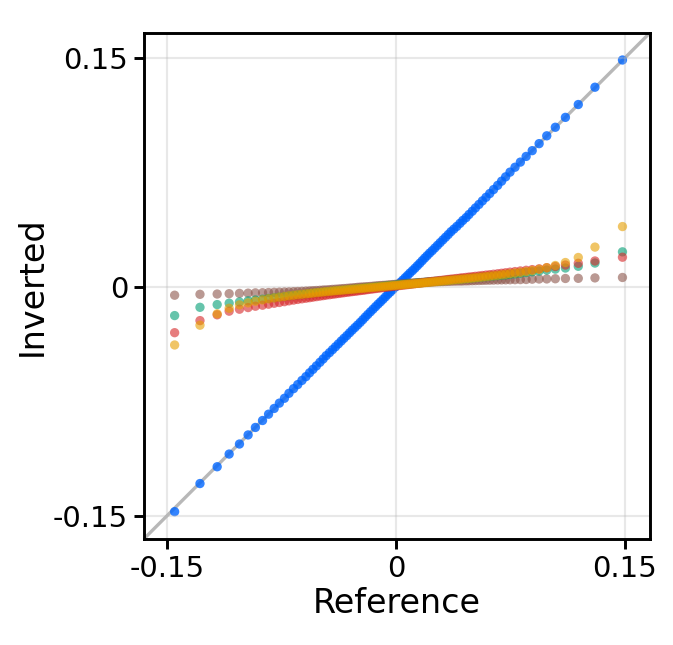} \\[-2pt]
    \multicolumn{3}{c}{(a) OU~($\kappa=0.5,\mu=3.0,\sigma=2.0$)} \\
    \includegraphics[width=0.48\linewidth]{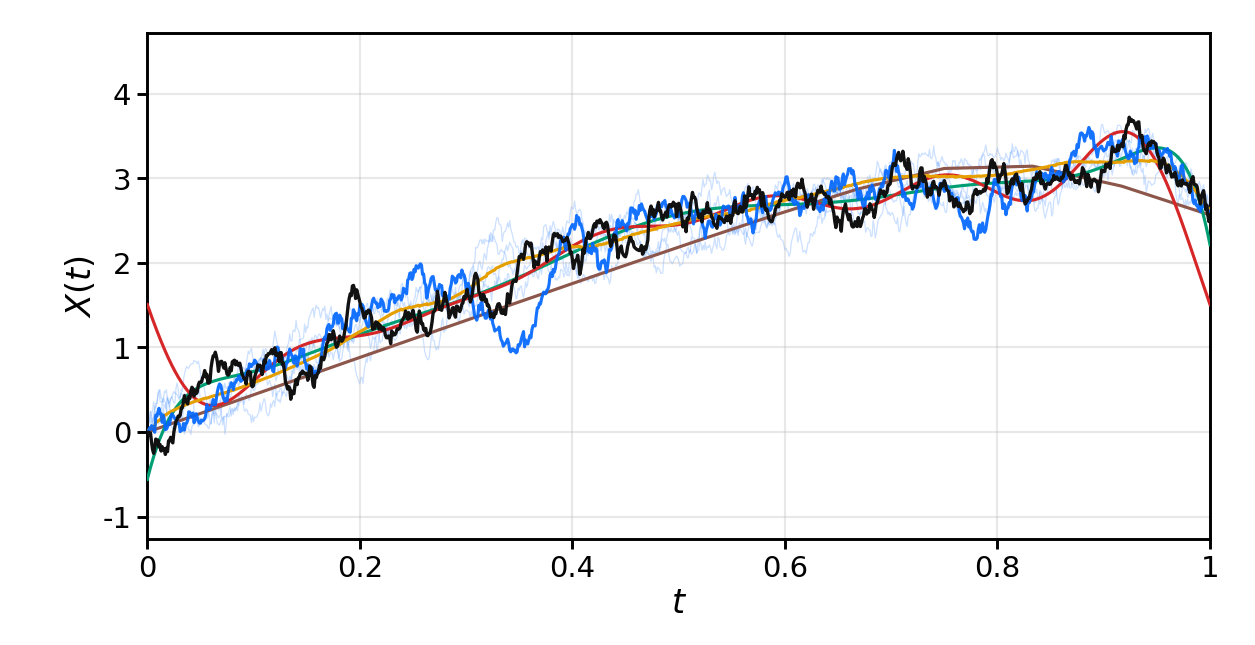}
    &
    \includegraphics[width=0.25\linewidth]{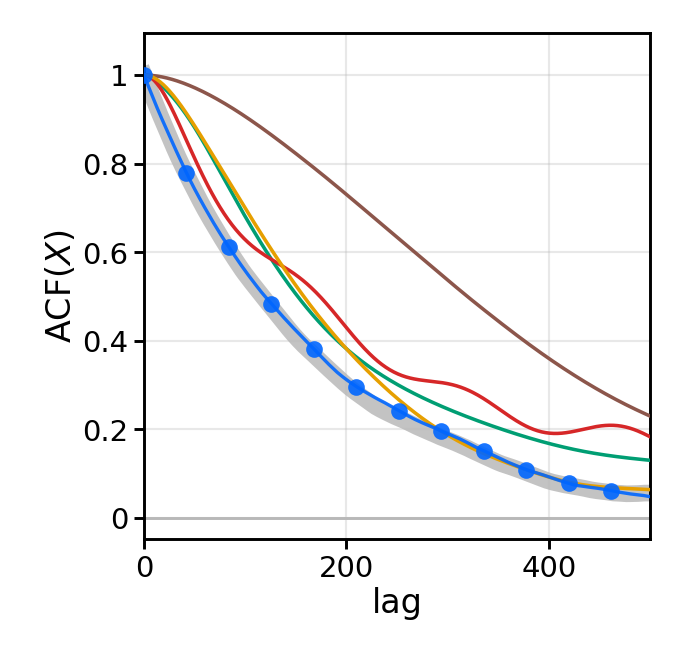}
    &
    \includegraphics[width=0.25\linewidth]{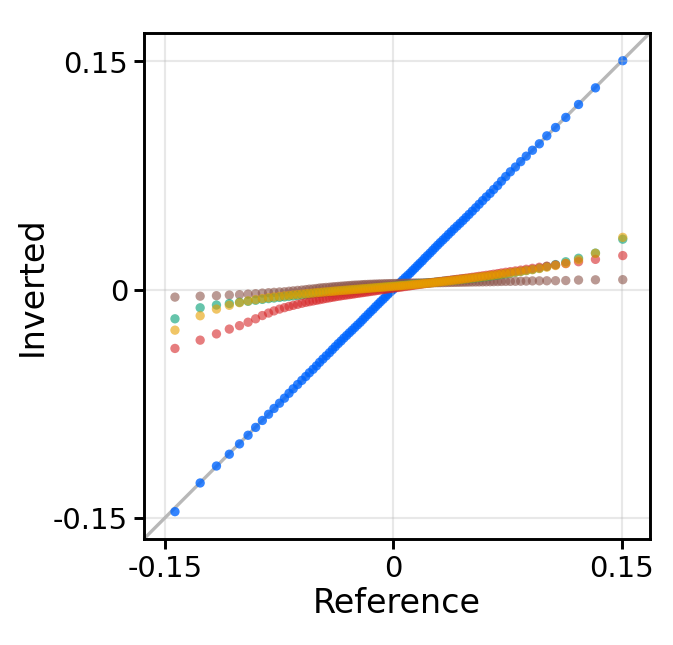} \\[-2pt]
    \multicolumn{3}{c}{(b) OU~($\kappa=5.0,\mu=3.0,\sigma=2.0$)}
\end{tabular}
}
\caption{\textbf{Path overlays, ACF plots, and QQ plots for OU with TLL6 conditioning.}
ACFs are computed from process values $X$.}
\label{fig:app-main-style-ou-tll}
\vspace{-4mm}
\end{figure}

\begin{figure}[h]
\centering
\captionsetup{skip=6pt}
\setlength{\tabcolsep}{0pt}
\renewcommand{\arraystretch}{0.8}
\scalebox{1}{
\begin{tabular}{ccc}
    \multicolumn{3}{c}{\includegraphics[height=0.015223\linewidth,trim=0 0 312.72bp 0,clip]{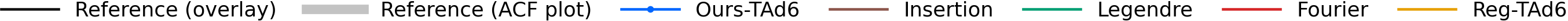}}\\
    \multicolumn{3}{c}{\includegraphics[height=0.015223\linewidth,trim=337.20bp 0 0 0,clip]{assets/results/figures/assets/appendix/ou/ou_theta0.5_mu3_sig2/TAd6_legend.png}}\\
    \includegraphics[width=0.48\linewidth]{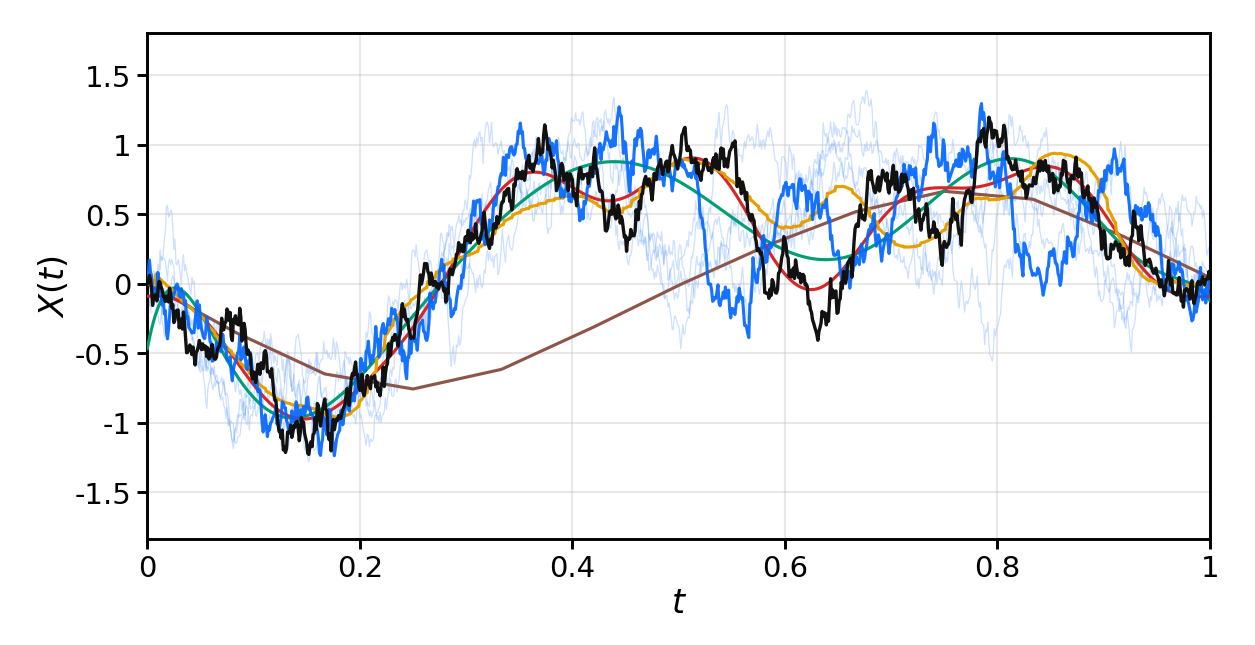}
    &
    \includegraphics[width=0.25\linewidth]{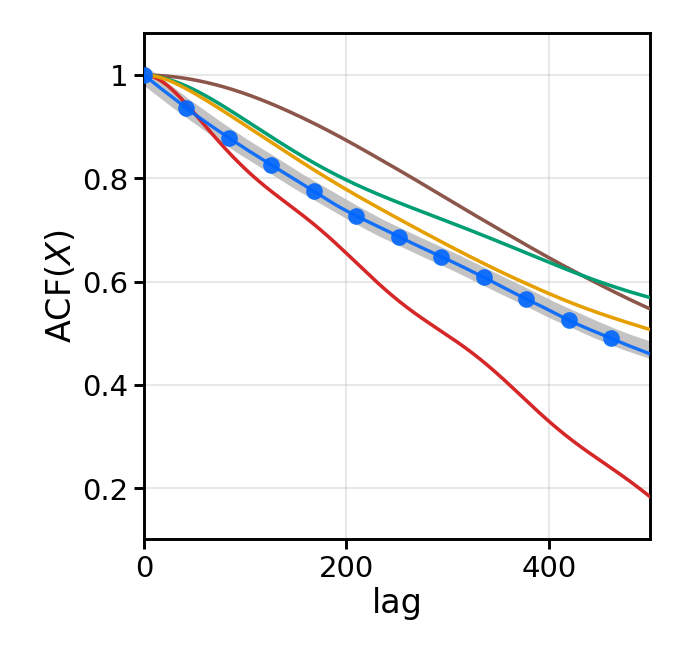}
    &
    \includegraphics[width=0.25\linewidth]{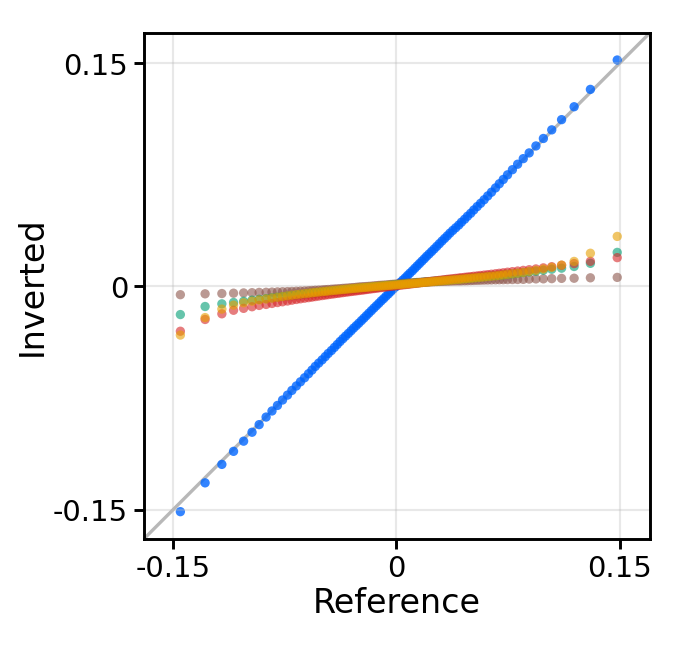} \\[-2pt]
    \multicolumn{3}{c}{(a) OU~($\kappa=0.5,\mu=3.0,\sigma=2.0$)} \\
    \includegraphics[width=0.48\linewidth]{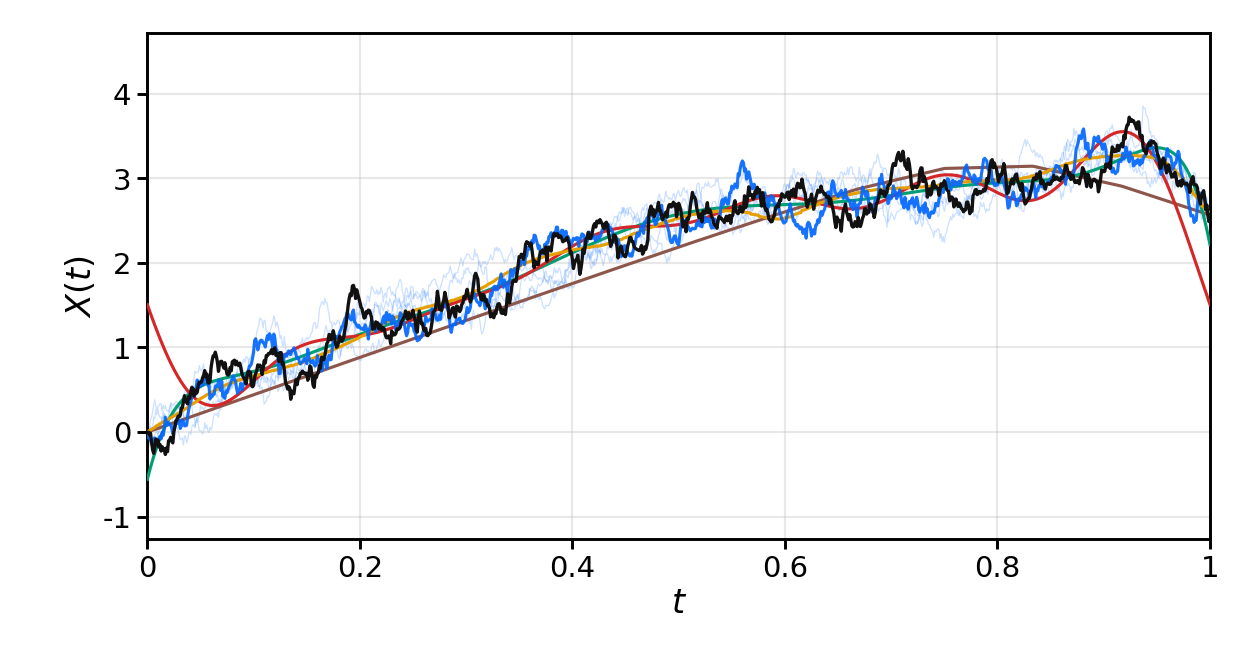}
    &
    \includegraphics[width=0.25\linewidth]{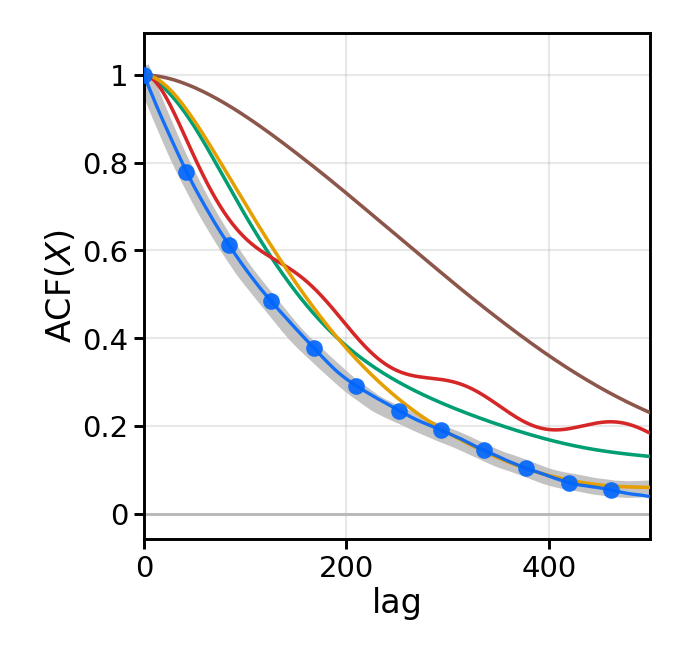}
    &
    \includegraphics[width=0.25\linewidth]{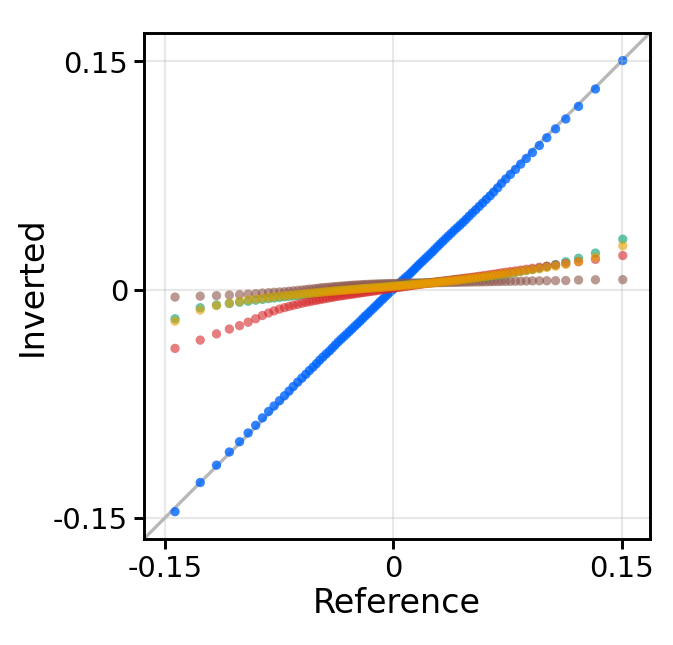} \\[-2pt]
    \multicolumn{3}{c}{(b) OU~($\kappa=5.0,\mu=3.0,\sigma=2.0$)}
\end{tabular}
}
\caption{\textbf{Path overlays, ACF plots, and QQ plots for OU with TA6 conditioning.}
ACFs are computed from process values $X$.}
\label{fig:app-main-style-ou-ta}
\vspace{-4mm}
\end{figure}

\clearpage
\subsection{Quantitative results for two-dimensional log-GBM}
\label{app:multivariate-gbm-experiments}

We compare LS and TA conditioning at depths $r=2,4,6$, and evaluate \OursTAd{4} against \Fourier~(order $n=4$) and \Legendre~(order $n=8$).

Across depths, the LS empirical reconstruction error closely matches the analytic oracle, while TA yields lower reconstruction errors~(\cref{tab:multivariate-gbm-reconstruction}).

\begin{table}[h]
\centering
\captionsetup{skip=6pt}
\caption{\textbf{Posterior-sampling reconstruction errors for two-dimensional correlated log-GBM.}
LS Oracle denotes the componentwise LS oracle in \cref{cor:multivariate-gbm-oracle}; LS Emp. and TA Emp. denote empirical errors under LS and TA conditioning.}
\label{tab:multivariate-gbm-reconstruction}
\renewcommand{\arraystretch}{1.1}
\setlength{\tabcolsep}{6pt}
\begin{tabular}{cccc}
\toprule
$r$ & LS Oracle & LS Emp. & TA Emp. \\
\midrule
2 & 1.09 & 1.08 & 1.02 \\
4 & 0.52 & 0.52 & 0.38 \\
6 & 0.34 & 0.34 & 0.33 \\
\bottomrule
\end{tabular}
\end{table}

\clearpage

\section{Additional financial-data results}
\label{app:real}

\subsection{Quantitative results for S\&P~500}
\label{app:real-quant}

\paragraph{Stylized-fact and autocorrelation errors.}
\cref{tab:snp500_structure_id,tab:snp500_structure_ood} report stylized-fact and autocorrelation errors on the S\&P~500 splits.
Probabilistic inversions generally achieve lower median errors than the deterministic baselines across these metrics, indicating better recovery of the distributional and temporal characteristics of the reference paths.
For reference, \cref{tab:snp500_reference_stats} reports the corresponding S\&P~500 stylized facts.

\begin{table}[h]
\centering
\captionsetup{skip=6pt}

\caption{\textbf{In-sample S\&P~500 stylized-fact and autocorrelation errors.}
Entries report median absolute errors with 10th--90th percentile ranges in brackets.
Bold entries indicate the lowest unrounded median error.}
\label{tab:snp500_structure_id}

\renewcommand{\arraystretch}{1.1}
\setlength{\tabcolsep}{3.5pt}
\scalebox{0.85}{%
\begin{tabular}{@{}lccccc@{}}
\toprule
& \multicolumn{3}{c}{Stylized-fact} & \multicolumn{2}{c}{Autocorrelation} \\
\cmidrule(lr){2-4}\cmidrule(lr){5-6}
Method & Volatility & Skewness & Kurtosis & ACF($r$) & ACF($|r|$) \\
\midrule
\textbf{Deterministic} \\
\Fourier & 0.11 [0.06, 0.19] & 0.39 [0.07, 1.39] & 1.89 [0.42, 4.74] & 0.41 [0.34, 0.47] & 0.23 [0.20, 0.29] \\
\Legendre & 0.12 [0.08, 0.20] & 1.37 [0.19, 3.68] & 5.24 [1.01, 20.59] & 0.31 [0.19, 0.47] & 0.18 [0.11, 0.38] \\
\Insertion & 0.13 [0.10, 0.21] & 0.50 [0.17, 1.04] & 3.19 [1.13, 7.31] & 0.80 [0.66, 0.87] & 0.50 [0.32, 0.75] \\
\RegTAd{4} & 0.10 [0.08, 0.16] & 1.43 [0.22, 2.61] & 11.33 [2.69, 32.69] & 0.15 [0.11, 0.22] & 0.11 [0.08, 0.15] \\
\RegTLLd{4} & 0.10 [0.07, 0.16] & 1.01 [0.15, 2.64] & 12.04 [3.74, 22.67] & 0.14 [0.11, 0.23] & 0.11 [0.08, 0.15] \\
\midrule
\textbf{Probabilistic} \\
\OursTAd{4} & 0.02 [0.00, 0.05] & 0.30 [0.06, 0.78] & 1.26 [0.19, 5.29] & 0.08 [0.06, 0.10] & 0.08 [0.06, 0.10] \\
\OursTLLd{4} & \textbf{0.02 [0.00, 0.05]} & \textbf{0.29 [0.08, 0.75]} & \textbf{1.19 [0.29, 3.87]} & \textbf{0.08 [0.06, 0.10]} & \textbf{0.08 [0.06, 0.10]} \\
\bottomrule
\end{tabular}
}
\end{table}

\begin{table}[h]
\centering
\captionsetup{skip=6pt}

\caption{\textbf{Out-of-sample S\&P~500 stylized-fact and autocorrelation errors.}
Entries report median absolute errors with 10th--90th percentile ranges in brackets.
Bold entries indicate the lowest unrounded median error.}
\label{tab:snp500_structure_ood}

\renewcommand{\arraystretch}{1.1}
\setlength{\tabcolsep}{3.5pt}
\scalebox{0.85}{%
\begin{tabular}{@{}lccccc@{}}
\toprule
& \multicolumn{3}{c}{Stylized-fact} & \multicolumn{2}{c}{Autocorrelation} \\
\cmidrule(lr){2-4}\cmidrule(lr){5-6}
Method & Volatility & Skewness & Kurtosis & ACF($r$) & ACF($|r|$) \\
\midrule
\textbf{Deterministic} \\
\Fourier & 0.08 [0.06, 0.16] & 1.35 [0.95, 1.51] & 2.37 [1.09, 17.36] & 0.36 [0.34, 0.43] & 0.24 [0.20, 0.27] \\
\Legendre & 0.10 [0.08, 0.16] & 1.25 [0.26, 3.26] & 9.61 [0.85, 20.02] & 0.27 [0.18, 0.48] & 0.19 [0.14, 0.32] \\
\Insertion & 0.11 [0.10, 0.17] & 0.94 [0.19, 1.28] & 2.77 [0.45, 18.47] & 0.80 [0.61, 0.87] & 0.57 [0.39, 0.78] \\
\RegTAd{4} & 0.09 [0.07, 0.13] & 1.47 [0.32, 3.75] & 11.58 [2.56, 29.35] & 0.12 [0.09, 0.15] & 0.11 [0.08, 0.15] \\
\RegTLLd{4} & 0.09 [0.08, 0.12] & 1.11 [0.17, 2.83] & 11.29 [3.15, 19.91] & 0.10 [0.08, 0.14] & 0.10 [0.08, 0.13] \\
\midrule
\textbf{Probabilistic} \\
\OursTAd{4} & 0.02 [0.01, 0.05] & \textbf{0.45 [0.06, 1.19]} & \textbf{2.35 [0.41, 14.25]} & \textbf{0.08 [0.06, 0.09]} & \textbf{0.08 [0.06, 0.10]} \\
\OursTLLd{4} & \textbf{0.02 [0.00, 0.05]} & 0.52 [0.11, 1.44] & 2.42 [0.49, 13.10] & 0.08 [0.06, 0.09] & 0.08 [0.06, 0.10] \\
\bottomrule
\end{tabular}
}
\end{table}

\begin{table}[h]
\centering
\captionsetup{skip=6pt}

\caption{\textbf{In-sample and out-of-sample S\&P~500 reference stylized facts.}
Statistics are computed over evaluation windows.
Entries report medians with 10th--90th percentile ranges in brackets.}
\label{tab:snp500_reference_stats}

\renewcommand{\arraystretch}{1.1}
\setlength{\tabcolsep}{4pt}
\scalebox{0.85}{
\begin{tabular}{@{}lccc@{}}
\toprule
Split & Volatility & Skewness & Kurtosis \\
\midrule
In-sample (2009--2022) & $0.15$ [$0.11$, $0.23$] & $-0.35$ [$-0.92$, $-0.13$] & $2.31$ [$0.81$, $6.68$] \\
Out-of-sample (2023--2025) & $0.13$ [$0.12$, $0.19$] & $-0.08$ [$-0.55$, $0.69$] & $1.60$ [$-0.08$, $17.09$] \\
\bottomrule
\end{tabular}
}
\end{table}

\clearpage
\paragraph{Signature fidelity.}
\cref{tab:signature_fidelity_real} shows that the learned methods achieve low signature errors in both in-sample and out-of-sample periods, indicating that their outputs remain close to the conditioning path in signature space on financial data.

\begin{table}[h]
\centering
\captionsetup{skip=6pt}

\caption{\textbf{Signature fidelity on in-sample and out-of-sample S\&P~500.}
Entries report relative MSE of truncated time-augmented signatures evaluated up to depth $4$, expressed as percentages.}
\label{tab:signature_fidelity_real}

\renewcommand{\arraystretch}{1.1}
\setlength{\tabcolsep}{4pt}
\scalebox{0.85}{
\begin{tabular}{l cc}
\toprule
Method & \shortstack{In-sample\\(2009--2022)} & \shortstack{Out-of-sample\\(2023--2025)} \\ \midrule
\textbf{Deterministic} \\
\Fourier   & 1.71\% & 3.71\% \\
\Legendre  & 0.08\% & 0.10\% \\
\Insertion & 0.03\% & 0.03\% \\
\RegTAd{4}    & \textless{}0.01\% & \textless{}0.01\% \\
\RegTLLd{4}   & \textless{}0.01\% & 0.03\% \\
\cmidrule(lr){1-3}
\textbf{Probabilistic} \\
\OursTAd{4}        & \textless{}0.01\% & 0.01\% \\
\OursTLLd{4}       & 0.01\% & 0.02\% \\
\bottomrule
\end{tabular}
}
\end{table}

\clearpage
\subsection{Qualitative results on S\&P~500}
\label{app:real-qualitative}

This subsection reports path overlays, absolute-return ACF plots, and QQ plots for S\&P~500 in-sample and out-of-sample examples under TLLd4 and TAd4 conditioning.
In these examples, probabilistic inversions better reproduce local fluctuations, whereas deterministic reconstructions tend to be smoother.
The absolute-return ACF and QQ plots generally support this pattern, with probabilistic inversions more closely matching the reference volatility clustering and marginal behavior.

\begin{figure}[h]
\centering
\captionsetup{skip=6pt}
\setlength{\tabcolsep}{0pt}
\renewcommand{\arraystretch}{0.8}
\scalebox{1}{
\begin{tabular}{ccc}
    \multicolumn{3}{c}{\includegraphics[height=0.015223\linewidth]{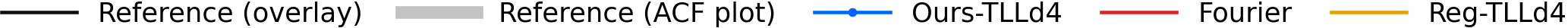}}\\
    \includegraphics[width=0.48\linewidth]{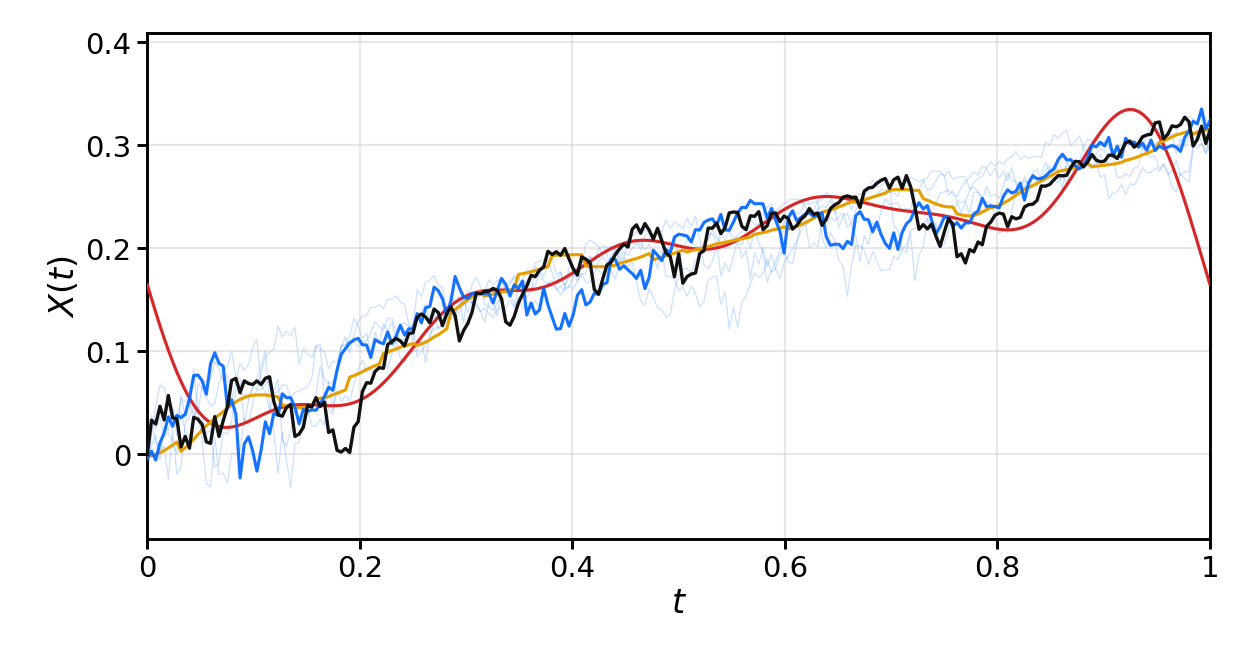}
    &
    \includegraphics[width=0.25\linewidth]{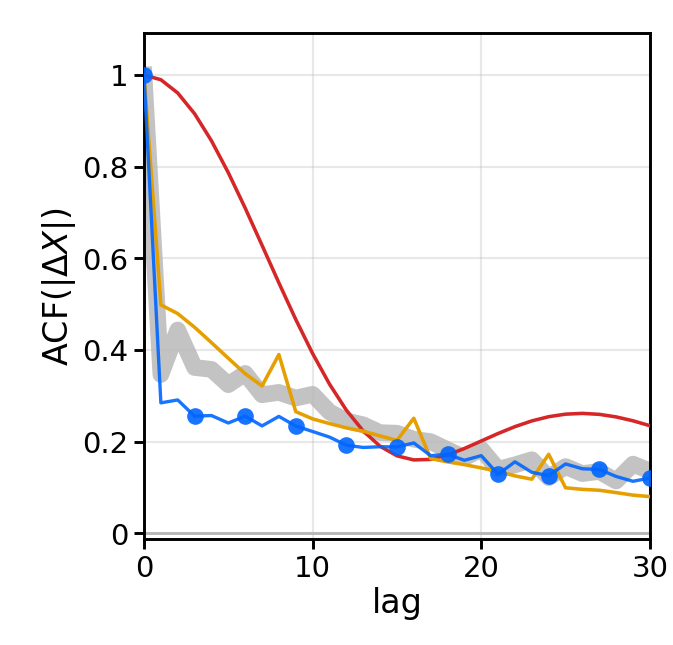}
    &
    \includegraphics[width=0.25\linewidth]{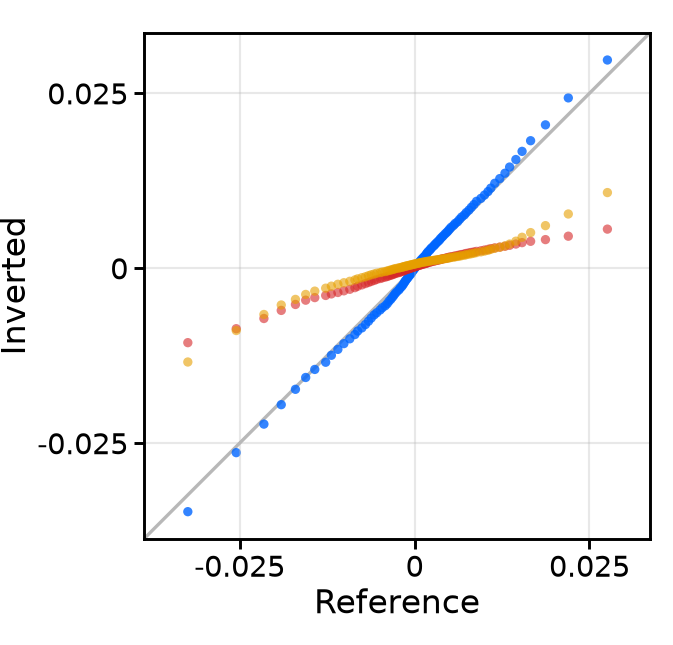} \\[-2pt]
    \multicolumn{3}{c}{(a) S\&P~500~(in-sample)} \\
    \includegraphics[width=0.48\linewidth]{assets/results/figures/assets/main/finance/out-distribution/TLL3d4_overlay_path03.png}
    &
    \includegraphics[width=0.25\linewidth]{assets/results/figures/assets/main/finance/out-distribution/TLL3d4_acf.png}
    &
    \includegraphics[width=0.25\linewidth]{assets/results/figures/assets/main/finance/out-distribution/TLL3d4_qq.png} \\[-2pt]
    \multicolumn{3}{c}{(b) S\&P~500~(out-of-sample)}
\end{tabular}
}
\caption{\textbf{Path overlays, ACF plots, and QQ plots for in-sample and out-of-sample S\&P~500 with TLL4 conditioning.}
ACFs are computed from absolute returns.}
\label{fig:app-real-overlay-tll}
\vspace{-4mm}
\end{figure}

\begin{figure}[h]
\centering
\captionsetup{skip=6pt}
\setlength{\tabcolsep}{0pt}
\renewcommand{\arraystretch}{0.8}
\scalebox{1}{
\begin{tabular}{ccc}
    \multicolumn{3}{c}{\includegraphics[height=0.015223\linewidth]{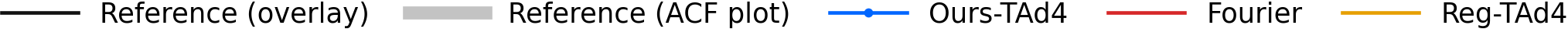}}\\
    \includegraphics[width=0.48\linewidth]{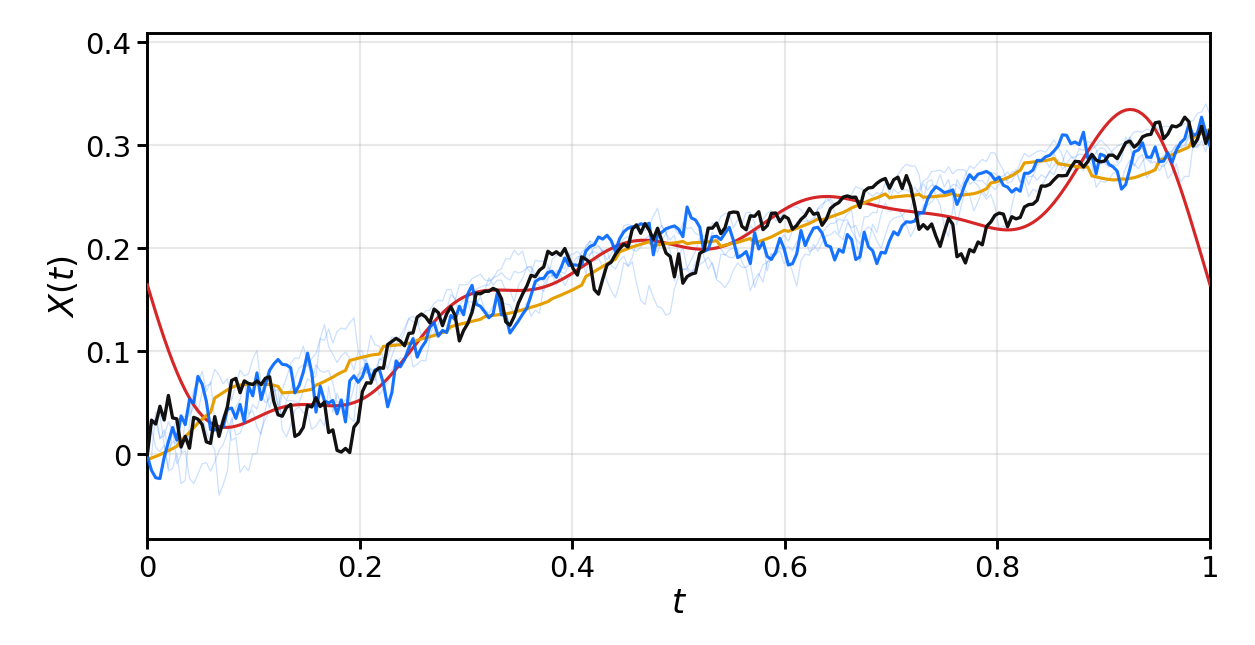}
    &
    \includegraphics[width=0.25\linewidth]{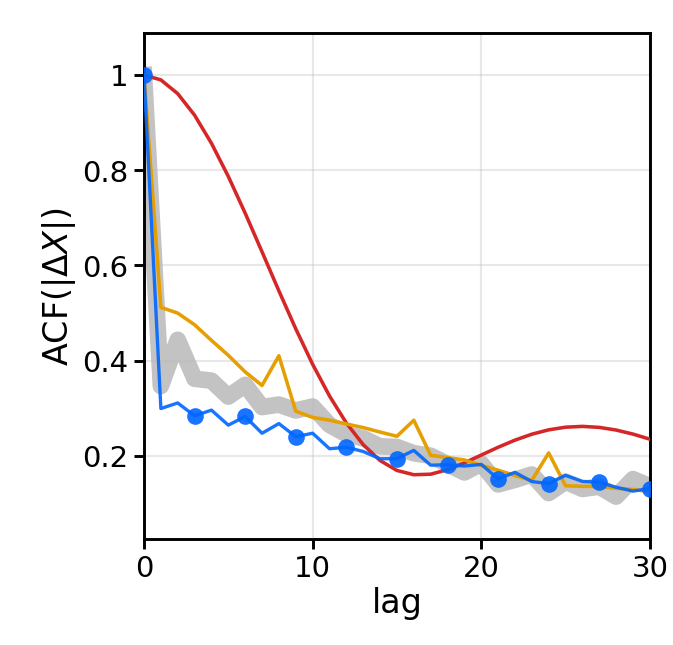}
    &
    \includegraphics[width=0.25\linewidth]{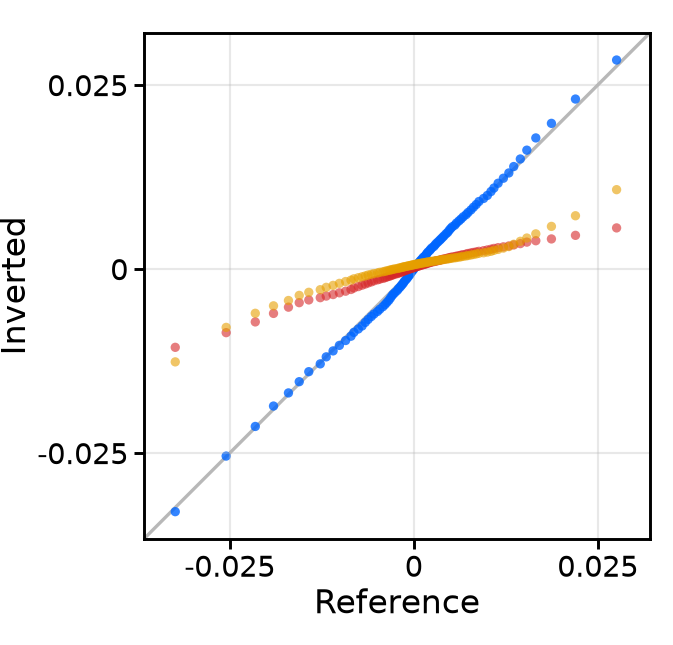} \\[-2pt]
    \multicolumn{3}{c}{(a) S\&P~500~(in-sample)} \\
    \includegraphics[width=0.48\linewidth]{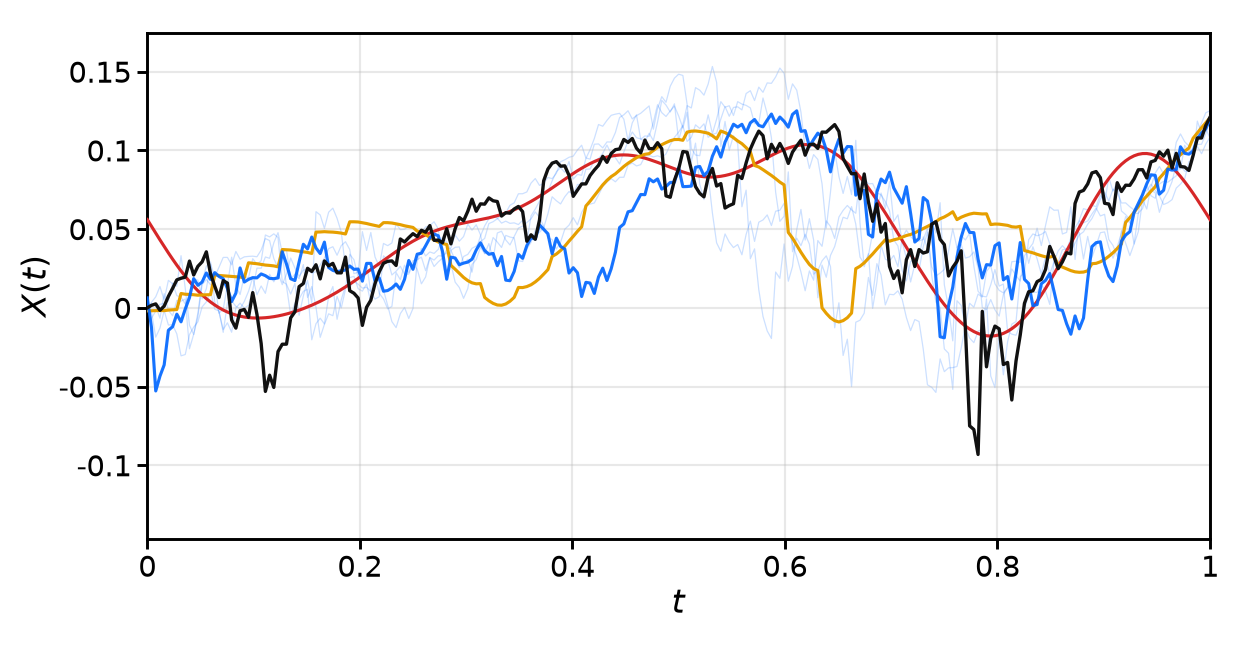}
    &
    \includegraphics[width=0.25\linewidth]{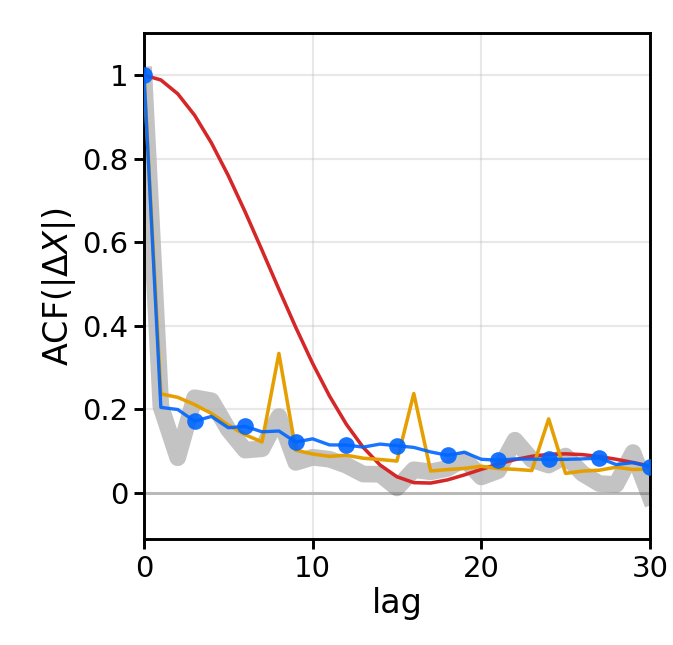}
    &
    \includegraphics[width=0.25\linewidth]{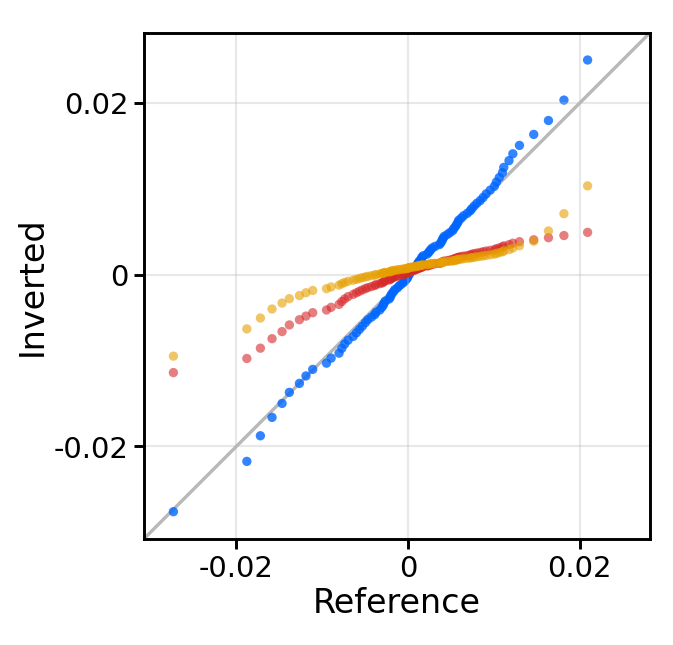} \\[-2pt]
    \multicolumn{3}{c}{(b) S\&P~500~(out-of-sample)}
\end{tabular}
}
\caption{\textbf{Path overlays, ACF plots, and QQ plots for in-sample and out-of-sample S\&P~500 with TA4 conditioning.}
ACFs are computed from absolute returns.}
\label{fig:app-real-overlay-ta}
\vspace{-4mm}
\end{figure}

\clearpage
\subsection{Quantitative results for two-dimensional foreign-exchange data}
\label{app:multivariate-fx}

We compare \OursTAd{4} and \OursTLLd{4} with \Fourier~(order $n=4$) and \Legendre~(order $n=8$) on paired GBP/USD and JPY/USD paths.
Both conditional samplers achieve lower median errors than the deterministic baselines on all seven in-sample diagnostics~(\cref{tab:fx-id}) and six of the seven out-of-sample diagnostics~(\cref{tab:fx-ood}).
The exception is out-of-sample absolute-return cross-series correlation, for which \Legendre has the lowest median error.

\begin{table}[h]
\centering
\captionsetup{skip=6pt}
\caption{\textbf{In-sample foreign-exchange stylized-fact, autocorrelation, and cross-series correlation errors.}
Entries report median absolute errors with 10th--90th percentile ranges in brackets.
$\rho(r)$ and $\rho(|r|)$ denote cross-series correlations of returns and absolute returns, respectively.
Bold entries indicate the lowest unrounded median error.}
\label{tab:fx-id}
\footnotesize
\renewcommand{\arraystretch}{1.2}
\setlength{\tabcolsep}{1.8pt}
\scalebox{0.90}{%
\begin{tabular}{@{}lccccccc@{}}
\toprule
& \multicolumn{3}{c}{Stylized-fact} & \multicolumn{2}{c}{Autocorrelation} & \multicolumn{2}{c}{Cross-series correlation} \\
\cmidrule(lr){2-4}\cmidrule(lr){5-6}\cmidrule(lr){7-8}
Method & Volatility & Skewness & Kurtosis & ACF($r$) & ACF($|r|$) & $\rho(r)$ & $\rho(|r|)$ \\
\midrule
\Fourier & 0.07 [0.06, 0.10] & 0.50 [0.22, 0.97] & 2.81 [1.66, 21.02] & 0.48 [0.44, 0.56] & 0.35 [0.30, 0.41] & 0.25 [0.07, 0.59] & 0.27 [0.05, 0.51] \\
\Legendre & 0.08 [0.06, 0.10] & 1.38 [0.56, 2.20] & 4.83 [1.42, 16.55] & 0.38 [0.30, 0.53] & 0.29 [0.23, 0.39] & 0.23 [0.05, 0.49] & 0.22 [0.07, 0.45] \\
\midrule
\OursTAd{4} & \textbf{0.01 [0.00, 0.02]} & 0.43 [0.12, 0.96] & 2.26 [0.60, 14.69] & 0.07 [0.06, 0.08] & 0.07 [0.06, 0.10] & 0.13 [0.01, 0.27] & 0.10 [0.01, 0.27] \\
\OursTLLd{4} & 0.01 [0.00, 0.02] & \textbf{0.42 [0.14, 0.98]} & \textbf{1.90 [0.76, 13.14]} & \textbf{0.07 [0.06, 0.08]} & \textbf{0.07 [0.06, 0.10]} & \textbf{0.10 [0.02, 0.24]} & \textbf{0.09 [0.02, 0.29]} \\
\bottomrule
\end{tabular}
}
\end{table}

\begin{table}[h]
\centering
\captionsetup{skip=6pt}
\caption{\textbf{Out-of-sample foreign-exchange stylized-fact, autocorrelation, and cross-series correlation errors.}
Entries report median absolute errors with 10th--90th percentile ranges in brackets.
$\rho(r)$ and $\rho(|r|)$ denote cross-series correlations of returns and absolute returns, respectively.}
\label{tab:fx-ood}
\footnotesize
\renewcommand{\arraystretch}{1.2}
\setlength{\tabcolsep}{1.8pt}
\scalebox{0.90}{%
\begin{tabular}{@{}lccccccc@{}}
\toprule
& \multicolumn{3}{c}{Stylized-fact} & \multicolumn{2}{c}{Autocorrelation} & \multicolumn{2}{c}{Cross-series correlation} \\
\cmidrule(lr){2-4}\cmidrule(lr){5-6}\cmidrule(lr){7-8}
Method & Volatility & Skewness & Kurtosis & ACF($r$) & ACF($|r|$) & $\rho(r)$ & $\rho(|r|)$ \\
\midrule
\Fourier & 0.06 [0.03, 0.09] & 0.95 [0.46, 3.00] & 3.45 [2.38, 19.59] & 0.46 [0.42, 0.51] & 0.35 [0.30, 0.39] & 0.56 [0.23, 0.93] & 0.30 [0.06, 0.56] \\
\Legendre & 0.06 [0.04, 0.10] & 1.59 [0.66, 3.06] & 5.23 [1.58, 18.56] & 0.39 [0.29, 0.54] & 0.33 [0.24, 0.45] & 0.39 [0.10, 0.81] & \textbf{0.19 [0.04, 0.38]} \\
\midrule
\OursTAd{4} & 0.02 [0.01, 0.04] & 0.48 [0.18, 2.15] & 2.92 [1.65, 18.65] & 0.07 [0.06, 0.08] & 0.10 [0.08, 0.13] & 0.17 [0.04, 0.30] & 0.27 [0.12, 0.40] \\
\OursTLLd{4} & \textbf{0.01 [0.00, 0.03]} & \textbf{0.45 [0.16, 2.12]} & \textbf{2.91 [1.72, 19.14]} & \textbf{0.07 [0.06, 0.08]} & \textbf{0.10 [0.08, 0.13]} & \textbf{0.14 [0.04, 0.32]} & 0.30 [0.16, 0.42] \\
\bottomrule
\end{tabular}
}
\end{table}

\clearpage

\section{Ablations and alternative modeling choices}
\label{app:ablations}

\subsection{Effect of log-signature preprocessing}
\label{app:log-signature-preprocessing}

\paragraph{Experimental setup.}
Higher-order log-signature coordinates can contain extreme values, so our main setup applies the level-wise scaling map described in \cref{app:training}, including both factorial scaling and the signed logarithm.
\cref{tab:loggbm-ablation-mse} compares level-wise scaling against no scaling on log-GBM, with the latter omitting this entire transformation.

\paragraph{Results.}
\cref{tab:loggbm-ablation-mse} shows that, with level-wise scaling, LS reconstruction errors remain close to the oracle even at greater depths, whereas without scaling they increasingly deviate from the oracle as depth increases.
These results support the use of level-wise scaling to maintain stable reconstruction performance as conditioning depth increases.

\begin{table}[h]
\centering
\captionsetup{skip=6pt}

\caption{\textbf{Posterior-sampling reconstruction errors with and without level-wise scaling of LS conditioning statistics on log-GBM.}
Entries report path MSE.
Bold entries indicate the LS reconstruction error closest to the oracle at each depth.}
\label{tab:loggbm-ablation-mse}

\renewcommand{\arraystretch}{1.1}
\setlength{\tabcolsep}{4pt}
\scalebox{0.85}{
\begin{tabular}{l ccc}
\toprule
Setup & $r=4$ & $r=5$ & $r=6$ \\ \midrule

LS Oracle & 0.260 & 0.207 & 0.172 \\
\midrule
w/o level-wise scaling & \textbf{0.260} & 0.210 & 0.206 \\
w/\phantom{o} level-wise scaling & 0.261 & \textbf{0.208} & \textbf{0.173} \\

\bottomrule
\end{tabular}
}
\end{table}

\clearpage
\subsection{Effect of the conditional distribution estimator}
\label{app:conditional-diffusion}

\paragraph{Experimental setup.}
We compare conditional diffusion with flow matching to assess how the choice of conditional distribution estimator affects reconstruction and parameter recovery.
We use the same architecture, conditioning, and experimental protocol as the flow-matching model.
The diffusion baseline uses the DDPM objective~\citep{ho2020denoising} with a cosine noise schedule and DDIM sampling.
We evaluate DDIM sampling with and without predicted-$x_0$ clipping to $[-10,10]$.

\paragraph{Results.}
\cref{tab:diffusion-ls-reconstruction} shows that, with clipping, conditional diffusion achieves reconstruction errors close to the LS oracle.
\cref{tab:diffusion-tll-parameters} shows that it also achieves low parameter-estimation errors under TLL conditioning.
However, unclipped DDIM is unstable in these experiments, while clipping imposes a fixed pointwise range not implied by the conditioning signature.
Flow matching achieves comparable oracle alignment and equal or lower reported parameter-estimation errors without clipping.

\begin{table}[h]
\centering
\captionsetup{skip=6pt}
\caption{\textbf{Posterior-sampling reconstruction errors of conditional diffusion and flow matching under LS conditioning.}
LS Oracle denotes the oracle posterior-sampling reconstruction error.
DDPM (no clip) and DDPM (clip) use DDIM sampling without and with clipping of the predicted $x_0$ to $[-10,10]$, respectively.
FM denotes flow matching without clipping.}
\label{tab:diffusion-ls-reconstruction}
\renewcommand{\arraystretch}{1.1}
\setlength{\tabcolsep}{5pt}
\begin{tabular}{lccccc}
\toprule
Process & $r$ & LS Oracle & DDPM (no clip) & DDPM (clip) & FM \\
\midrule
\multirow{3}{*}{log-GBM} & 2 & 0.54 & 14.44 & 0.53 & 0.55 \\
& 4 & 0.26 & 11.34 & 0.25 & 0.26 \\
& 6 & 0.17 & 25.55 & 0.18 & 0.17 \\
\midrule
\multirow{3}{*}{fBM} & 2 & 0.66 & 13.28 & 0.67 & 0.66 \\
& 4 & 0.35 & 14.68 & 0.37 & 0.36 \\
& 6 & 0.25 & 15.20 & 0.27 & 0.27 \\
\bottomrule
\end{tabular}
\end{table}

\begin{table}[h]
\centering
\captionsetup{skip=6pt}
\caption{\textbf{Parameter-estimation errors of conditional diffusion and flow matching under TLL conditioning.}
DDPM (clip) denotes conditional diffusion with DDIM sampling and predicted-$x_0$ clipping to $[-10,10]$; FM denotes flow matching without clipping.
Bold entries indicate the lower reported error in each DDPM--FM pair, including ties.}
\label{tab:diffusion-tll-parameters}
\renewcommand{\arraystretch}{1.1}
\setlength{\tabcolsep}{4pt}
\begin{tabular}{ccccccc}
\toprule
& \multicolumn{2}{c}{log-GBM $\sigma$} & \multicolumn{2}{c}{fBM $\sigma$} & \multicolumn{2}{c}{fBM $H$} \\
\cmidrule(lr){2-3}\cmidrule(lr){4-5}\cmidrule(lr){6-7}
$r$ & DDPM (clip) & FM & DDPM (clip) & FM & DDPM (clip) & FM \\
\midrule
2 & 0.06 & \textbf{0.04} & 0.07 & \textbf{0.05} & 0.007 & \textbf{0.006} \\
4 & 0.05 & \textbf{0.04} & 0.08 & \textbf{0.05} & 0.007 & \textbf{0.006} \\
6 & 0.06 & \textbf{0.05} & 0.08 & \textbf{0.07} & \textbf{0.007} & \textbf{0.007} \\
\bottomrule
\end{tabular}
\end{table}

\end{document}